\documentclass[11pt]{article} 

\usepackage{hyperref}
\usepackage{url}
\usepackage{mathtools}
\usepackage{fullpage}

\usepackage{dsfont}[sans]

\usepackage{amsmath,amssymb,amsthm, amsfonts}
\usepackage{mathbbol}
\usepackage{mathrsfs}  
\usepackage{multicol}
\usepackage{rotating}
\usepackage{enumitem, comment, xifthen}

\usepackage{mathbbol}

\usepackage{booktabs}

\usepackage[utf8]{inputenc} 
\usepackage[T1]{fontenc}    
\usepackage{hyperref}       
\usepackage{url}            
\usepackage{booktabs}       
\usepackage{amsfonts}       
\usepackage{nicefrac}       

\usepackage{macros}
\usepackage{comment}
\usepackage{enumitem}

\newcommand{\FunctionalNP}[1]{f}

\def\nTestFunctions{n_{\TestFunctionClass{}}}

\def\QfncTest{\Qfnc'}

\newcommand{\ConcentrabilityGeneric}[1]{K^{#1}}
\newcommand{\ConcentrabilityGenericSub}[2]{K^{#1}_{(#2)}}

\def\pInP{\policy \in \PolicyClass}

\def\DatasetDistributionStateActions{\mu}

\def\DatasetDistributionEmpiricalStateAction{\Dataset}

\newcommand{\ConfidenceInterval}[1]{L}

\newcommand{\ConfidenceIntervalEmpirical}[1]{\widehat{L}}

\newcommand{\EigenVectorEmpirical}[1]{\widehat u_{#1}}
\newcommand{\EigenValueEmpirical}[1]{\widehat \lambda_{#1}}

\newcommand{\SuperEmpiricalFeasible}{\EmpiricalFeasibleSet{\policy}(\Rad, \regpar;
  \TestFunctionClass{\policy})}

\newcommand{\SuperPopulationFeasible}{\PopulationFeasibleSet{\policy}(4 \Rad, \regpar;
  \TestFunctionClass{\policy})}
  
\newcommand{\TestFunction}[1]{f_{#1}}

\newcommand{\TestFunctionReg}[0]{\lambda}
\newcommand{\TestNormaRegularizerEmp}[1]
{\sqrt{\norm{\TestFunction{#1}}{\nSamples}^2 + \regpar}}

\newcommand{\TestNormaRegularizerPop}[1]
{\sqrt{\norm{\TestFunction{#1}}{\DatasetDistributionStateActions}^2 + \regpar}}

\newcommand{\mirrdist}{\ensuremath{\nu}}
\newcommand{\LinSpace}{\ensuremath{\mathcal{L}}}
\newcommand{\LinClass}{\LinSpace}

\newcommand{\martone}{\ensuremath{W}}
\newcommand{\marttwo}{\ensuremath{\martone'}}

\newcommand{\goodname}{off-policy cost coefficient}

\newcommand{\shortgoodname}{OPC }
\newcommand{\Gclass}{\ensuremath{\mathcal{G}}}

\newenvironment{carlist}
 {\begin{list}{$\bullet$}
 {\setlength{\topsep}{0in} \setlength{\partopsep}{0in}
  \setlength{\parsep}{0in} \setlength{\itemsep}{\parskip}
  \setlength{\leftmargin}{0.15in} \setlength{\rightmargin}{0.08in}
  \setlength{\listparindent}{0in} \setlength{\labelwidth}{0.08in}
  \setlength{\labelsep}{0.1in} \setlength{\itemindent}{0in}}}
 {\end{list}}

\newcommand{\bcar}{\begin{carlist}}
\newcommand{\ecar}{\end{carlist}}

\newcommand{\ftil}{\tilde{f}}
\newcommand{\Qtil}{\tilde{Q}}
\newcommand{\pitil}{\tilde{\pi}}

\long\def\mycomment#1{}

\newcommand{\Rad}{\ensuremath{\rho}}

\newcommand{\regpar}{\ensuremath{\lambda}}
\newcommand{\EmpNorm}[1]{\|#1\|_{\nSamples}}
\newcommand{\munorm}[1]{\|#1\|_\mudist}

\newcommand{\mudist}{\DatasetDistributionStateActions}
\newcommand{\numobs}{\nSamples}
\newcommand{\PolComplex}{\PolComplexGen{\policy}}

\newcommand{\PolMuComplex}{\Expecti{\BellmanError{\Qfnc}{\policy}{}} {
\mudist}}

\newcommand{\PolComplexGen}[1]{\Expecti{\BellmanError{\Qfnc}{#1}{}}{#1}}

\newcommand{\pol}{\policy}
\newcommand{\polcomp}{\ensuremath{\policy}}

\newcommand{\ConcSimple}{K^\policy}

\newcommand{\ZvarPlain}{\ensuremath{Z_\numobs}}
\newcommand{\Zvar}[2]{\ensuremath{\ZvarPlain(#1, #2)}}
\newcommand{\ZvarShort}{\Zvar{\Qfnc}{\policy}}
\newcommand{\Term}{T}
\newcommand{\Fclass}{\TestFunctionClass{}}

\newcommand{\BellError}{\ensuremath{\mathcal{B}}}
\newcommand{\Diff}[1]{\ensuremath{\mathcal{D}^{#1}}}
\newcommand{\Mbar}{\ensuremath{\bar{\Mplain}}}
\newcommand{\Mplain}{\ensuremath{M}}
\newcommand{\Exp}{\ensuremath{\mathbb{E}}}
\newcommand{\SpecFun}{\ensuremath{\Psi_\numobs}}
\newcommand{\testfun}{\ensuremath{f}}

\newcommand{\cspec}{\bar{c}}
\newcommand{\BellmanErr}{\ensuremath{\mathcal{B}}}
\newcommand{\ctil}{\ensuremath{\tilde{c}}}
\newcommand{\DF}{\ensuremath{\mathcal{F}}}

\newcommand{\TestClassPlain}{\ensuremath{\mathscr{F}}}
\newcommand{\TestFunctionClass}[1]{\TestClassPlain^{#1}}
\newcommand{\TestFunctionClassRandom}[1]{\widetilde{\TestClassPlain}^{#1}}

\newcommand{\TestFunctionISClass}[1]{\TestClassPlain^{\text{IS}}_{#1}}

\newcommand{\TestFunctionClassBubnov}[1]{\TestClassPlain^{\mathcal B}_{#1}}

\def\identifier{o} 
\def\SumOverSamples{\sum_{\iSample=1}^{\nSamples}}

\newcommand{\sarsiz}[1]{(z_{#1})}
\newcommand{\sarsizNp}[1]{z_{#1}}
\newcommand{\sarsi}[1]{(\state_{#1},\action_{#1},\reward_{#1},\successorstate_{#1},\identifier_{#1})}
\def\psai{(\state,\action,\identifier)}

\newcommand{\psaindex}{(\state_i, \action_i)}
\newcommand{\testfunc}{\ensuremath{f}}
\newcommand{\Identifier}{\ensuremath{O}}

\def\psai{(\state,\action,\identifier)}

\newcommand{\CoveringNumber}[2]{\mathcal N_{#2}(#1)}

\newcommand{\policyBeh}[1]{\policy^b}

\def\SumOverSars{\sum_{\sars{} \in \Dataset}}

\def\precision{\epsilon}

\def\LinPolicyClass{\PolicyClass_{\text{lin}}}
\def\LinPhi{\phi}
\newcommand{\LinPhiBootstrap}[1]{\LinPhi^{+#1}}
\newcommand{\LinPhiEmpirical}[1]{\phi_{#1}}

\newcommand{\nLinEffSq}[1]{n_{#1}}
\newcommand{\nLinEff}[1]{\sqrt{\nLinEffSq{#1}}}

\newcommand{\LinPhiEmpiricalExpecti}[1]{\widehat\LinPhi_{#1}}
\newcommand{\LinPhiEmpiricalExpectiBootstrap}[2]{\widehat\LinPhi_{#1}^{+#2}}
\newcommand{\LinEmpiricalReward}[1]{\widehat \reward_{#1}}

\def\LinPar{\theta}

\newcommand{\CriticParBest}[1]{\wPar^{#1}_\star}
\newcommand{\CriticParProjectionFull}[2]{\mathscr P^{#1}(#2)}
\newcommand{\CriticParProjection}[1]{\CriticParProjectionFull{#1}{\CriticPar{}}}

\def\horizon{\frac{1}{1-\discount}}

\def\piemp{\widetilde \pi}

\def\RewardLaw{R}
\def\TransitionLaw{\Pro}
\def\Filtration{\mathcal F}

\newcommand{\Dpi}[1]{\mu_{#1}}

\def\DatasetDistribution{d}

\def\Qfnc{Q}

\newcommand{\Qclass}[1]{\Q^{#1}}
\newcommand{\QclassErr}[1]{\mathcal{E}^{#1}}
\def\QfncErr{\epsilon}
\def\QfncErrNC{{\epsilon}}
\newcommand{\QclassErrCentered}[1]{\mathcal{E}^{#1}_\star}

\newcommand{\Frequencies}[1]{p_{#1}} 
\newcommand{\VminEmp}[1]{\widehat{V}_{\text{min}}^{#1}}
\newcommand{\VmaxEmp}[1]{\widehat{V}_{\text{max}}^{#1}}

\newcommand{\QminEmp}[1]{\widehat{\underline Q}_{#1}}

\newcommand{\EmpiricalFeasibleSet}[1]{\widehat{\mathscr{C}}_{\nSamples}^{#1}}
\newcommand{\PopulationFeasibleSet}[1]{\mathscr{C}_{\nSamples}^{#1}}

\newcommand{\PopulationFeasibleSetInfty}[1]{\mathscr C_{\infty}^{#1}}

\newcommand{\nConstraints}[0]{m}
\def\iSample{i}
\def\nSamples{n}
\def\iConstraint{j}
\def\IConstraint{J}
\def\BRQClass{\widetilde {\Qclass{}}}
\def\BRPolicyClass{\widetilde {\PolicyClass}}
\newcommand{\BRleft}[1]{\nu_{#1}}
\newcommand{\BRright}[1]{\xi_{#1}}
\newcommand{\BRDelta}[1]{\Delta_{#1}}
\newcommand{\BRDeltaCoord}[1]{\widetilde \Delta_{#1}}
\newcommand{\BRCovariance}[1]{\Sigma_{\BRleft{}}}

\def\CovarianceStandard{\Sigma}
\def\CovarianceStandardReg{\Sigma_{\lambda}}
\newcommand{\CovarianceBootstrap}[1]{\CovarianceStandard^{+#1}}

\newcommand{\CovarianceWithBootstrapReg}[1]{\CovarianceBootstrap{#1}_{\lambda,\text{Boot}}}
\def\CovarianceStandardExplicit
{\Expecti{\LinPhi\LinPhi^\top}{\DatasetDistributionStateActions} }
\def\CovarianceEmpiricalStandard{\widehat \Sigma}
\def\CovarianceEmpiricalStandardExplicit
{\SumOverSamples \LinPhiEmpirical{\iSample}\LinPhiEmpirical{\iSample}^\top}
\newcommand{\TDError}[3]{\delta^{#2}_{#3}#1}

\newcommand{\TDErrorDefCompact}[3]{#1(\state_{#3},\action_{#3}) - \reward_{#3} - \discount #1(\successorstate_{#3} ,#2)}
\newcommand{\TestFunctionDefCompact}[2]{\TestFunction{#1}(\state_{#2},\action_{#2}, \identifier_{#2})}

\newcommand{\TrDError}[1]{\delta_{R}}
\newcommand{\BellmanError}[3]{\mathcal B^{#2}_{#3}#1}
\newcommand{\BellmanErrorDefCompact}[3]{#1(\state_{#3},\action_{#3}) - (\BellmanEvaluation{#2}#1)(\state_{#3},\action_{#3}) }

\newcommand{\Expecti}[2]{\E_{#2}{{#1}}}

\def\startdistribution{\nu_{\text{start}}}

\newcommand{\Estart}[2]{\E_{\MyState \sim \startdistribution}(#1)(\MyState #2)}

\def\SumOverConstraints{\sum_{\iConstraint =1}^{\nConstraints}}

\newcommand{\ActorPolicy}[1]{\policy_{#1}}
\def\iter{t}
\def\nIter{T}

\def\LearningRate{\eta}
\def\comparator{\widetilde \policy}

\newcommand{\ActorPar}[1]{\LinPar_{#1}}
\newcommand{\CriticPar}[1]{\wPar_{#1}}

\def\pitil{\pitilde}

\def\MyState{S}
\def\Action{A}

\newcommand{\MirrorRegret}[1]{\mathcal{E}_{\tiny{opt}}(#1)}

\def\ActorParNorm{\nIter}

\def\MDP{\mathcal{M}}

\newcommand{\QpiWeak}[1]{\Qfnc_\star^{#1}}
\newcommand{\QpiProj}[2]{\mathscr P^{#1}(#2)}
\newcommand{\QpiAdv}[2]{\Qfnc_{#1}^{#2}}
\newcommand{\VpiAdv}[2]{V_{#1}^{#2}}
\newcommand{\ApiAdv}[2]{A_{#1}^{#2}}

\newcommand{\MDPadv}[1]{\MDP_{#1}}
\newcommand{\RewardLawAdv}[1]{\RewardLaw_{#1}}

\def\InputFunctionalSpace{
\ensuremath{(\PolicyClass,\TestFunctionClass{},\Qclass{})}}

\newcommand{\InputFunctionClass}{\InputFunctionalSpace}

\newcommand{\scaling}[1]{b_{#1}}
\newcommand{\scalingsq}[1]{b^2_{#1}}

\def\cOne{c_1}
\def\cTwo{c_2}
\def\cThree{c_3}
\def\cFour{c_4}

\newcommand{\intermediate}[6]{
Under the event of \cref{thm:NewPolicyEvaluation},
the statement of \cref{EqnSandwichZvar} holds, and in particular
\begin{align*}
\frac{1}{\cOne(\TestNormaRegularizerPop{})}
\geq 
\frac{1}{\TestNormaRegularizerEmp{} }
\geq 
\frac{1}{\cTwo (\TestNormaRegularizerPop{})}.
\end{align*}
Thus, the $\iConstraint$ constraint reads
\begin{align*}
	\frac{\ConfidenceInterval{\FailureProbability}}{\sqrt{\nSamples}}
	\gtrsim 
	\frac{
	\innerprodweighted{\TestFunction{\iConstraint}}{\BellmanError{\Qfnc}{\policy}{}}
	{\DatasetDistributionStateActions}
	}
	{
	\TestNormaRegularizerEmp{}
	}
	=
	\frac{
	\innerprodweighted{\TestFunction{\iConstraint}}{\BellmanError{\Qfnc}{\policy}{}}
	{\DatasetDistributionStateActions}
	}
	{
	\sqrt{\EigenValueEmpirical{\iConstraint} + \TestFunctionReg}
	}
\end{align*}
where the last step follows from 
\begin{align*}
	\norm{\TestFunction{\iConstraint}}{\DatasetDistributionEmpiricalStateAction}^2 
	= 
	\frac{1}{\nSamples} \SumOverSars (\TestFunction{\iConstraint}\psa)^2 
	=
	\frac{1}{\nSamples} \SumOverSamples
	(\EigenVectorEmpirical{\iConstraint}^\top
	\LinPhiEmpirical{\iSample})^2 
	= 
	\EigenVectorEmpirical{\iConstraint}^\top
	   \CovarianceEmpiricalStandard
	   \EigenVectorEmpirical{\iConstraint} 
	= \EigenValueEmpirical{\iConstraint}.
\end{align*}
Now, squaring and summing over the constraints and using #6 yields
\begin{align*}
\dim \frac{\ConfidenceInterval{\FailureProbability}^2}{\nSamples} 
& \gtrsim 
\SumOverConstraints 
\innerprodweighted{ \frac{\EigenVectorEmpirical{\iConstraint}^\top\LinPhi}{\sqrt{\EigenValueEmpirical{\iConstraint} + \TestFunctionReg}} }
{#1 #3}{\DatasetDistributionStateActions}^2 \\
& =
\SumOverConstraints \Big[\frac{\EigenVectorEmpirical{\iConstraint}^\top}{\sqrt{\EigenValueEmpirical{\iConstraint} + \TestFunctionReg}} 
\Expecti{ \LinPhi
{#1 #3}}{\DatasetDistributionStateActions}\Big]^2 \\
& =
\SumOverConstraints \Big[\frac{\EigenVectorEmpirical{\iConstraint}^\top}{\sqrt{\EigenValueEmpirical{\iConstraint} + \TestFunctionReg}} 
\underbrace{#5 #3}_{\defeq y} \Big]^2 \\
& = 
y^\top \Big(\SumOverConstraints 
\frac{\EigenVectorEmpirical{\iConstraint}\EigenVectorEmpirical{\iConstraint}^\top}{\EigenValueEmpirical{\iConstraint} + \TestFunctionReg}\Big) y \\
& = 
y^\top \Big( \CovarianceEmpiricalStandard + \TestFunctionReg\Identity \Big)^{-1} y \\
& \gtrsim 
y^\top \CovarianceStandardReg^{-1} y.
\end{align*}
The last inequality holds via \fullref{lem:CovarianceConcentration}
with probability at least $1-\FailureProbability$ since
$
\TestFunctionReg{}
$
is a large enough regularizer.
Let us complete the quadratic form:
\begin{align*}
	\norm{y + \TestFunctionReg #3 }
	{\CovarianceStandardReg^{-1}}^2
	\leq 
	(
	\norm{y}
	{\CovarianceStandardReg^{-1}}
	+
	\TestFunctionReg \norm{#3 }
	{\CovarianceStandardReg^{-1}} 
	)^2
	& \lesssim
	\norm{y}
	{\CovarianceStandardReg^{-1}}^2
	+
	\TestFunctionReg.
\end{align*}
Therefore, adding $\TestFunctionReg$ to both sides of the prior display
and noticing that 
\hidecom{$\info{\TestFunctionReg \lesssim \frac{\ConfidenceInterval{\policy}^2}{\nSamples}}$}
$\TestFunctionReg \lesssim \frac{\ConfidenceInterval{\policy}^2}{\nSamples}$
gives
\begin{align*}
\dim \frac{\ConfidenceInterval{\FailureProbability}^2}{\nSamples} 
& \gtrsim 
\norm{y + \TestFunctionReg #3 }
	{\CovarianceStandardReg^{-1}}^2 \\
	& =  
	#3 #2 ^\top \Big( \CovarianceStandardReg^{-1} \Big) #2  
	#3 \\
	& =
	#3(#4)#3 \\
	& = 
	\norm{#3}{#4}^2.
\end{align*}
}

\begin{document}

\begin{center}
  {\bf{\LARGE{
        Bellman Residual Orthogonalization \\
for Offline Reinforcement Learning}}}

  \vspace*{0.5in}

\begin{tabular}{lcl}
  Andrea Zanette$^\star$ 
  && Martin J. Wainwright$^{\star,\dagger}$ \\
  \texttt{zanette@berkeley.edu} && \texttt{wainwrig@berkeley.edu}
\end{tabular}

\vspace*{0.10in}

\begin{tabular}{c}
  Department of Electrical Engineering and Computer
  Sciences$^{\star}$ \\
  Department of Statistics$^\dagger$ \\
  UC Berkeley, Berkeley, CA \\
  \\
  Department of Electrical Engineering and Computer
  Sciences$^{\dagger}$ \\
  Department of Mathematics$^\dagger$ \\
  Massachusetts Institute of Technology, Cambridge, MA
\end{tabular}

  \vspace*{0.5in}

  \begin{abstract}
    We propose and analyze a reinforcement learning principle that
approximates the Bellman equations by enforcing their validity only
along an user-defined space of test functions.  Focusing on
applications to model-free offline RL with function approximation, we
exploit this principle to derive confidence intervals for off-policy
evaluation, as well as to optimize over policies within a prescribed
policy class.  We prove an oracle inequality on our policy
optimization procedure in terms of a trade-off between the value and
uncertainty of an arbitrary comparator policy.  Different choices of
test function spaces allow us to tackle different problems within a
common framework.  We characterize the loss of efficiency in moving
from on-policy to off-policy data using our procedures, and establish
connections to concentrability coefficients studied in past work.  We
examine in depth the implementation of our methods with linear
function approximation, and provide theoretical guarantees with
polynomial-time implementations even when Bellman closure does not
hold.

  \end{abstract}
\end{center}

\section{Introduction}

Markov decision processes (MDP) provide a general framework for
optimal decision-making in sequential settings
(e.g.,~\cite{puterman1994markov,Bertsekas_dyn1,Bertsekas_dyn2}).
Reinforcement learning refers to a general class of procedures for
estimating near-optimal policies based on data from an unknown MDP
(e.g.,~\cite{bertsekas1996neuro,sutton2018reinforcement}).  Different
classes of problems can be distinguished depending on our access to
the data-generating mechanism.  Many modern applications of RL involve
learning based on a pre-collected or offline dataset.  Moreover, the
state-action spaces are often sufficiently complex that it becomes
necessary to implement function approximation.  In this paper, we
focus on model-free offline reinforcement learning (RL) with function
approximation, where prior knowledge about the MDP is encoded via the
value function.  In this setting, we focus on two fundamental
problems: (1) offline policy evaluation---namely, the task of
accurately predicting the value of a target policy; and (2) offline
policy optimization, which is the task of finding a high-performance
policy.

There are various broad classes of approaches to off-policy
evaluation, including importance
sampling~\cite{precup2000eligibility,thomas2016data,jiang2016doubly,liu2018breaking},
as well as regression-based
methods~\cite{lagoudakis2003least,munos2008finite,chen2019information}.
Many methods for offline policy optimization build on these
techniques, with a line of recent papers including the addition of
pessimism~\cite{jin2021pessimism,xie2021bellman,zanette2021provable}.
We provide a more detailed summary of the literature
in~\cref{sec:Literature}.

In contrast, this work investigates a different model-free
principle---different from importance sampling or regression-based
methods---to learn from an offline dataset.  It belongs to the class
of weight learning algorithms, which leverage an auxiliary function
class to either encode the marginalized importance weights of the
target policy~\cite{liu2018breaking,xie2020Q}, or estimates of the
Bellman errors~\cite{antos2008learning,chen2019information,xie2020Q}.
Some work has considered kernel classes~\cite{feng2020accountable} or
other weight classes to construct off-policy
estimators~\cite{uehara2020minimax} as well as confidence intervals at
the population level~\cite{jiang2020minimax}.  However, these works do
not examine in depth the statistical aspects of the problem, nor
elaborate upon the design of the weight function classes.\footnote{For
instance, the paper~\cite{feng2020accountable} only shows validity of
ther intervals, not a performance bound; on the other hand, the
paper~\cite{jiang2020minimax} gives analyses at the population level,
and so does not address the alignment of weight functions with respect
to the dataset in the construction of the empirical estimator, which
we do via self-normalization and regularization.  This precludes
obtaining the same type of guarantees that we present here.}  The last
two considerations are essential to obtaining data-dependent
procedures accompanied by rigorous guarantees, and to provide guidance
on the choice of weight class, which are key contributions of this
paper.

For space reasons, we motivate our approach in the idealized case
where the Bellman operator is known in~\cref{sec:appBRO}, and compare
with the weight learning literature at the population level
in~\cref{sec:WeightLearning}.  Let us summarize our main contributions
in the following three paragraphs.


\paragraph{Conceptual contributions} Our paper makes two novel
contributions of conceptual nature:
\begin{enumerate}[leftmargin=*]
\item We propose a method, based on \emph{approximate empirical
orthogonalization} of the Bellman residual along test functions, to
  construct confidence intervals and to perform policy optimization.
\item We propose a sample-based approximation of such principle, based
  on \emph{self-normalization} and \emph{regularization}, and obtain
  general guarantees for parametric as well as non-parametric problems.
\end{enumerate}
The construction of the estimator, its statistical analysis, and the
concrete consequences (described in the next paragraph) are the major
distinctions with respect to past work on weight learning
methods~\cite{uehara2020minimax,jiang2020minimax}.  Our analysis
highlights the statistical trade-offs in the choice of the test
functions.  (See~\cref{sec:WeightLearning} for comparison with past
work at the population level.)

\paragraph{Domain-specific results}
In order to illustrate the broad effectiveness and applicability of
our general method and analysis, we consider several domains of
interest.  We show how to recover various results from past work---and
to obtain novel ones---by making appropriate choices of the test
functions and invoking our main result.  Among these consequences, we
discuss the following:
\begin{enumerate}[leftmargin=*]
\item When marginalized importance weights are available, they
can be used as test class.  In this case we recover a similar results
    	as the paper~\cite{xie2020Q}; however, here we only require
    	concentrability with respect to a comparator policy instead of
    	over all policies in the class.
\item When some knowledge of
    	the Bellman error class is available, it can be used as test
    	class.  Similar results have appeared previously either with
    	stronger concentrability~\cite{chen2019information} or in the
    	special case of Bellman closure~\cite{xie2021bellman}.
\item We provide a test class that projects the Bellman residual
along the error space of the $\Q$ class.  The resulting procedure is
as an extension of the LSTD algorithm \cite{bradtke1996linear} to
non-linear spaces, which makes it a natural approach if no
domain-specific knowledge is available.  A related result is the lower
bound by \cite{foster2021offline}, which proves that without Bellman
closure learning is hard even with small density ratios.  In contrast,
our work shows that learning is still possible even with large density
ratios.
\item Finally, our procedure inherits some form of ``multiple
robustness''.  For example, the two test classes corresponding to
Bellman completeness and marginalized importance weights can be used
together, and guarantees will be obtained if \emph{either} Bellman
completeness holds or the importance weights are correct.  We examine
this issue in \cref{sec:MultipleRobustness}.
\end{enumerate}

\paragraph{Linear setting} We examine in depth an application to the
linear setting, where we propose the first \emph{computationally
tractable} policy optimization procedure \emph{without assuming
Bellman completeness}.  The closest result here is given in the
paper~\cite{zanette2021provable}, which holds under Bellman closure.
Our procedure can be thought of making use of LSTD-type estimates so
as to establish confidence intervals for the projected Bellman
equations, and then using an iterative scheme for policy improvement.


\section{Background and set-up}
\label{sec:background}

We begin with some notation used throughout the paper.  For a given
probability distribution $\Distribution$ over a space $\mathcal{X}$,
we define the $L^2(\Distribution)$-inner product and semi-norm as
$\innerprodweighted{f_1}{f_2}{\Distribution} =
\Expecti{[f_1f_2]}{\Distribution}, $ and $ \norm{f_1}{\Distribution} =
\sqrt{\innerprodweighted{f_1}{f_1}{\Distribution}}{}$. The identity
function that returns one for every input is denoted by $\1$.  We
frequently use notation such as $c, c', \ctil, \cOne,\cTwo$ etc. to
denote constants that can take on different values in different
sections of the paper.

\subsection{Markov decision processes and Bellman errors}
 
We focus on infinite-horizon discounted Markov decision
processes~\cite{puterman1994markov,bertsekas1996neuro,sutton2018reinforcement}
with discount factor $\discount \in [0,1)$, state space $\StateSpace$,
and an action set $\ActionSpace$.  For each state-action pair $\psa$,
there is a reward distribution $\RewardLaw\psa$ supported in $[0,1]$
with mean $\reward\psa$, and a transition
$\TransitionLaw(\cdot \mid \state, \action)$.

A (stationary) stochastic policy $\policy$ maps states
to actions.  For a given policy, its
$Q$-function is the discounted sum of future
rewards based on starting from the pair $\psa$, and then following the
policy $\policy$ in all future time steps
$\Qfnc^{\pi}\psa = \reward\psa + \sum_{\hstep = 0}^{\infty}
\discount^{\hstep} \Exp [ \reward_{\hstep}({\SState}_\hstep,
  {\AAction}_\hstep) \mid (\SState_0, \AAction_0) = \psa],
$
where the expectation is taken over trajectories with
$
\AAction_{\hstep} \sim \pol(\cdot \mid \SState_\hstep), \quad
\mbox{and} \quad \SState_{\hstep+1} \sim \TransitionLaw(\cdot \mid
\SState_\hstep, \AAction_\hstep) \quad \mbox{for $\hstep = 1, 2,
  \ldots$.}
$
We also use $\Qpi{\policy}(\state,
\policy) = \Expecti{\Qpi{\policy}(\state, \AAction)}{\AAction
  \sim \policy(\cdot \mid \state)}$ 
  and define the
\emph{Bellman evaluation operator} as
$
(\BellmanEvaluation{\policy}\Qfnc)\psa = \reward\psa +
  \Expecti{\Qfnc(\SState^+ ,\policy)} {\SState^+ \sim
    \TransitionLaw(\cdot \mid \state, \action)}.
$
The value function satisfies
$
\Vpi{\policy}(\state) = \Qpi{\policy}(\state,\policy).
$
In our analysis, we assume that policies have action-value functions
that satisfy the uniform bound $\sup_{\psa} \abs{\Qpi{\policy}\psa}
\leq 1$.  
We are also interested in approximating optimal policies,  
whose value and action-value functions are defined as $\Vstar(\state) =
V^{\pistar}(\state) = \sup_{\policy} \Vpi{\policy}(\state)$ and
$\Qstar\psa = \Qfnc^{\pistar}\psa = \sup_{\policy}
\Qfnc^{\policy}\psa$.

We assume that the starting state $\SState_0$ is
drawn according to $\startdistribution$ and study $
  \Vpi{\policy} = \Exp_{\SState_0 \sim
    \startdistribution}[\Vpi{\policy}(\SState_0)]
$. We define the \emph{discounted
occupancy measure} associated with a policy $\policy$ as 
the distribution over the state action space
$
\DistributionOfPolicy{\policy}\psa = (1 - \discount )
\sum_{\hstep=0}^\infty \discount^\hstep \Pro_\hstep[
  (\SState_\hstep,\AAction_\hstep) = \psa ].
$
We adopt the shorthand notation $\Exp_\policy$ for
expectations over $\DistributionOfPolicy{\policy}$.  
For any functions
$f, g: \StateSpace \times \ActionSpace \rightarrow \R$, we make
frequent use of the shorthands
$
\Exp_\policy[f] \defeq \Exp_{(\SState, \AAction) \sim
  \DistributionOfPolicy{\policy}}[f(\SState, \AAction)], \quad
\mbox{and} \quad \inprod{f}{g}_\policy \defeq \Exp_{(\SState,
  \AAction) \sim \DistributionOfPolicy{\policy}} \big[f(\SState,
  \AAction) \, g(\SState, \AAction) \big].
$
Note moreover that we have $\inprod{\1}{f}_\policy = \Exp_\policy[f]$
where $\1$ denotes the identity function.

For a given $\Qfnc$-function and
policy $\policy$, let us define the \emph{temporal difference error}
(or TD error) associated with the sample $\sarsizNp{} = \sars{}$ and
the \emph{Bellman error} at $\psa$
\begin{align*}
(\TDError{\Qfnc}{\policy}{})\sarsiz{} \defeq \Qfnc\psa - \reward -
  \gamma \Qfnc(\successorstate,\policy), \qquad
  (\BellmanError{\Qfnc}{\policy}{})\psa \defeq \Qfnc\psa - \reward\psa
  - \gamma
  \E_{\successorstate\sim\TransitionLaw\psa}\Qfnc(\successorstate,\policy).
\end{align*}
The TD error is a random variable function of $\sarsizNp{}$, while the
Bellman error is its conditional expectation with respect to the
immediate reward and successor state at $\psa$.  Many of our bounds
involve the quantity
$\Expecti{\BellmanError{\Qfnc}{\policy}{}}{\policy} =
\Expecti{\big[\BellmanError{\Qfnc}{\policy}{} \PSA \big]} {\PSA \sim
  \DistributionOfPolicy{\policy}}.$
  
Finally, we introduce the data generation mechanism.
A more general sampling model is described in \cref{sec:GeneralGuarantees}.

\begin{assumption}[I.i.d. dataset]
  \label{asm:IIDDataset}
  An i.i.d. dataset is a collection $\Dataset = \{ \sarsi{\iSample}
  \}_{i=1}^\numobs$ such that for each $i = 1, \ldots, \numobs$ we
  have
$(\state_i, \action_i, \identifier_i) \sim
  \DatasetDistributionStateActions$ and conditioned on
  $(\state_\iSample,\action_\iSample,\identifier_\iSample )$, we
  observe a noisy reward $\reward_\iSample =
  \reward(\state_\iSample,\action_\iSample) + \eta_i$ with $\E[\eta_i
    \mid \Filtration_{\iSample} ] = 0, \; |\reward_i| \leq 1$ and the
  next state $\successorstate_\iSample\sim
  \TransitionLaw(\state_\iSample,\action_\iSample)$.
\end{assumption}

\subsection{Function Spaces and Weak Representation}
 
Our methods involve three different types of function spaces,
corresponding to policies, action-value functions, and test
functions. A test function $\TestFunction{}$ is a mapping
$\psai \mapsto \TestFunction{}\psai$ such
that \mbox{$\sup_{\psai}\abs{\TestFunction{}\psai} \leq 1$}, where
$\identifier$ is an optional identifier containing side information.
Our methodology involves the following three function classes:
\bcar
\item a \emph{policy class} $\PolicyClass$ that contains all policies
  $\policy$ of interest (for evaluation or optimization);
\item for each $\policy$, the \emph{predictor class}
  $\Qclass{\policy}$ of action-value functions $Q$ that we permit; and
\item for each $\policy$, the \emph{test function class}
  $\TestFunctionClass{\policy}$ that we use to enforce the Bellman
  residual constraints.
\ecar
We use the shorthands $\Qclass{}
= \cup_{\policy \in \PolicyClass}\Qclass{\policy}$ and
\mbox{$\TestFunctionClass{} = \cup_{\policy \in
    \PolicyClass} \TestFunctionClass{\policy}$}.
We assume \emph{weak realizability}:
\begin{assumption}[Weak Realizability]
\label{asm:WeakRealizability}
For a given policy $\policy$, the predictor class $\Qclass{\policy}$
is weakly realizable with respect to the test space
$\TestFunctionClass{\policy}$ and the measure
$\DatasetDistributionStateActions$ if there exists a predictor
$\QpiWeak{\policy} \in \Qclass{\policy}$ such that
\begin{align}
\label{EqnWeakRealizable}  
\innerprodweighted{\TestFunction{}}
                  {\BellmanError{\QpiWeak{\policy}}{\policy}{}}
                  {\DatasetDistributionStateActions} = 0 \; \text{for
                    all } \TestFunction{} \in
                  \TestFunctionClass{\policy} \qquad \text{and} \qquad
                  \innerprodweighted{\1}{\BellmanError{\QpiWeak{\policy}}{\policy}{}}{\policy}
                  = 0.
	\end{align}
\end{assumption}
The first condition requires the predictor to satisfy the Bellman
equations \emph{on average}.  The second condition amounts
to requiring that the predictor returns the value of $\policy$ at the
start distribution: using \cref{lem:Simulation} stated in the sequel,
we have
\begin{align*}
\Expecti{\QpiWeak{\policy}(\MyState ,\policy)}{\MyState \sim
  \startdistribution} - \Vpi{\policy} =
\Expecti{[\QpiWeak{\policy}-\Qpi{\policy}](\MyState
  ,\policy)}{\MyState \sim \startdistribution} = \horizon
\Expecti{\BellmanError{\QpiWeak{\policy}}{\policy}{}}{\policy} =
\horizon
\innerprodweighted{\1}{\BellmanError{\QpiWeak{\policy}}{\policy}{}}{\policy}
= 0.
\end{align*}
This weak notion should be contrasted with \emph{strong
realizability}, which requires a function 
$\Qpi{\policy} \in \Qclass{\policy}$ that satisfies the Bellman equation in
all state-action pairs.  

A stronger assumption that we sometime use is Bellman closure,
which requires that
$
\BellmanEvaluation{\policy}(\Qfnc) \in \Qclass{\policy} \;
\mbox{for all $\Qfnc \in \Qclass{\policy}$.}
$
The corresponding `weak' version is given in \cref{sec:appWeakClosure}.

\section{Policy Estimates via the Weak Bellman Equations}
 
\label{sec:Algorithms}

In this section, we introduce our high-level approach, first at the
population level and then in terms of regularized/normalized
sample-based approximations.

\subsection{Weak Bellman equations, empirical approximations and
  confidence intervals}

We begin by noting that the predictor $\Qpi{\policy}$ satisfies the
Bellman equations everywhere in the state-action space, i.e.,
$\BellmanError{\Qpi{\policy}}{\policy}{} = 0$.  However, if our
dataset is ``small'' relative to the complexity of (functions) on the
state-action space, then it is unrealistic to enforce such a stringent
condition.  Instead, the idea is to control the Bellman error in a
weighted-average sense, where the weights are given by a set of
\emph{test functions}.  At the idealized population level
(corresponding to an infinite sample size), we consider predictors
that satisfy the conditions
\begin{align}
\label{eqn:PoP}
\innerprodweighted{\TestFunction{}}{\BellmanError{\Qfnc}{\policy}{}}{\DatasetDistributionStateActions}
= 0, \qquad \text{for all} \; \TestFunction{} \in
\TestFunctionClass{\policy}.
\end{align}
where $\TestFunctionClass{\policy}$ is a user-defined set of test
functions.  The two main challenges here are how to use data to
enforce an approximate version of such constraints (along with
rigorous  data-dependent guarantees), and how to
design the test function space.  We begin with the former challenge.
\medskip

\paragraph{Construction of the empirical set} Given a dataset
$\Dataset = \{(\state_i, \action_i, \reward_i, \successorstate_i, \identifier_i)
\}_{i=1}^\numobs$, we can approximate the Bellman errors by a linear
combination of the temporal difference errors:
\begin{align}
\label{eqn:WeightedTemporaResidual}
\int \TestFunction{}\psa \underbrace{[\Qfnc\psa -
    (\BellmanEvaluation{\policy} \Qfnc)\psa]}_{=
  \BellmanError{\Qfnc}{\policy}{}\psa }
d\DatasetDistributionStateActions \approx \frac{1}{\nSamples}
\sum_{i=1}^\numobs \TestFunction{}\psaindex \underbrace{
  [\Qfnc\psaindex - \reward_i - \discount
    \Qfnc(\successorstate_i,\policy) ]}_{=
  \TDError{\Qfnc}{\policy}{}\sarsi{i} }.
\end{align}
Note that the approximation~\eqref{eqn:WeightedTemporaResidual}
corresponds to a weighted linear combination of temporal differences.
Written more compactly in inner product notation,
equation~\eqref{eqn:WeightedTemporaResidual} reads
$\innerprodweighted{\TestFunction{}}{\BellmanError{\Qfnc}{\policy}{}}{\mudist}
\approx
\innerprodweighted{\TestFunction{}}{\TDError{\Qfnc}{\policy}{}}{\numobs}$,
where
$
\innerprodweighted{f}{g}{\numobs} = 
\frac{1}{\nSamples}\sum_{\sarsi{} \in \Dataset}
(fg)\sarsi{}.
$

In general, the action value function $\Qpi{\policy}$ does not satisfy
$\inprod{\TestFunction{}}{ \TDError{\Qpi{\policy}}{\policy}{}}_\numobs
= 0$ because the empirical
approximation~\eqref{eqn:WeightedTemporaResidual} involves sampling
error.  For these reasons, in order to (approximately) identify
$\Qpi{\policy}$, we impose only inequalities.  Given a class of test
functions $\TestFunctionClass{\policy}$, a radius parameter $\Rad \geq
0$ and regularization parameter $\regpar \geq 0$, we define the set
\begin{align}
\label{eqn:EmpiricalBellmanGalerkinEquations}
\SuperEmpiricalFeasible \defeq \left \{ \Qfnc \in \Qclass{\policy}
\quad \text{such that} \quad
\frac{
  \abs{\innerprodweighted{\TestFunction{}}{\TDError{\Qfnc}{\pi}{}}{\nSamples}}
} { \TestNormaRegularizerEmp{} } \leq \sqrt{\frac{\Rad}{\nSamples}}
\quad \mbox{for all $\testfunc \in \TestFunctionClass{\policy}$}\right \}.
\end{align}
When the choices of $(\Rad, \regpar)$ are clear from the context, we
adopt the shorthand $\EmpiricalFeasibleSet{\policy}(
\TestFunctionClass{\policy})$, or $\EmpiricalFeasibleSet{\policy}$
when the function class $\TestFunctionClass{\policy}$ is also clear.
If $\TestFunctionClass{\policy}$ and $\Qclass{\policy}$ have finite
cardinality, $\Rad \approx \ln
|\TestFunctionClass{\policy}||\Qclass{\policy}| + \ln
1/\FailureProbability$, where $\FailureProbability$ is a prescribed
failure probability.

Our definition of the empirical constraint
set~\eqref{eqn:EmpiricalBellmanGalerkinEquations} has two key
components: first, the division by $\TestNormaRegularizerEmp{}$
corresponds to a form of \emph{self-normalization}, whereas the
addition of $\lambda$ corresponds to a form of \emph{regularization}.
Self-normalization is needed so that the constraints remain suitably
scale-invariant.  More importantly---in conjunction with the
regularization---it ensures that test functions that have poor
coverage under the dataset do not have major effects on the solution.
In particular, the empirical norm $\EmpNorm{f}^2$ in the
self-normalization measures how well the given test function is
covered by the dataset.  Any test function with poor coverage (i.e.,
$\EmpNorm{f}^2 \approx 0$) will not yield useful information, and the
regularization counteracts its influence.  In our guarantees, the
choices of $\lambda$ and $\rho$ are critical; as shown in our theory,
we typically have $\lambda = \rho/\numobs$, where $\rho$ scales with
the metric entropy of the predictor, test and policy spaces.
Disregarding $\rho$, the right-hand side of the constraint decays as
$1/\sqrt{\nSamples}$, so that the constraints are
enforced more tightly as the sample size increases.


\paragraph{Confidence bounds and policy optimization:}
First, for any fixed policy $\policy$, we can use the feasibility
set~\eqref{eqn:EmpiricalBellmanGalerkinEquations} to compute the lower
and upper estimates
\begin{align}
\label{eqn:ConfidenceIntervalEmpirical}
\VminEmp{\policy} \defeq \min_{\Qfnc \in \SuperEmpiricalFeasible}
\Expecti{\big[\Qfnc(\MyState ,\policy)\big]}{\MyState \sim
  \startdistribution}, \; \text{and} \quad \VmaxEmp{\policy}
\defeq \max_{\Qfnc \in \SuperEmpiricalFeasible}
\Expecti{\big[\Qfnc(\MyState ,\policy)\big]}{\MyState \sim
  \startdistribution},
\end{align} 
corresponding to estimates of the minimum and maximum value that the
policy $\policy$ can take at the initial distribution.
The family of lower estimates can be used to perform
policy optimization over the class $\PolicyClass$, in particular by
solving the \emph{max-min} problem
\begin{align}
\label{eqn:MaxMinEmpirical}
  \max_{\pInP} \Big [\min_{\Qfnc \in \EmpiricalFeasibleSet{\policy}}
    \Expecti{\Qfnc(\MyState ,\policy)}{\MyState \sim
      \startdistribution} \Big], \qquad \text{or equivalently} \qquad
  \max_{\pInP}\VminEmp{\policy}.
\end{align}

\paragraph{Form of guarantees}

Let us now specify and discuss the types of guarantees that we
establish for our estimators~\eqref{eqn:ConfidenceIntervalEmpirical}
and~\eqref{eqn:MaxMinEmpirical}.  All of our theoretical guarantees
involve a $\mudist$-based counterpart
$\PopulationFeasibleSet{\policy}$ of the data-dependent set
$\EmpiricalFeasibleSet{\policy}$.  
More precisely, we define the population set
\begin{align}
\label{eqn:ApproximateBellmanGalerkingEquations}
\SuperPopulationFeasible \defeq \biggr\{ \Qfnc \in \Qclass{\policy}
\quad \text{such that} \quad \frac{
  \abs{\innerprodweighted{\TestFunction{}}{\BellmanError{\Qfnc}{\pi}{}}{\DatasetDistributionStateActions}}
} { \TestNormaRegularizerPop{} } \leq \sqrt{\frac{4 \Rad}{\nSamples}}
\qquad \mbox{for all $\testfunc \in \TestFunctionClass{}$} \biggr\},
\end{align}
where $\inprod{f}{g}_\DatasetDistributionStateActions \defeq \int f
\psa g \psa d \mu$ is the inner product induced by a
distribution\footnote{See Section~\ref{SecDataGen} for a precise
definition of the relevant $\mu$ for a fairly general sampling model.}
$\mu$ over $\psa$.  As before, we use the shorthand notation
$\PopulationFeasibleSet{\policy}$ when the underlying arguments are
clear from context.  Moreover, in the sequel, we generally ignore the
constant $4$ in the
definition~\eqref{eqn:ApproximateBellmanGalerkingEquations} by
assuming that $\Rad$ is rescaled appropriately---e.g., that we use a
factor of $\frac{1}{4}$ in defining the empirical set.

It should be noted that in contrast to the set
$\EmpiricalFeasibleSet{\policy}$, the set
$\PopulationFeasibleSet{\policy}$ is \emph{non-random}
and it is defined in terms of the
distribution $\DatasetDistributionStateActions$ and the input space
$\InputFunctionalSpace$.  It relaxes the orthogonality constraints in
the weak Bellman formulation~\eqref{eqn:PoP}.
Our guarantees for off-policy confidence intervals take the following form:
\begin{subequations}
\label{EqnPolEval}  
\begin{align}
\label{EqnCoverage}
\mbox{\underline{Coverage guarantee:}} & \qquad \big[
  \VminEmp{\policy}, \VmaxEmp{\policy} \big] \ni \Vpi{\policy}. \\
  \label{EqnWidthBound}
\mbox{\underline{Width bound:}} & \qquad \max \Big \{
|\VminEmp{\policy} - \Vpi{\policy} |, \; |\VmaxEmp{\policy} -
\Vpi{\policy} | \Big \} \leq \horizon \max_{\Qfnc \in
  \PopulationFeasibleSet{\policy}(\TestFunctionClass{\policy})}
|\PolComplex|.
  \end{align}
\end{subequations}
Turning to policy optimization, let $\piemp$ be a solution to the
max-min criterion~\eqref{eqn:MaxMinEmpirical}.  Then we prove a result
of the following type:
\begin{align}
\label{EqnOracle}
\mbox{\underline{Oracle inequality:}} \qquad \Vpi{\piemp} \geq
\max_{\policy \in \PolicyClass} \Big \{ \underbrace{ \vphantom{
    \frac{1}{1 - \discount} \max_{\Qfnc \in
      \PopulationFeasibleSet{\policy}(\TestFunctionClass{})}
    |\PolComplex| } \Vpi{\policy}}_{\mbox{\tiny{Value}}} -
\underbrace{\frac{1}{1 - \discount} \max_{\Qfnc \in
    \PopulationFeasibleSet{\policy}(\TestFunctionClass{})}
  |\PolComplex|}_{\mbox{\tiny{Evaluation uncertainty}}} \Big \}.
\end{align}
Note that this result guarantees that the estimator competes
against an oracle that can search over all policies, and
select one based on the optimal trade-off between its value
and evaluation uncertainty.



\subsection{High-probability guarantees}

In this section, we present some high-probability guarantees.  So as
to facilitate understanding under space constraints, we state here
results under simplifying assumptions: (a) the dataset originates from
a fixed distribution, and (b) the classes $\PolicyClass{}, \TestFunctionClass{}$ and
$\Qclass{}$ have finite cardinality.  We emphasize
that~\cref{sec:GeneralGuarantees} provides a far more general version
of this result, with an extremely flexible sampling model, and
involving metric entropies of parametric or non-parametric function
classes.

\begin{theorem}[Guarantees for finite classes]
\label{thm:NewPolicyEvaluationFiniteClasses}
Consider 
a triple $\InputFunctionClass$
that is weakly Bellman realizable (\cref{asm:WeakRealizability}); an
i.i.d. dataset (\Cref{asm:IIDDataset}); and the choices $\Rad = c
\big \{ \log (|\TestFunctionClass{}| |\PolicyClass | |\Qclass{}|) +
\log (1/\FailureProbability) \big \}$ and $\regpar = c'
\Rad/\nSamples$ for some constants $c, c'$.  Then w.p. at least $1 -
\delta$:
\bcar
\item \underline{Policy evaluation:} For any $\pi \in \PolicyClass$,
  the estimates $(\VminEmp{\policy}, \VmaxEmp{\policy})$ specify a
  confidence interval satisfying the coverage~\eqref{EqnCoverage} and
  width bounds~\eqref{EqnWidthBound}
  \item \underline{Policy optimization:} Any max-min
    policy~\eqref{eqn:MaxMinEmpirical} $\piemp$ satisfies the oracle
    inequality~\eqref{EqnOracle}.  \ecar
\end{theorem}


\section{Concentrability Coefficients and Test Spaces}
\label{sec:Applications}

In this section, we develop some connections to concentrability
coefficients that have been used in past work, and discuss various
choices of the test class.  Like the predictor class
$\Qclass{\policy}$, the test class $\TestFunctionClass{\policy}$
encodes domain knowledge, and thus its choice is delicate.  Different
from the predictor class, the test class does not require a
`realizability' condition.  As a general principle, the test functions
should be chosen as orthogonal as possible with respect to the Bellman
residual, so as to enable rapid progress towards the solution; at the
same time, they should be sufficiently ``aligned'' with the dataset,
meaning that
$\norm{\TestFunction{}}{\DatasetDistributionStateActions}$ or its
empirical counterpart $\norm{\TestFunction{}}{\numobs}$ should be
large.  Given a test class, each additional test function posits a new
constraint which helps identify the correct predictor; at the same
time, it increases the metric entropy (parameter $\Rad$), which makes
each individual constraints more loose.  In summary, there are
trade-offs to be made in the selection of the test class
$\TestFunctionClass{}$, much like $\Qclass{}$.

In order to assess the statistical cost that we pay
for off-policy data, it is natural to define the \emph{\goodname}
(OPC) as
\begin{align}
\label{EqnConcSimple}  
  \ConcSimple(\PopulationFeasibleSet{\policy}, \Rad, \regpar) & \defeq
  \max_{\Qfnc \in \PopulationFeasibleSet{\policy}}
  \frac{|\PolComplex|^2}{(1 + \regpar) \frac{\Rad}{\numobs}} 
  =   
  \max_{\Qfnc \in \PopulationFeasibleSet{\policy}}
  \frac{\innerprodweighted{\1}{\BellmanError{\Qfnc}{\policy}{} }{\policy} ^2}{(1 + \regpar) \frac{\Rad}{\numobs}},
\end{align}

With this notation, our off-policy width bound~\eqref{EqnWidthBound}
can be re-expressed as
\begin{subequations}
\begin{align}
\label{eqn:ConcreteCI}
\abs{\VminEmp{\policy} - \VmaxEmp{\policy}} \leq 2 \frac{\sqrt{1 +
    \regpar}}{1-\discount} \sqrt{\ConcSimple \frac{ \Rad}{\numobs}},
\end{align}
while the oracle inequality~\eqref{EqnOracle} for policy optimization
can be re-expressed in the form
\begin{align}
\label{EqnConcSimpleBound}  
\Vpi{\piemp} \geq \max_{\policy \in \PolicyClass} \Big\{ \Vpi{\policy}
- \frac{\sqrt{1 + \regpar}}{1-\discount} \sqrt{\ConcSimple \frac{
    \Rad}{\numobs}} \Big \},
\end{align}
\end{subequations}
Since $\regpar \sim \Rad/\nSamples$, the factor $\sqrt{1 + \regpar}$
can be bounded by a constant in the typical case $\nSamples \geq \Rad$.
We now offer concrete examples of the \shortgoodname,
while deferring further examples to \cref{sec:appConc}.

\subsection{Likelihood ratios}

Our broader goal is to obtain
small Bellman error along the distribution induced by $\policy$.  
Assume that one constructs a test function class
$\TestFunctionClass{\policy}$ of possible likelihood ratios.  

\begin{proposition}[Likelihood ratio bounds]
\label{prop:LikeRatio}
Assume that for some constant $\scaling{\policy}$, the test function
defined as $\TestFunction{}^*\psa = \frac{1}{\scaling{\policy}}
\frac{\DistributionOfPolicy{\policy}\psa}{\DatasetDistributionStateActions\psa}$
belongs to $\TestFunctionClass{\policy}$ and satisfies
$\norm{\TestFunction{}^*}{\infty} \leq 1$.  
Then the {\shortgoodname}
coefficient satisfies
\begin{align}
\label{EqnLikeRatioBound}  
  \ConcentrabilityGeneric{\policy} \stackrel{(i)}{\leq} \frac{ \Expecti{\Big[\frac{\DistributionOfPolicy{\policy}\PSA}{\DatasetDistributionStateActions\PSA}\Big]
    }{\policy} + \scalingsq{\policy} \regpar} {1
    + \regpar} \; \stackrel{(ii)}{\leq} \l \frac{\scaling{\policy} \big(1
    + \scaling{\policy} \regpar \big)}{1 + \regpar}.
\end{align}
\end{proposition}
Here $\scaling{\policy}$ is a scaling parameter that ensures 
$\norm{\TestFunction{}^*}{\infty} \leq 1$.
Concretely one can take 
$\scaling{\policy} = \sup_{(\state,\action)} 
\frac{\DistributionOfPolicy{\policy}(\state,\action)}{\DatasetDistributionStateActions(\state,\action)}$.

The proof is in \cref{sec:LikeRatio}.
Since $\regpar = \regpar_\numobs
\rightarrow 0$ as $\nSamples$ increases,
the {\shortgoodname} coefficient is
bounded by a multiple of the expected ratio
$\Expecti{\Big[\frac{\DistributionOfPolicy{\policy}\PSA}{\DatasetDistributionStateActions\PSA}\Big]}{\policy}{}$.
Up to an additive offset, this expectation is equivalent to the
$\chi^2$-distribution between the policy-induced occupation measure
$\DistributionOfPolicy{\policy}$ and data-generating distribution
$\DatasetDistributionStateActions$.
The concentrability coefficient can be plugged back into
\cref{eqn:ConcreteCI,EqnConcSimpleBound}
 to obtain a concrete policy optimization bound.  In this case, we
recover a result similar to \cite{xie2020Q}, but with a much milder
concentrability coefficient that involves only the chosen comparator
policy.



\subsection{The error test space}
 
\label{sec:ErrorTestSpace}

We now turn to the discussion of a choice for the test space that
extends the LSTD algorithm to non-linear spaces.  A simplification to
the linear setting is presented later in \cref{sec:Linear}.

As is well known, the LSTD algorithm \cite{bradtke1996linear} can be
seen as minimizing the Bellman error projected onto the linear
prediction space $\Qfnc$. Define the transition operator $
(\TransitionOperator{\policy}\Qfnc)\psa =
\Expecti{\Qfnc(\successorstate,\policy ) } {\successorstate \sim
  \TransitionLaw\psa} $, and the prediction error $\QfncErr = \Qfnc -
\QpiWeak{\policy}$, where $\QpiWeak{\policy}$ is a $Q$-function from
the definition of weak realizability.  The Bellman error can be
re-written as $ \BellmanError{\Qfnc}{\policy}{} =
\BellmanError{\Qfnc}{\policy}{} -
\BellmanError{\QpiWeak{\policy}}{\policy}{} = (\IdentityOperator -
\discount\TransitionOperator{\policy})\QfncErr $.  When realizability
holds, in the linear setting and at the population level, the LSTD
solution seeks to satisfy the projected Bellman equations
\begin{align}
   \innerprodweighted{\TestFunction{}}
                     {\BellmanError{\Qfnc}{\policy}{}}
                     {\DatasetDistributionStateActions}
   = 0,
   \quad
   \text{for all $\TestFunction{} \in \QclassErrCentered{\policy}$}.
\end{align}
In the linear case, $\QclassErrCentered{\policy} $ is the class of
linear functions $\Qclass{\policy}$ used as predictors; when
$\Qclass{\policy}$ is non-linear, we can extend the LSTD method by
using the (nonlinear) error test space
$\TestFunctionClass{\policy} 
= \QclassErrCentered{\policy}
= \{\Qfnc - \QpiWeak{\policy} \}$.  
Since
$\QclassErrCentered{\policy}$ is unknown (as it depends on the weak
solution $\QpiWeak{\policy}$), we choose instead the larger class
\begin{align*}
  \QclassErr{\policy} = \{ \Qfnc - \Qfnc' \mid \Qfnc,\Qfnc' \in
  \Qclass{\policy} \},
\end{align*}
which contains $\QclassErrCentered{\policy}$.  
The resulting approach can be seen as performing a projection of the Bellman operator
$\BellmanError{\Qfnc}{\policy}{}$ into the
error space $\QclassErrCentered{\policy}$, 
much like LSTD does in the linear setting.
However, different from LSTD, our procedure returns confidence intervals 
as opposed to a point estimator.  
This choice of the test space is related to
the Bubnov-Galerkin method~\cite{repin2017one} for linear spaces; it
selects the test space $\TestFunctionClass{\policy}$ to be identical
to the trial space $\QclassErrCentered{\policy}$ that contains all
possible solution errors.
\begin{lemma}[\shortgoodname coefficient from prediction error]
  \label{lem:PredictionError}
For any test function class $\TestFunctionClass{\policy} \supseteq
\QclassErr{\policy}$, we have
\begin{align}
 \label{EqnPredErrorBound}
  \ConcSimple & \leq \max_{\Qfnc \in \Qclass{\policy}} \big \{ \frac{
    \norm{\QfncErr}{\DatasetDistributionStateActions}^2 +
    \TestFunctionReg } { \norm{\1}{\policy}^2 + \TestFunctionReg } \;
  \frac{ \innerprodweighted{\1} {\BellmanError{\Qfnc}{\policy}{}}
    {\policy}^2 } { \innerprodweighted{\QfncErr}
    {\BellmanError{\Qfnc}{\policy}{}}
    {\DatasetDistributionStateActions}^2 } \big \} = \max_{\QfncErr
    \in \QclassErrCentered{\policy}} \big \{ \frac{
    \norm{\QfncErr}{\DatasetDistributionStateActions}^2 +
    \TestFunctionReg } { \norm{\1}{\policy}^2 + \TestFunctionReg } \;
  \frac{ \innerprodweighted{\1} {(\IdentityOperator -
      \discount\TransitionOperator{\policy})\QfncErr} {\policy}^2 } {
    \innerprodweighted{\QfncErr} {(\IdentityOperator -
      \discount\TransitionOperator{\policy})\QfncErr}
                      {\DatasetDistributionStateActions}^2 } \big \}.
\end{align}
\end{lemma}
 
The above coefficient measures the ratio between the Bellman error 
along the distribution of the target policy $\policy$
and that projected 
onto the error space $\QclassErrCentered{\policy}$ defined by $\Qclass{\policy}$.
It is a concentrability coefficient that \emph{always} applies,
as the choice of the test space does not require domain knowledge. 
See~\cref{sec:PredictionError} for the proof,
and \cref{sec:appBubnov} for further comments and insights,
as well as a simplification in the special case of Bellman closure.

 
\subsection{The Bellman test space}
 
\label{sec:DomainKnowledge}
In the prior section we controlled the projected Bellman error.
Another longstanding approach in reinforcement learning 
is to control the Bellman error itself,
for example by minimizing the squared Bellman residual.
In general, this cannot be done if only an offline dataset
is available due to the well known \emph{double sampling} issue.
However, in some cases we can use an helper 
class to try to capture the Bellman error.
Such class needs to be a superset of the class of \emph{Bellman test functions} 
given by
\begin{align}
\label{EqnBellmanTest}
\TestFunctionClassBubnov{\policy} & \defeq \{
\BellmanError{\Qfnc}{\policy}{} \mid \Qfnc \in \Qclass{\policy} \}.
\end{align}
%
Any test class that contains the above
allows us to control the Bellman residual, as we show next.

\begin{lemma}[Bellman Test Functions]
\label{lem:BellmanTestFunctions}
For any test function class $ \TestFunctionClass{\policy}$ that
contains $\TestFunctionClassBubnov{\policy}$, we have
\begin{subequations}
  \begin{align}
    \label{EqnBellBound}
\norm{\BellmanError{\Qfnc}{\policy}{}}{\DatasetDistributionStateActions}
\leq \cOne \sqrt{\frac{\Rad}{\numobs}} \qquad \mbox{for any $\Qfnc
  \in \PopulationFeasibleSet{\policy}(\TestFunctionClass{\policy})$.}
  \end{align}
  Moreover, the {\goodname} is upper bounded as
  \begin{align}
\label{EqnBellBoundConc}    
\ConcSimple & \stackrel{(i)}{\leq} \cOne \sup_{\Qfnc \in
  \Qclass{\policy}} \frac{
  \innerprodweighted{\1}{\BellmanError{\Qfnc}{\policy}{}}{\policy}^2
}{
  \norm{\BellmanError{\Qfnc}{\policy}{}}{\DatasetDistributionStateActions}
  ^2} \stackrel{(ii)}{\leq} \cOne \sup_{\Qfnc \in \Qclass{\policy}}
\frac{ \norm{\BellmanError{\Qfnc}{\policy}{}}{\policy}^2 }{
  \norm{\BellmanError{\Qfnc}{\policy}{}}{\DatasetDistributionStateActions}
  ^2} \stackrel{(iii)}{\leq} \cOne \sup_{\psa}
\frac{\DistributionOfPolicy{\policy}\psa}
      {\DatasetDistributionStateActions\psa}.
  \end{align}
\end{subequations}
\end{lemma}
\noindent See~\cref{sec:BellmanTestFunctions} for the proof of this
claim.

Consequently, whenever the test class includes the Bellman test
functions, the {\goodname} is at most the ratio between the squared
Bellman residuals along the data generating distribution and the
target distribution.  
%
%
If Bellman closure holds, then the
prediction error space $\QclassErr{\policy}$ introduced in
\cref{sec:ErrorTestSpace} contains the Bellman test functions:
for $\Qfnc \in \Qclass{\policy}$, we can write
$
  \BellmanError{\Qfnc}{\policy}{} = \Qfnc -
  \BellmanEvaluation{\policy}\Qfnc 
  \in \QclassErr{\policy}
$.
This fact allows us to recover a result in the
recent paper~\cite{xie2021bellman} in the special case of Bellman closure,
although the approach presented here is more general.

\subsection{Combining test spaces}
\label{sec:MultipleRobustness}

Often, it is natural to construct a test space that is a union of
several simpler classes. 
A simple but valuable observation is that the resulting procedure inherits
the best of the \shortgoodname coefficients.  
Suppose that we are given a collection 
$\{ \TestFunctionClass{\policy}_m
\}_{m=1}^M$ of $M$ different test function classes, and define the
union $\TestFunctionClass{\policy} = \bigcup_{m=1}^M
\TestFunctionClass{\policy}_m$.  For each $m = 1, \ldots, M$, let
$\ConcSimple_m$ be the \shortgoodname coefficient defined by the
function class $\TestFunctionClass{\policy}_m$ and radius $\Rad$, and
let $\ConcSimple(\TestFunctionClass{})$ be the \shortgoodname
coefficient associated with the full class.  Then we have the
following guarantee:
\begin{lemma}[Multiple test classes]
  \label{prop:MultipleRobustness}
$
  \label{EqnMinBound}
\ConcSimple(\TestFunctionClass{}) \leq \min_{m = 1, \ldots, M}
\ConcSimple_m.
$
\end{lemma}
\noindent This guarantee is a straightforward consequence of our
construction of the feasibility sets: in particular, we have
$\PopulationFeasibleSet{\policy}(\TestFunctionClass{}) = \cap_{m =1}^M
\PopulationFeasibleSet{\policy}(\TestFunctionClass{}_m)$, and
consequently, by the variational definition of the {\goodname}
$\ConcSimple(\TestFunctionClass{})$ as optimization over
$\PopulationFeasibleSet{\policy}(\TestFunctionClass{})$, the
bound~\eqref{EqnMinBound} follows.  In words, when multiple test
spaces are combined, then our algorithms inherit the best (smallest)
\shortgoodname coefficient over all individual test spaces.  While
this behavior is attractive, one must note that there is a statistical
cost to using a union of test spaces: the choice of $\Rad$ scales as a
function of $\TestFunctionClass{}$ via its metric entropy.  This
increase in $\Rad$ must be balanced with the benefits of using
multiple test spaces.\footnote{For space reasons, we defer to
\cref{sec:IS2BC} an application in which we construct a test function
space as a union of subclasses, and thereby obtain a method that
automatically leverages Bellman closure when it holds, falls back to
importance sampling if closure fails, and falls back to a worst-case
bound in general.}


\section{Linear Setting}
\label{sec:Linear}

In this section, we turn to a detailed analysis of our estimators
using function classes that are linear in a feature map.  
Let $\LinPhi: \StateSpace \times \ActionSpace \rightarrow \R^\dim$ be
a given feature map, and consider linear expansions $g_{\CriticPar{}}
\psa \defeq \inprod{\CriticPar{}}{\LinPhi \psa} \; = \; \sum_{j=1}^d
\CriticPar{j} \LinPhi_j \psa$.  The class of \emph{linear functions}
takes the form
\begin{align}
\label{EqnLinClass}  
  \LinSpace & \defeq \{ \psa \mapsto g_{\CriticPar{}} \psa \mid
  \CriticPar{} \in \R^{\dim}, \; \norm{\CriticPar{}}{2} \leq 1 \}.
\end{align}
Throughout our analysis, we assume that $\norm{\LinPhi(\state,
  \action)}{2} \leq 1$ for all state-action pairs.

Following the approach in \cref{sec:ErrorTestSpace},
which is based on the LSTD method,
we should choose the test function class $\TestFunctionClass{\policy} = \LinClass$,
as in the linear case the prediction error is linear.

In order to obtain a computationally efficient implementation, we need
to use a test class that is a ``simpler'' subset of $\LinClass$.  In
particular, for linear functions, it is not hard to show that the
estimates $\VminEmp{\policy}$ and $\VmaxEmp{\policy}$ from
equation~\eqref{eqn:ConfidenceIntervalEmpirical} can be computed by
solving a quadratic program, with two linear constraints for each test
function.  (See~\cref{sec:LinearConfidenceIntervals} for the details.)
Consequently, the computational complexity scales linearly with the
number of test functions.  Thus, if we restrict ourselves to a finite
test class contained within $\LinClass$, we will obtain a
computationally efficient approach.  

\subsection{A computationally friendly test class and \shortgoodname coefficients}

Define the empirical covariance matrix $\CovarianceEmpiricalStandard =
\frac{1}{\nSamples} \SumOverSamples \LinPhiEmpirical{\iSample}
\LinPhiEmpirical{\iSample}^T$ \mbox{where $\LinPhiEmpirical{\iSample}
  \defeq \LinPhi(\state_\iSample,\action_\iSample)$.}  Let
$\{\EigenVectorEmpirical{j} \}_{j=1}^\dim$ be the eigenvectors of
empirical covariance matrix $\CovarianceEmpiricalStandard$, and
suppose that they are normalized to have unit $\ell_2$-norm.  We use
these normalized eigenvectors to define the finite test class
\begin{align}
\label{eqn:LinTestFunction}
\TestFunctionClassRandom{\policy} \defeq \{ f_j, j = 1, \ldots, \dim
\} \quad \mbox{where $f_j \psa \defeq
  \inprod{\EigenVectorEmpirical{j}}{\LinPhi \psa}$}
\end{align}
A few observations are in order:
\bcar
\item This test class has only $d$ functions, so that our QP
  implementation has $2 \dim$ constraints, and can be solved in
  polynomial time. (Again, see \cref{sec:LinearConfidenceIntervals}
  for details.)
\item Since $\TestFunctionClassRandom{\policy}$ is a subset of
  $\LinClass$ the choice of radius $\Rad = c(
  \frac{\dim}{\numobs} + \log 1/\delta)$ is valid for some constant $c$.
\ecar

\newcommand{\ActConc}{\ConcSimple(\TestFunctionClassRandom{\policy})}


\paragraph{Concentrability} When weak Bellman closure does not hold,
then our analysis needs to take into account how errors propagate via
the dynamics.  In particular, we define the \emph{next-state feature
extractor} $\LinPhiBootstrap{\policy} \psa \defeq
\Expecti{\LinPhi(\successorstate,\policy)}{\successorstate \sim
  \TransitionLaw\psa}$, along with the population covariance matrix
$\CovarianceStandard \defeq \E_\mudist \big[\LinPhi \psa \LinPhi^\top
  \psa \big]$, and its $\regpar$-regularized version
$\CovarianceStandardReg \defeq \CovarianceStandard + \TestFunctionReg
\Identity$.  We also define the matrices
\begin{align*}
\CovarianceBootstrap{\policy} \defeq
\Expecti{[\LinPhi(\LinPhiBootstrap{\policy})^\top]}{
  \DatasetDistributionStateActions}, \quad
\CovarianceWithBootstrapReg{\policy} \defeq
(\CovarianceStandardReg^{\frac{1}{2}} - \discount
\CovarianceStandardReg^{-\frac{1}{2}} \CovarianceBootstrap{\policy}
)^\top (\CovarianceStandardReg^{\frac{1}{2}} - \discount
\CovarianceStandardReg^{-\frac{1}{2}} \CovarianceBootstrap{\policy}).
\end{align*}
The matrix $\CovarianceBootstrap{\policy}$ is the cross-covariance
between successive states, whereas the matrix
$\CovarianceWithBootstrapReg{\policy}$ is a suitably renormalized and
symmetrized version of the matrix $\CovarianceStandard^{\frac{1}{2}} -
\discount \CovarianceStandard^{-\frac{1}{2}}
\CovarianceBootstrap{\policy}$, which arises naturally from the policy
evaluation equation.
We refer to quantities that contain evaluations at the next-state
(e.g., $\LinPhiBootstrap{\policy}$) as bootstrapping terms, and now
bound the \shortgoodname coefficient in the presence of such terms:
\begin{proposition}[\shortgoodname bounds with bootstrapping]
\label{prop:LinearConcentrability}
Under weak realizability, we have
\begin{align}
\label{EqnOPCBootstrap}  
  \ActConc & \leq c \; \dim \norm{\Expecti{[\LinPhi - \discount
        \LinPhiBootstrap{\policy}]}{\policy}}
           {(\CovarianceWithBootstrapReg{\policy})^{-1}}^2 \qquad
           \mbox{with probability at least $1-\FailureProbability$.}
\end{align}
\end{proposition}
\noindent See~\cref{sec:LinearConcentrability} for the proof. 
The bound~\eqref{EqnOPCBootstrap} takes a familiar form,
as it involves the same matrices used to define the LSTD
solution. This is expected, as our approach here is essentially
equivalent to the LSTD method; the difference is that
LSTD only gives a point estimate as
opposed to the confidence intervals that we present here;
however, they are both derived from the same principle, 
namely from the Bellman equations projected along the predictor (error) space.

The bound quantifies how the feature extractor
$\LinPhi$ together with the bootstrapping term
$\LinPhiBootstrap{\policy}$, averaged along the target policy
$\policy$, interact with the covariance matrix with bootstrapping
$\CovarianceWithBootstrapReg{\policy}$.  It is an approximation to the
\shortgoodname coefficient bound derived
in~\cref{lem:PredictionError}.  The bootstrapping terms capture the
temporal difference correlations that can arise in reinforcement
learning when strong assumptions like Bellman closure do not hold.  As
a consequence, such an OPC coefficient being small is a
\emph{sufficient} condition for reliable off-policy prediction.  This
bound on the OPC coefficient always applies, and it reduces to the
simpler one~\eqref{EqnOPCClosure} when
weak Bellman closure holds, with no need to inform the algorithm
of the simplified setting; 
see \cref{sec:LinearConcentrabilityBellmanClosure} for the proof. 

\begin{proposition}[\shortgoodname bounds under weak Bellman Closure]
\label{prop:LinearConcentrabilityBellmanClosure}
Under Bellman closure, we have
\begin{align}
\label{EqnOPCClosure}  
\ActConc & \leq c \; \dim \norm{\Expecti{
    \LinPhi}{\policy}}{\CovarianceStandardReg^{-1}}^2 \qquad
\mbox{with probability at least $1 - \delta$.}
\end{align}
\end{proposition}


\subsection{Actor-critic scheme for policy optimization}
\label{sec:LinearApproximateOptimization}

Having described a practical procedure to compute $\VminEmp{\policy}$,
we now turn to the computation of the max-min estimator for policy
optimization. We define the \emph{soft-max policy class}
\begin{align}
\label{eqn:SoftMax}
	\LinPolicyClass \defeq \Big\{ \psa \mapsto
        \frac{e^{\innerprod{\LinPhi\psa}{\ActorPar{}}}}{\sum_{\successoraction
            \in \ActionSpace}
          e^{\innerprod{\LinPhi(\state,\successoraction)}{\ActorPar{}}}}
        \mid \norm{\ActorPar{}}{2} \leq \ActorParNorm, \; \ActorPar{}
        \in \R^{\dim} \Big\}.
\end{align}
In order to compute the max-min solution~\eqref{eqn:MaxMinEmpirical}
over this policy class, we implement an actor-critic method, in which
the actor performs a variant of mirror descent.\footnote{Strictly
speaking, it is mirror ascent, but we use the conventional
terminology.}
\bcar
\item At each iteration $\iter = 1, \ldots, \ActorParNorm$, the policy
  $\ActorPolicy{\iter} \in \LinPolicyClass$ can be identified with a
  parameter $\ActorPar{\iter} \in \R^\dim$.  The sequence is
  initialized with $\ActorPar{1} = 0$.  
\item Using the finite test function class~\eqref{eqn:LinTestFunction}
  based on normalized eigenvectors, the pessimistic value estimate
  $\VminEmp{\ActorPolicy{\iter}}$ is computed by solving a quadratic
  program, as previously described.  This computation returns the
  weight vector $\CriticPar{\iter}$ of the associated optimal
  action-value function.
\item Using the action-value vector $\CriticPar{\iter}$, we update the
  actor's parameter as
\begin{align}
\label{eqn:LinearActorUpdate}
	\ActorPar{\iter+1} = \ActorPar{\iter} + \LearningRate
        \CriticPar{\iter} \qquad \mbox{where $\LearningRate =
          \sqrt{\frac{\log\abs{\ActionSpace}}{2\nIter}}$ is a stepsize
          parameter. }
\end{align}
\ecar
We now state a guarantee on the behavior of this procedure,
based on two \shortgoodname coefficients:
\begin{align}
\label{EqnNewConc}  
\ConcentrabilityGenericSub{\comparator}{1} = \dim \l
\norm{\Expecti{\LinPhi}{\comparator}}{\CovarianceStandardReg^{-1}}^2,
\quad \mbox{and} \quad \ConcentrabilityGenericSub{\comparator}{2} =
\dim \; \sup_{\policy \in \PolicyClass} \Big \{
\norm{\Expecti{[\LinPhi - \discount
      \LinPhiBootstrap{\policy}]}{\comparator}}
     {(\CovarianceWithBootstrapReg{\policy})^{-1}}^2 \Big \}.
\end{align}
Moreover, in making the following assertion, we assume that every weak
solution $\QpiWeak{\policy}$ can be evaluated against the distribution
of a comparator policy $\comparator \in \PolicyClass$, i.e.,
$\innerprodweighted{\1}
{\BellmanError{\QpiWeak{\policy}}{\policy}{}}{\comparator} = 0$ for
all $\policy \in \PolicyClass $.  (This assumption is still weaker
than strong realizability).
\begin{theorem}[Approximate Guarantees for Linear Soft-Max Optimization]
  \label{thm:LinearApproximation}
  Under the above conditions, running the procedure for $\nIter$
  rounds returns a policy sequence
  $\{\ActorPolicy{\iter}\}_{\iter=1}^\nIter$ such that, for any
  comparator policy $\comparator\in \PolicyClass$,
  \begin{align}
    \label{EqnMirrorBound}
\frac{1}{\nIter} \sumiter \big \{ \Vpi{\comparator} -
\Vpi{\ActorPolicy{\iter}} \big \} & \leq \frac{\cOne}{1-\discount}
\biggr \{ \underbrace{\sqrt{\frac{\log \abs{\ActionSpace}}{\nIter}}
  \vphantom{\sqrt{\ConcentrabilityGenericSub{\comparator}{\cdot}\frac{\dim
        \log(\nSamples\nIter) + \log
        \frac{\nSamples}{\FailureProbability} }{\nSamples}}}
}_{\text{Optimization error}} +
\underbrace{\sqrt{\ConcentrabilityGenericSub{\comparator}{\cdot}
    \frac{\dim \log(\nSamples\nIter) + \log \big(\frac{\nSamples
      }{\FailureProbability}\big) }{\nSamples}}}_{\text{Statistical
    error}} \biggr \},
\end{align}
with probability at least $1 - \FailureProbability$.  This bound
always holds with
\mbox{$\ConcentrabilityGenericSub{\comparator}{\cdot} =
  \ConcentrabilityGenericSub{\comparator}{2}$,} and moreover, it holds
with \mbox{$\ConcentrabilityGenericSub{\comparator}{\cdot} =
  \ConcentrabilityGenericSub{\comparator}{1}$} when weak Bellman
closure is in force.
\end{theorem}
\noindent See~\cref{sec:LinearApproximation} for the proof.  Whenever
Bellman closure holds, the result automatically inherits the more
favorable concentrability coefficient
$\ConcentrabilityGenericSub{\comparator}{2}$, as originally derived in
\cref{prop:LinearConcentrabilityBellmanClosure}. The resulting bound
is only $\sqrt{\dim}$ worse than the lower bound recently established
in the paper~\cite{zanette2021provable}. However, the method proposed
here is robust, in that it provides guarantees even when Bellman
closure does not hold.  In this case, we have a guarantee in terms of
the \shortgoodname coefficient
$\ConcentrabilityGenericSub{\comparator}{1}$.  Note that it is a
uniform version of the one derived previously
in~\cref{prop:LinearConcentrability}, in that there is an additional
supremum over the policy class.  This supremum arises due to the use
of gradient-based method, which implicitly searches over policies in
bootstrapping terms; see \cref{sec:LinearDiscussion} for a more
detailed discussion of this issue.

\section*{Acknowledgment}
AZ was partially supported by NSF-FODSI grant 2023505. In addition, this work was partially supported by NSF-DMS grant 2015454, NSF-IIS grant 1909365, as well as Office of Naval Research grant DOD-ONR-N00014-18-1-2640 to MJW.
The authors are grateful to Nan Jiang and Alekh Agarwal 
for pointing out further connections
with the existing literature,
as well as to the reviewers for pointing out clarity issues.


{\small{
\bibliographystyle{alpha}
\bibliography{rl}

\newcommand{\etalchar}[1]{$^{#1}$}
\begin{thebibliography}{FKSLX21}

\bibitem[AMS07]{antos2007fitted}
Andr{\'a}s Antos, R{\'e}mi Munos, and Csaba Szepesv{\'a}ri.
\newblock Fitted {Q}-iteration in continuous action-space {M}{D}{P}s.
\newblock 2007.

\bibitem[ASM08]{antos2008learning}
Andr{\'a}s Antos, Csaba Szepesv{\'a}ri, and R{\'e}mi Munos.
\newblock Learning near-optimal policies with {B}ellman-residual minimization
  based fitted policy iteration and a single sample path.
\newblock {\em Machine Learning}, 71(1):89--129, 2008.

\bibitem[ASN20]{agarwal2020optimistic}
Rishabh Agarwal, Dale Schuurmans, and Mohammad Norouzi.
\newblock An optimistic perspective on offline reinforcement learning.
\newblock In {\em International Conference on Machine Learning}, pages
  104--114. PMLR, 2020.

\bibitem[BB96]{bradtke1996linear}
Steven~J Bradtke and Andrew~G Barto.
\newblock Linear least-squares algorithms for temporal difference learning.
\newblock {\em Machine learning}, 22(1-3):33--57, 1996.

\bibitem[Ber95a]{Bertsekas_dyn1}
D.~P. Bertsekas.
\newblock {\em Dynamic programming and stochastic control}, volume~1.
\newblock Athena Scientific, Belmont, MA, 1995.

\bibitem[Ber95b]{Bertsekas_dyn2}
D.P. Bertsekas.
\newblock {\em Dynamic programming and stochastic control}, volume~2.
\newblock Athena Scientific, Belmont, MA, 1995.

\bibitem[BGB20]{buckman2020importance}
Jacob Buckman, Carles Gelada, and Marc~G Bellemare.
\newblock The importance of pessimism in fixed-dataset policy optimization.
\newblock {\em arXiv preprint arXiv:2009.06799}, 2020.

\bibitem[BLL{\etalchar{+}}11]{beygelzimer2011contextual}
Alina Beygelzimer, John Langford, Lihong Li, Lev Reyzin, and Robert Schapire.
\newblock Contextual bandit algorithms with supervised learning guarantees.
\newblock In {\em Proceedings of the Fourteenth International Conference on
  Artificial Intelligence and Statistics}, pages 19--26. JMLR Workshop and
  Conference Proceedings, 2011.

\bibitem[BT96]{bertsekas1996neuro}
Dimitri~P Bertsekas and John~N Tsitsiklis.
\newblock {\em Neuro-dynamic programming}.
\newblock Athena Scientific, 1996.

\bibitem[CJ19]{chen2019information}
Jinglin Chen and Nan Jiang.
\newblock Information-theoretic considerations in batch reinforcement learning.
\newblock In {\em International Conference on Machine Learning}, pages
  1042--1051, 2019.

\bibitem[CQ22]{chen2022well}
Xiaohong Chen and Zhengling Qi.
\newblock On well-posedness and minimax optimal rates of nonparametric
  q-function estimation in off-policy evaluation.
\newblock {\em arXiv preprint arXiv:2201.06169}, 2022.

\bibitem[DJL21]{duan2021risk}
Yaqi Duan, Chi Jin, and Zhiyuan Li.
\newblock Risk bounds and rademacher complexity in batch reinforcement
  learning.
\newblock {\em arXiv preprint arXiv:2103.13883}, 2021.

\bibitem[DW20]{duan2020minimax}
Yaqi Duan and Mengdi Wang.
\newblock Minimax-optimal off-policy evaluation with linear function
  approximation.
\newblock {\em arXiv preprint arXiv:2002.09516}, 2020.

\bibitem[Eva10]{evans2010partial}
Lawrence~C Evans.
\newblock {\em Partial differential equations}, volume~19.
\newblock American Mathematical Soc., 2010.

\bibitem[FCG18]{farajtabar2018more}
Mehrdad Farajtabar, Yinlam Chow, and Mohammad Ghavamzadeh.
\newblock More robust doubly robust off-policy evaluation.
\newblock In {\em International Conference on Machine Learning}, pages
  1447--1456. PMLR, 2018.

\bibitem[FGSM16]{farahmand2016regularized}
Amir-massoud Farahmand, Mohammad Ghavamzadeh, Csaba Szepesv{\'a}ri, and Shie
  Mannor.
\newblock Regularized policy iteration with nonparametric function spaces.
\newblock {\em The Journal of Machine Learning Research}, 17(1):4809--4874,
  2016.

\bibitem[FKSLX21]{foster2021offline}
Dylan~J. Foster, Akshay Krishnamurthy, David Simchi-Levi, and Yunzong Xu.
\newblock Offline reinforcement learning: Fundamental barriers for value
  function approximation, 2021.

\bibitem[Fle84]{fletcher1984computational}
Clive~AJ Fletcher.
\newblock Computational {G}alerkin methods.
\newblock In {\em Computational {G}alerkin methods}, pages 72--85. Springer,
  1984.

\bibitem[FRTL20]{feng2020accountable}
Yihao Feng, Tongzheng Ren, Ziyang Tang, and Qiang Liu.
\newblock Accountable off-policy evaluation with kernel bellman statistics.
\newblock In {\em International Conference on Machine Learning}, pages
  3102--3111. PMLR, 2020.

\bibitem[FSM10]{farahmand2010error}
Amir-massoud Farahmand, Csaba Szepesv{\'a}ri, and R{\'e}mi Munos.
\newblock Error propagation for approximate policy and value iteration.
\newblock In {\em Advances in Neural Information Processing Systems (NIPS)},
  2010.

\bibitem[Gal15]{galerkin1915series}
Boris~Grigoryevich Galerkin.
\newblock Series solution of some problems of elastic equilibrium of rods and
  plates.
\newblock {\em Vestnik inzhenerov i tekhnikov}, 19(7):897--908, 1915.

\bibitem[Haz21]{hazan2021introduction}
Elad Hazan.
\newblock Introduction to online convex optimization, 2021.

\bibitem[HJD{\etalchar{+}}21]{hao2021bootstrapping}
Botao Hao, Xiang Ji, Yaqi Duan, Hao Lu, Csaba Szepesv{\'a}ri, and Mengdi Wang.
\newblock Bootstrapping statistical inference for off-policy evaluation.
\newblock {\em arXiv preprint arXiv:2102.03607}, 2021.

\bibitem[JGS{\etalchar{+}}19]{jaques2019way}
Natasha Jaques, Asma Ghandeharioun, Judy~Hanwen Shen, Craig Ferguson, Agata
  Lapedriza, Noah Jones, Shixiang Gu, and Rosalind Picard.
\newblock Way off-policy batch deep reinforcement learning of implicit human
  preferences in dialog.
\newblock {\em arXiv preprint arXiv:1907.00456}, 2019.

\bibitem[JH20]{jiang2020minimax}
Nan Jiang and Jiawei Huang.
\newblock Minimax value interval for off-policy evaluation and policy
  optimization.
\newblock {\em arXiv preprint arXiv:2002.02081}, 2020.

\bibitem[JKA{\etalchar{+}}17]{jiang17contextual}
Nan Jiang, Akshay Krishnamurthy, Alekh Agarwal, John Langford, and Robert~E.
  Schapire.
\newblock Contextual decision processes with low {B}ellman rank are
  {PAC}-learnable.
\newblock In Doina Precup and Yee~Whye Teh, editors, {\em International
  Conference on Machine Learning (ICML)}, volume~70 of {\em Proceedings of
  Machine Learning Research}, pages 1704--1713, International Convention
  Centre, Sydney, Australia, 06--11 Aug 2017. PMLR.

\bibitem[JL16]{jiang2016doubly}
Nan Jiang and Lihong Li.
\newblock Doubly robust off-policy value evaluation for reinforcement learning.
\newblock In {\em International Conference on Machine Learning}, pages
  652--661. PMLR, 2016.

\bibitem[JYW21]{jin2021pessimism}
Ying Jin, Zhuoran Yang, and Zhaoran Wang.
\newblock Is pessimism provably efficient for offline rl?
\newblock In {\em International Conference on Machine Learning}, pages
  5084--5096. PMLR, 2021.

\bibitem[K{\etalchar{+}}03]{kakade2003sample}
Sham~Machandranath Kakade et~al.
\newblock {\em On the sample complexity of reinforcement learning}.
\newblock PhD thesis, University of London London, England, 2003.

\bibitem[KFTL19]{kumar2019stabilizing}
Aviral Kumar, Justin Fu, George Tucker, and Sergey Levine.
\newblock Stabilizing off-policy q-learning via bootstrapping error reduction.
\newblock {\em arXiv preprint arXiv:1906.00949}, 2019.

\bibitem[KHSL21]{kumar2021should}
Aviral Kumar, Joey Hong, Anikait Singh, and Sergey Levine.
\newblock Should i run offline reinforcement learning or behavioral cloning?
\newblock In {\em Deep RL Workshop NeurIPS 2021}, 2021.

\bibitem[KRNJ20]{kidambi2020morel}
Rahul Kidambi, Aravind Rajeswaran, Praneeth Netrapalli, and Thorsten Joachims.
\newblock Morel: Model-based offline reinforcement learning.
\newblock {\em arXiv preprint arXiv:2005.05951}, 2020.

\bibitem[KU19]{kallus2019efficiently}
Nathan Kallus and Masatoshi Uehara.
\newblock Efficiently breaking the curse of horizon in off-policy evaluation
  with double reinforcement learning.
\newblock {\em arXiv preprint arXiv:1909.05850}, 2019.

\bibitem[LLTZ18]{liu2018breaking}
Qiang Liu, Lihong Li, Ziyang Tang, and Dengyong Zhou.
\newblock Breaking the curse of horizon: Infinite-horizon off-policy
  estimation.
\newblock In {\em Advances in Neural Information Processing Systems}, pages
  5356--5366, 2018.

\bibitem[LP03]{lagoudakis2003least}
Michail~G Lagoudakis and Ronald Parr.
\newblock Least-squares policy iteration.
\newblock {\em Journal of machine learning research}, 4(Dec):1107--1149, 2003.

\bibitem[LSAB20]{liu2020provably}
Yao Liu, Adith Swaminathan, Alekh Agarwal, and Emma Brunskill.
\newblock Provably good batch reinforcement learning without great exploration.
\newblock {\em arXiv preprint arXiv:2007.08202}, 2020.

\bibitem[LTDC19]{laroche2019safe}
Romain Laroche, Paul Trichelair, and Remi~Tachet Des~Combes.
\newblock Safe policy improvement with baseline bootstrapping.
\newblock In {\em International Conference on Machine Learning}, pages
  3652--3661. PMLR, 2019.

\bibitem[LTND21]{lee2021model}
Jonathan~N Lee, George Tucker, Ofir Nachum, and Bo~Dai.
\newblock Model selection in batch policy optimization.
\newblock {\em arXiv preprint arXiv:2112.12320}, 2021.

\bibitem[MS08]{munos2008finite}
R{\'e}mi Munos and Csaba Szepesv{\'a}ri.
\newblock Finite-time bounds for fitted value iteration.
\newblock {\em Journal of Machine Learning Research}, 9(May):815--857, 2008.

\bibitem[Mun03]{munos2003error}
R{\'e}mi Munos.
\newblock Error bounds for approximate policy iteration.
\newblock In {\em ICML}, volume~3, pages 560--567, 2003.

\bibitem[Mun05]{munos2005error}
R{\'e}mi Munos.
\newblock Error bounds for approximate value iteration.
\newblock In {\em AAAI Conference on Artificial Intelligence (AAAI)}, 2005.

\bibitem[ND20]{nachum2020reinforcement}
Ofir Nachum and Bo~Dai.
\newblock Reinforcement learning via {F}enchel-{R}ockafellar duality.
\newblock {\em arXiv preprint arXiv:2001.01866}, 2020.

\bibitem[NDGL20]{nair2020accelerating}
Ashvin Nair, Murtaza Dalal, Abhishek Gupta, and Sergey Levine.
\newblock Accelerating online reinforcement learning with offline datasets.
\newblock {\em arXiv preprint arXiv:2006.09359}, 2020.

\bibitem[NDK{\etalchar{+}}19]{nachum2019algaedice}
Ofir Nachum, Bo~Dai, Ilya Kostrikov, Yinlam Chow, Lihong Li, and Dale
  Schuurmans.
\newblock Algaedice: Policy gradient from arbitrary experience.
\newblock {\em arXiv preprint arXiv:1912.02074}, 2019.

\bibitem[Pre00]{precup2000eligibility}
Doina Precup.
\newblock Eligibility traces for off-policy policy evaluation.
\newblock {\em Computer Science Department Faculty Publication Series},
  page~80, 2000.

\bibitem[Put94]{puterman1994markov}
Martin~L. Puterman.
\newblock {\em Markov Decision Processes: Discrete Stochastic Dynamic
  Programming}.
\newblock John Wiley \& Sons, Inc., New York, NY, USA, 1994.

\bibitem[Rep17]{repin2017one}
Sergey Repin.
\newblock One hundred years of the {G}alerkin method.
\newblock {\em Computational Methods in Applied Mathematics}, 17(3):351--357,
  2017.

\bibitem[RZM{\etalchar{+}}21]{rashidinejad2021bridging}
Paria Rashidinejad, Banghua Zhu, Cong Ma, Jiantao Jiao, and Stuart Russell.
\newblock Bridging offline reinforcement learning and imitation learning: A
  tale of pessimism.
\newblock {\em arXiv preprint arXiv:2103.12021}, 2021.

\bibitem[SB18]{sutton2018reinforcement}
Richard~S Sutton and Andrew~G Barto.
\newblock {\em Reinforcement learning: An introduction}.
\newblock MIT Press, 2018.

\bibitem[SSB{\etalchar{+}}20]{siegel2020keep}
Noah~Y Siegel, Jost~Tobias Springenberg, Felix Berkenkamp, Abbas Abdolmaleki,
  Michael Neunert, Thomas Lampe, Roland Hafner, Nicolas Heess, and Martin
  Riedmiller.
\newblock Keep doing what worked: Behavioral modelling priors for offline
  reinforcement learning.
\newblock {\em arXiv preprint arXiv:2002.08396}, 2020.

\bibitem[TB16]{thomas2016data}
Philip Thomas and Emma Brunskill.
\newblock Data-efficient off-policy policy evaluation for reinforcement
  learning.
\newblock In {\em International Conference on Machine Learning}, pages
  2139--2148, 2016.

\bibitem[TFL{\etalchar{+}}19]{tang2019doubly}
Ziyang Tang, Yihao Feng, Lihong Li, Dengyong Zhou, and Qiang Liu.
\newblock Doubly robust bias reduction in infinite horizon off-policy
  estimation.
\newblock {\em arXiv preprint arXiv:1910.07186}, 2019.

\bibitem[UHJ20]{uehara2020minimax}
Masatoshi Uehara, Jiawei Huang, and Nan Jiang.
\newblock Minimax weight and q-function learning for off-policy evaluation.
\newblock In {\em International Conference on Machine Learning}, pages
  9659--9668. PMLR, 2020.

\bibitem[UIJ{\etalchar{+}}21]{uehara2021finite}
Masatoshi Uehara, Masaaki Imaizumi, Nan Jiang, Nathan Kallus, Wen Sun, and
  Tengyang Xie.
\newblock Finite sample analysis of minimax offline reinforcement learning:
  Completeness, fast rates and first-order efficiency.
\newblock {\em arXiv preprint arXiv:2102.02981}, 2021.

\bibitem[US21]{uehara2021pessimistic}
Masatoshi Uehara and Wen Sun.
\newblock Pessimistic model-based offline reinforcement learning under partial
  coverage, 2021.

\bibitem[VJY21]{voloshin2021minimax}
Cameron Voloshin, Nan Jiang, and Yisong Yue.
\newblock Minimax model learning.
\newblock In {\em International Conference on Artificial Intelligence and
  Statistics}, pages 1612--1620. PMLR, 2021.

\bibitem[Wai19]{wainwright2019high}
Martin~J Wainwright.
\newblock {\em High-dimensional statistics: A non-asymptotic viewpoint},
  volume~48.
\newblock Cambridge University Press, 2019.

\bibitem[WFK20]{wang2020statistical}
Ruosong Wang, Dean~P Foster, and Sham~M Kakade.
\newblock What are the statistical limits of offline rl with linear function
  approximation?
\newblock {\em arXiv preprint arXiv:2010.11895}, 2020.

\bibitem[WN{\.Z}{\etalchar{+}}20]{wang2020critic}
Ziyu Wang, Alexander Novikov, Konrad {\.Z}o{\l}na, Jost~Tobias Springenberg,
  Scott Reed, Bobak Shahriari, Noah Siegel, Josh Merel, Caglar Gulcehre,
  Nicolas Heess, et~al.
\newblock Critic regularized regression.
\newblock {\em arXiv preprint arXiv:2006.15134}, 2020.

\bibitem[WTN19]{wu2019behavior}
Yifan Wu, George Tucker, and Ofir Nachum.
\newblock Behavior regularized offline reinforcement learning.
\newblock {\em arXiv preprint arXiv:1911.11361}, 2019.

\bibitem[XCJ{\etalchar{+}}21]{xie2021bellman}
Tengyang Xie, Ching-An Cheng, Nan Jiang, Paul Mineiro, and Alekh Agarwal.
\newblock {B}ellman-consistent pessimism for offline reinforcement learning.
\newblock {\em arXiv preprint arXiv:2106.06926}, 2021.

\bibitem[XJ20a]{xie2020batch}
Tengyang Xie and Nan Jiang.
\newblock Batch value-function approximation with only realizability.
\newblock {\em arXiv preprint arXiv:2008.04990}, 2020.

\bibitem[XJ20b]{xie2020Q}
Tengyang Xie and Nan Jiang.
\newblock Q* approximation schemes for batch reinforcement learning: A
  theoretical comparison.
\newblock volume 124 of {\em Proceedings of Machine Learning Research}, pages
  550--559, Virtual, 03--06 Aug 2020. PMLR.

\bibitem[XMW19]{xie2019towards}
Tengyang Xie, Yifei Ma, and Yu-Xiang Wang.
\newblock Towards optimal off-policy evaluation for reinforcement learning with
  marginalized importance sampling.
\newblock In {\em Advances in Neural Information Processing Systems}, pages
  9668--9678, 2019.

\bibitem[YB10]{YuBer10}
H.~Yu and D.~P. Bertsekas.
\newblock Error bounds for approximations from projected linear equations.
\newblock {\em Mathematics of Operations Research}, 35(2):306--329, 2010.

\bibitem[YBW20]{yin2020near}
Ming Yin, Yu~Bai, and Yu-Xiang Wang.
\newblock Near optimal provable uniform convergence in off-policy evaluation
  for reinforcement learning.
\newblock {\em arXiv preprint arXiv:2007.03760}, 2020.

\bibitem[YND{\etalchar{+}}20]{yang2020off}
Mengjiao Yang, Ofir Nachum, Bo~Dai, Lihong Li, and Dale Schuurmans.
\newblock Off-policy evaluation via the regularized lagrangian.
\newblock {\em arXiv preprint arXiv:2007.03438}, 2020.

\bibitem[YQCC21]{yang2021pessimistic}
Chao-Han~Huck Yang, Zhengling Qi, Yifan Cui, and Pin-Yu Chen.
\newblock Pessimistic model selection for offline deep reinforcement learning,
  2021.

\bibitem[YTY{\etalchar{+}}20]{yu2020mopo}
Tianhe Yu, Garrett Thomas, Lantao Yu, Stefano Ermon, James Zou, Sergey Levine,
  Chelsea Finn, and Tengyu Ma.
\newblock Mopo: Model-based offline policy optimization.
\newblock {\em arXiv preprint arXiv:2005.13239}, 2020.

\bibitem[YW20]{yin2020asymptotically}
Ming Yin and Yu-Xiang Wang.
\newblock Asymptotically efficient off-policy evaluation for tabular
  reinforcement learning.
\newblock In {\em International Conference on Artificial Intelligence and
  Statistics}, pages 3948--3958. PMLR, 2020.

\bibitem[YW21]{yin2021towards}
Ming Yin and Yu-Xiang Wang.
\newblock Towards instance-optimal offline reinforcement learning with
  pessimism.
\newblock {\em arXiv preprint arXiv:2110.08695}, 2021.

\bibitem[YWDW]{yinnear}
Ming Yin, Yu-Xiang Wang, Yaqi Duan, and Mengdi Wang.
\newblock Near-optimal offline reinforcement learning with linear
  representation: Leveraging variance information with pessimism.

\bibitem[Zan20]{zanette2020exponential}
Andrea Zanette.
\newblock Exponential lower bounds for batch reinforcement learning: Batch rl
  can be exponentially harder than online {R}{L}.
\newblock {\em arXiv preprint arXiv:2012.08005}, 2020.

\bibitem[ZDLS20]{zhang2020gendice}
Ruiyi Zhang, Bo~Dai, Lihong Li, and Dale Schuurmans.
\newblock Gendice: Generalized offline estimation of stationary values.
\newblock {\em arXiv preprint arXiv:2002.09072}, 2020.

\bibitem[ZHH{\etalchar{+}}22]{zhan2022offline}
Wenhao Zhan, Baihe Huang, Audrey Huang, Nan Jiang, and Jason~D Lee.
\newblock Offline reinforcement learning with realizability and single-policy
  concentrability.
\newblock {\em arXiv preprint arXiv:2202.04634}, 2022.

\bibitem[ZJZ21]{zhang2021optimal}
Zihan Zhang, Xiangyang Ji, and Yuan Zhou.
\newblock Almost optimal batch-regret tradeoff for batch linear contextual
  bandits, 2021.

\bibitem[ZLW20]{zhang2020gradientdice}
Shangtong Zhang, Bo~Liu, and Shimon Whiteson.
\newblock Gradientdice: Rethinking generalized offline estimation of stationary
  values.
\newblock In {\em International Conference on Machine Learning}, pages
  11194--11203. PMLR, 2020.

\bibitem[ZSU{\etalchar{+}}22]{zhang2022efficient}
Xuezhou Zhang, Yuda Song, Masatoshi Uehara, Mengdi Wang, Alekh Agarwal, and Wen
  Sun.
\newblock Efficient reinforcement learning in block {M}{D}{P}s: {A} model-free
  representation learning approach, 2022.

\bibitem[ZWB21]{zanette2021provable}
Andrea Zanette, Martin~J Wainwright, and Emma Brunskill.
\newblock Provable benefits of actor-critic methods for offline reinforcement
  learning.
\newblock {\em arXiv preprint arXiv:2108.08812}, 2021.

\end{thebibliography}
}}

\newpage
\tableofcontents

\newpage
\section{Additional Discussion and Results}

\subsection{Bellman Residual Orthogonalization}
\label{sec:appBRO}
Suppose that our goal is to estimate the action-value function
$\Qpi{\policy}$ of a given policy $\policy$.  This function is known
to be a fixed point of the Bellman evaluation operator
$\BellmanEvaluation{\policy}$ associated with the policy $\policy$.
Thus, when the MDP is known, one option is to (approximately) solve
the Bellman evaluation equations \mbox{$\Qfnc(\state, \action) =
  (\BellmanEvaluation{\policy}\Qfnc)(\state, \action)$} for all
state-action pairs. However, even if function approximation for
$\Qfnc$ is implemented, it is still difficult to directly solve these
equations if the state-action space is sufficiently complex.

This observation motivates the strategy taken in this paper: instead
of enforcing the Bellman equations for all state-action pairs, suppose
that we do so only in an average sense, and with respect to a certain
set of functions.  More formally, a \emph{test function} is a mapping
from the state-action space to the real line; any such function serves
to enforce the Bellman equations in an average sense in the following
way.  Let $\TestFunctionClass{\policy}$ denote some user-prescribed
class of test functions, which we refer to as the \emph{test space}.
Then for a given measure $\DatasetDistributionStateActions$, we
require only that the action-value function $\Qpi{\policy}$ satisfy
the integral constraints
\begin{align}
\label{eqn:WeakFormulation}
\smlinprod{\TestFunction{}}{\Qfnc - \BellmanEvaluation{\policy}
  (\Qfnc)}_{\DatasetDistributionStateActions} \defeq \int
\TestFunction{}\psa [\Qfnc\psa - (\BellmanEvaluation{\policy}
  \Qfnc)\psa] d\DatasetDistributionStateActions = 0, \qquad \mbox{for
  all $\TestFunction{} \in \TestFunctionClass{\policy}$.}
\end{align}
We refer to this design principle as \emph{Bellman residual
orthogonalization}, because it requires the Bellman error function to
be orthogonal to a set of test functions, as measured under the
$L^2(\DatasetDistributionStateActions)$ inner product.  Of course, by
enlarging the test space $\TestFunctionClass{\policy}$, the Bellman
error is required to be orthogonal to more test functions, and it will
ultimately be zero if enough test functions are added as constraints.
But at the same time, as shown by our analysis, any such enlargement
has both computational and statistical costs, so there are tradeoffs
to be understood.

In numerical analysis, especially in solving partial differential
equations, the design principle~\eqref{eqn:WeakFormulation} is called
the weak or variational formulation (e.g.,~\cite{evans2010partial}),
and its solutions are referred to as weak solutions.  Here we are
advocating a \emph{weak formulation} of the Bellman equations.  Of
course, the constraints~\eqref{eqn:WeakFormulation} are necessary but
not sufficient: the \emph{weak (Bellman) solutions} need not solve the
Bellman equations.  However, whenever we need to learn based on a
limited dataset, it is unreasonable to satisfy the Bellman equations
everywhere; instead, by choosing the test space appropriately, we can
seek to satisfy the Bellman equations over regions of the state-action
space that are most important.  In some cases, the
formulation~\eqref{eqn:WeakFormulation} can be fruitfully viewed as a
type of Galerkin approximation
(e.g.,~\cite{galerkin1915series,fletcher1984computational}) to the
Bellman equations.  For example, when both the test functions and
$Q$-value functions belong to some linear space (and the empirical
constraints are enforced exactly), then the weak formulation and
Galerkin approximation lead to the least-squares temporal difference
(LSTD) estimator; this connection between Galerkin methods and LSTD
has been noted in past work by Yu and Bertsekas~\cite{YuBer10}.  In
this paper, our goal is to understand the weak
formulation~\eqref{eqn:WeakFormulation} in a broader sense for general
test and predictor classes.

\subsection{Comparison with Weight Learning Methods}
\label{sec:WeightLearning}
The work closest to ours is \cite{jiang2020minimax}.
They also use an auxiliary weight function class, 
which is comparable to our test class. 
However, the test class is used in different ways;
we compare them in this section at the population level.\footnote{
The empirical estimator in \cite{jiang2020minimax} 
does not take into account the `alignment' of each 
weight function with respect to the dataset,
which we do through self-normalization and regularization
in the construction of the empirical estimator.
This precludes obtaining the same type of strong 
finite time guarantees that we are able to derive here.}
Let us assume that weak realizability holds and that
$\TestFunctionClass{}$ is symmetric, i.e., 
if $\TestFunction{} \in \TestFunctionClass{}$ 
then $-\TestFunction{} \in \TestFunctionClass{}$ as well.                                
At the population level,
our program seeks to solve      
\begin{align}
\label{eqn:PopProgComparison}
      	\sup_{\Qfnc \in \Qclass{\policy}} & \Expecti{\Qfnc(\state,\policy)}{\state \sim \startdistribution} 
      	\quad \text{s.t. } \quad 
      	\sup_{\TestFunction{} \in \TestFunctionClass{}}
      	\innerprodweighted{\TestFunction{}}{\BellmanError{\Qfnc}{\policy}{}}{\DatasetDistributionStateActions} = 0, 
      \end{align}
      \def\weight{w}
      which is equivalent for any $\weight \in \TestFunctionClass{}$ to
       \begin{align*}
      	\sup_{\Qfnc \in \Qclass{\policy}} & \Expecti{\Qfnc(\state,\policy)}{\state \sim \startdistribution} 
      	- \horizon \innerprodweighted{\weight}{\BellmanError{\Qfnc}{\policy}{}}{\DatasetDistributionStateActions}
      	\quad \text{s.t. } \quad 
      	\sup_{\TestFunction{} \in \TestFunctionClass{}}
      	\innerprodweighted{\TestFunction{}}{\BellmanError{\Qfnc}{\policy}{}}{\DatasetDistributionStateActions} = 0.
      \end{align*}
      Removing the constraints leads to the upper bound
            \begin{align*}
      	\sup_{\Qfnc \in \Qclass{\policy}} & \Expecti{\Qfnc(\state,\policy)}{\state \sim \startdistribution} 
      	- \horizon \innerprodweighted{\weight}{\BellmanError{\Qfnc}{\policy}{}}{\DatasetDistributionStateActions}.
      \end{align*}
      Since this is a valid upper bound for any $\weight{} \in \TestFunctionClass{}$, 
      minimizing over $\weight$ must still yield an upper bound,
      which reads
      \begin{align*}
      	\inf_{\weight \in \TestFunctionClass{}}\sup_{\Qfnc \in \Qclass{\policy}} & 
      	\Expecti{\Qfnc(\state,\policy)}{\state \sim \startdistribution}  
      	- \horizon \innerprodweighted{\weight}{\BellmanError{\Qfnc}{\policy}{}}{\DatasetDistributionStateActions}.
      \end{align*}
      This is the population program for ``weight learning'',
      as described in \cite{jiang2020minimax}.
      It follows that Bellman residual orthogonalization always produces tighter confidence intervals than ``weight learning''
      at the population level.
      
      Another interesting comparison is with ``value learning'', 
      also described in \cite{jiang2020minimax}.
      In this case, assuming symmetric $\TestFunctionClass{}$, 
	  we can equivalently express the population program \eqref{eqn:PopProgComparison}
	  using a Lagrange multiplier as follows
      \begin{align}
      \label{eqn:LagProgComparison}
      		\sup_{\Qfnc \in \Qclass{\policy}} & \Expecti{\Qfnc(\state,\policy)}{\state \sim \startdistribution} 
      	- 
      	\sup_{\lambda \geq 0 ,\TestFunction{} \in \TestFunctionClass{}}
      	\lambda\innerprodweighted{\TestFunction{}}{\BellmanError{\Qfnc}{\policy}{}}{\DatasetDistributionStateActions}. 
      \end{align}
      Rearranging we obtain
         \begin{align*}
      		\sup_{\Qfnc \in \Qclass{\policy}} 
      		\inf_{\lambda \geq 0 ,\TestFunction{} \in \TestFunctionClass{}}& \Expecti{\Qfnc(\state,\policy)}{\state \sim \startdistribution} 
      	- 
      	\lambda\innerprodweighted{\TestFunction{}}{\BellmanError{\Qfnc}{\policy}{}}{\DatasetDistributionStateActions}. 
      \end{align*}
      The ``value learning'' program proposed in \cite{jiang2020minimax} has a similar formulation to ours but differs in two key aspects.
      The first---and most important---is that \cite{jiang2020minimax} ignores the Lagrange multiplier;
      this means ``value learning'' is not longer associated to 
      a constrained program. 
      While the Lagrange multiplier could be ``incorporated''
      into the test class $\TestFunctionClass{}$, doing so
      would cause the entropy of 
      $\TestFunctionClass{}$ to be unbounded.
      Another point of difference is that ``value learning'' uses such expression with $\lambda = 1$
      to derive the confidence interval \emph{lower bound}, while we use it to construct 
      the confidence interval \emph{upper bound}. 
      While this may seem like a contradiction, 
      we notice that the expression is derived using different assumptions: 
      we assume weak realizability of $\Qfnc$,
      while \cite{jiang2020minimax} 
      assumes realizability of the density ratios between $\DatasetDistributionStateActions$
      and the discounted occupancy measure $\policy$.
           
\subsection{Additional Literature}
\label{sec:Literature}

Here we summarize some additional literature. The efficiency of
off-policy tabular RL has been investigated in the
papers~\cite{yin2020near,yin2020asymptotically,yin2021towards}.  For
empirical studies on offline RL, see the
papers~\cite{laroche2019safe,jaques2019way,wu2019behavior,agarwal2020optimistic,wang2020critic,siegel2020keep,nair2020accelerating,yang2021pessimistic,kumar2021should,buckman2020importance,kumar2019stabilizing,kidambi2020morel,yu2020mopo}.

Some of the classical RL algorithm are presented in the
papers~\cite{munos2003error,munos2005error,antos2007fitted,antos2008learning,farahmand2010error,farahmand2016regularized}.
For a more modern analysis, see~\cite{chen2019information}.  These
works generally make additionally assumptions on top of realizability.
Alternatively, one can use importance sampling
\cite{precup2000eligibility,thomas2016data,jiang2016doubly,farajtabar2018more}.
A more recent idea is to look at the distributions
themselves~\cite{liu2018breaking,nachum2019algaedice,xie2019towards,zhang2020gendice,zhang2020gradientdice,yang2020off,kallus2019efficiently}.

Offline policy optimization with pessimism has been studied in the
papers~\cite{liu2020provably,rashidinejad2021bridging,jin2021pessimism,xie2021bellman,zanette2021provable,yinnear,uehara2021pessimistic}.
There exists a fairly extensive literature on lower bounds with linear
representations, including the two
papers~\cite{zanette2020exponential,wang2020statistical} that
concurrently derived the first exponential lower bounds for the
offline setting, and \cite{foster2021offline} proves that
realizability and coverage alone are insufficient.

In the context of off-policy optimization several works have
investigated methods that assume only realizability of the optimal
policy~\cite{xie2020batch,xie2020Q}.  Related work includes the
papers~\cite{duan2020minimax,duan2021risk,jiang2020minimax,uehara2020minimax,tang2019doubly,nachum2020reinforcement,voloshin2021minimax,hao2021bootstrapping,zhang2022efficient,uehara2021finite,chen2022well,lee2021model}.
Among concurrent works, we note \cite{zhan2022offline}.

\subsection{Definition of Weak Bellman Closure}
\label{sec:appWeakClosure}
\begin{definition}[Weak Bellman Closure]
The Bellman operator $\BellmanEvaluation{\policy}$ is \emph{weakly
closed} with respect to the triple $\big( \Qclass{\policy},
\TestFunctionClass{\policy}, \DatasetDistributionStateActions \big)$
if for any $\Qfnc \in \Qclass{\policy}$, there exists a predictor
$\QpiProj{\policy}{\Qfnc} \in \Qclass{\policy}$ such that
\begin{align}  
\innerprodweighted{\TestFunction{}}{\QpiProj{\policy}{\Qfnc}}{\DatasetDistributionStateActions}
= \innerprodweighted{\TestFunction{}}
{\BellmanEvaluation{\policy}(\Qfnc)}
{\DatasetDistributionStateActions}.
\end{align}
\end{definition}

\newpage
\subsection{Additional results on the concentrability coefficients}
\label{sec:appConc}
\subsubsection{Testing with the identity function}

Suppose that the identity function $\1$ belongs to the test class.
Doing so amounts to requiring that the Bellman error is controlled in
an average sense over all the data.  When this choice is made, we can
derive some generic upper bounds on $\ConcSimple$, which we state and
prove here:
\begin{lemma}
If $\1 \in \TestFunctionClass{\policy}$, then we have the upper bounds
\begin{align}
\label{eqn:InEq}
\ConcSimple & \stackrel{(i)}{\leq} \frac{\max_{\Qfnc \in
    \PopulationFeasibleSet{\policy}} |\PolComplex|^2}{ \max_{\Qfnc \in
    \PopulationFeasibleSet{\policy}} |\PolMuComplex|^2} \;
\stackrel{(ii)}{\leq} \ConcSimple_* \defeq \max_{\Qfnc \in
  \PopulationFeasibleSet{\policy}}
\frac{|\PolComplex|^2}{|\PolMuComplex|^2}.
\end{align}
\end{lemma}
\begin{proof}
Since $\1 \in \TestFunctionClass{}$, the definition of
$\PopulationFeasibleSet{\policy}$ implies that
\begin{align*}
\max_{\Qfnc \in \PopulationFeasibleSet{\policy}} |\PolMuComplex|^2 &
\leq \big( \munorm{\1}^2 + \regpar \big) \frac{\Rad}{\numobs} 
= \big( 1 +
\regpar \big) \frac{\Rad}{\numobs}.
\end{align*}
The upper bound (i) then follows from the definition of $\ConcSimple$.
The upper bound (ii) follows since the right hand side is the maximum ratio.
\end{proof}

Note that large values of $\ConcSimple_*$ (defined in \cref{eqn:InEq}) 
can arise when there exist
$Q$-functions in the set $\PopulationFeasibleSet{\policy}$ that have
low average Bellman error under the data-generating distribution
$\mudist$, but relatively large values under $\policy$.  Of course,
the likelihood of such unfavorable choices of $Q$ is reduced when
we use a larger test function class, which then reduces the size of
$\PopulationFeasibleSet{\policy}$.  However, we pay a price in
choosing a larger test function class, since the
choice~\eqref{EqnRadChoice} of the radius $\Rad$ needed for
\cref{thm:NewPolicyEvaluation} depends on its complexity.



\subsubsection{Mixture distributions}

Now suppose that the dataset consists of a collection of trajectories
collected by different protocols.  More precisely, for each
$\iConstraint = 1, \ldots, \nConstraints$, let $\Dpi{\iConstraint}$ be
a particular protocol for generating a trajectory.  Suppose that we
generate data by first sampling a random index $\IConstraint \in
[\nConstraints]$ according to a probability distribution
$\{\Frequencies{\iConstraint} \}_{\iConstraint=1}^{\nConstraints}$,
and conditioned $\IConstraint = \iConstraint$, we sample $\psai$
according to $\Dpi{\iConstraint}$.  The resulting data follows a
mixture distribution, where we set $\identifier = \iConstraint$ to tag
the protocol used to generate the data. 
To be clear, for each sample $i = 1, \ldots, \nSamples$, 
we sample $\IConstraint$ as described, 
and then draw a single sample $\psai \sim \Dpi{\iConstraint}$ .

Following the intuition given in the previous section, it is natural
to include test functions that code for the protocol---that is, the
binary-indicator functions
\begin{align}
f_\iConstraint \psai & = \begin{cases} 1 & \mbox{if
    $\identifier=\iConstraint$} \\ 0 & \mbox{otherwise.}
\end{cases}
\end{align}
This test function, when included in the weak formulation, enforces
the Bellman evaluation equations for the policy $\policy \in
\PolicyClass$ under consideration along the distribution induced by
each data-generating policy $\Dpi{\iConstraint}$.

\begin{lemma}[Mixture Policy Concentrability]
  \label{lem:MixturePolicyConcentrability}  
Suppose that $\mudist$ is an $\nConstraints$-component mixture, and
that the indicator functions $\{f_\iConstraint
\}_{\iConstraint=1}^\nConstraints$ are included in the test class.
Then we have the upper bounds
\begin{align}
\label{EqnMixturePolicyUpper}  
  \ConcSimple & \stackrel{(i)}{\leq} \frac{1 + \nConstraints
    \regpar}{1 + \regpar} \; \frac{\max \limits_{\Qfnc\in
      \PopulationFeasibleSet{\policy}}
    [\Expecti{\BellmanError{\Qfnc}{\policy}{}}{\policy}]^2} {\max
    \limits_{\Qfnc \in \PopulationFeasibleSet{\policy}}
    \SumOverConstraints \Frequencies{\iConstraint}^2
                        [\Expecti{\BellmanError{\Qfnc}{\policy}{}}{\Dpi{\iConstraint}}]^2}
  \; \stackrel{(ii)}{\leq} \; \frac{1 + \nConstraints \regpar}{1 +
    \regpar} \; \max_{\Qfnc\in \PopulationFeasibleSet{\policy}} \left
  \{ \frac{ [\Expecti{\BellmanError{\Qfnc}{\policy}{}}{\policy}]^2} {
    \SumOverConstraints \Frequencies{\iConstraint}^2
                        [\Expecti{\BellmanError{\Qfnc}{\policy}{}}{\Dpi{\iConstraint}}]^2}
  \right \}.
\end{align}
\end{lemma}
\begin{proof}
From the definition of $\ConcSimple$, it suffices to show that
  \begin{align*}
\max \limits_{\Qfnc \in \PopulationFeasibleSet{\policy}}
\SumOverConstraints
\Frequencies{\iConstraint}^2[\Expecti{\BellmanError{\Qfnc}{\policy}{}}{\Dpi{\iConstraint}}]^2
\leq \frac{\Rad}{\numobs} \; \big(1 + \nConstraints \regpar \big).
\end{align*}
A direct calculation yields
$\innerprodweighted{\TestFunction{\iConstraint}}
{\BellmanError{\Qfnc}{\policy}{}}{\DatasetDistributionStateActions} =
\Expecti{\Indicator \{\identifier = \iConstraint \}
  \BellmanError{\Qfnc}{\policy}{}}{\DatasetDistributionStateActions} =
\Frequencies{\iConstraint}
\Expecti{\BellmanError{\Qfnc}{\policy}{}}{\Dpi{\iConstraint}}$.
Moreover, since each $\TestFunction{\iConstraint}$ belongs to the test
class by assumption, we have the upper bound
$\Big|\Frequencies{\iConstraint}
\Expecti{\BellmanError{\Qfnc}{\policy}{}}{\Dpi{\iConstraint}} \Big|
\leq \sqrt{\frac{\Rad}{\numobs}} \; \sqrt{
  \munorm{\TestFunction{\iConstraint}}^2 + \regpar}$.  Squaring each
term and summing over the constraints yields
\begin{align*}
\SumOverConstraints
\Frequencies{\iConstraint}^2[\Expecti{\BellmanError{\Qfnc}{\policy}{}}{\Dpi{\iConstraint}}]^2
\leq \frac{\Rad}{\numobs} \sum_{\iConstraint=1}^\nConstraints \big(
\munorm{\TestFunction{\iConstraint}}^2 + \regpar \big) =
\frac{\Rad}{\numobs} \big(1 + \nConstraints \regpar \big),
\end{align*}
where the final equality follows since
$\sum_{\iConstraint=1}^{\nConstraints}
\munorm{\TestFunction{\iConstraint}}^2 = 1$.
\end{proof}

As shown by the upper bound, the off-policy coefficient $\ConcSimple$
provides a measure of how the squared-averaged Bellman errors along
the policies $\{ \Dpi{\iConstraint}
\}_{\iConstraint=1}^\nConstraints$, weighted by their probabilities
$\{ \Frequencies{\iConstraint} \}_{\iConstraint=1}^{\nConstraints}$,
transfers to the evaluation policy $\policy$.  Note that the
regularization parameter $\regpar$ decays as a function of the sample
size---e.g., as $1/\nSamples$ in \cref{thm:NewPolicyEvaluation}---the
factor $(1 + \nConstraints \TestFunctionReg)/(1 + \TestFunctionReg)$
approaches one as $\numobs$ increases (for a fixed number
$\nConstraints$ of mixture components).


\subsubsection{Bellman Rank for off-policy evaluation}

In this section, we show how more refined bounds can be obtained
when---in addition to a mixture condition---additional structure is
imposed on the problem.  In particular, we consider a notion similar
to that of Bellman rank~\cite{jiang17contextual}, but suitably
adapted\footnote{ The original definition essentially takes
$\BRPolicyClass$ as the set of all greedy policies with respect to
$\BRQClass$.  Since a dataset need not originate from greedy policies,
the definition of Bellman rank is adapted in a natural way.}  to the
off-policy setting.

Given a policy class $\BRPolicyClass$ and a predictor class
$\BRQClass$, we say that it has Bellman rank is $\dim$ if there exist
two maps $\BRleft{} : \BRPolicyClass \rightarrow \R^\dim$ and
$\BRright{}: \BRQClass \rightarrow \R^\dim$ such that
\begin{align}
\label{eqn:BellmanRank}
\Expecti{\BellmanError{\Qfnc}{\policy}{}}{\policy} =
\smlinprod{\BRleft{\policy}}{\BRright{\Qfnc}}_{\R^d}, \qquad \text{for
  all } \; \policy \in \BRPolicyClass \; \text{and} \; \Qfnc \in
\BRQClass.
\end{align}
In words, the average Bellman error of any predictor $\Qfnc$ along any
given policy $\policy$ can be expressed as the Euclidean inner product
between two $\dim$-dimensional vectors, one for the policy and one for
the predictor.  As in the previous section, we assume that the data is
generated by a mixture of $\nConstraints$ different distributions (or
equivalently policies) $\{\Dpi{\iConstraint}
\}_{\iConstraint=1}^\nConstraints$.  In the off-policy setting, we
require that the policy class $\BRPolicyClass$ contains all of these
policies as well as the target policy---viz.  $\{\Dpi{\iConstraint} \}
\cup \{ \policy \} \subseteq \BRPolicyClass$.  Moreover, the predictor
class $\BRQClass$ should contain the predictor class for the target
policy, i.e., $\Qclass{\policy} \subseteq \BRQClass$.  We also assume
weak realizability for this discussion.

Our result depends on a positive semidefinite matrix determined by the
mixture weights $\{\Frequencies{\iConstraint}
\}_{\iConstraint=1}^\nConstraints$ along with the embeddings
$\{\BRleft{\Dpi{\iConstraint}}\}_{\iConstraint=1}^\nConstraints$ of the
associated policies that generated the data.  In particular, we define
\begin{align*}
  \BRCovariance{} = \SumOverConstraints \Frequencies{\iConstraint}^2
  \BRleft{\Dpi{\iConstraint}} \BRleft{\Dpi{\iConstraint}}^\top.
\end{align*}
Assuming that this is matrix is positive definite,\footnote{If not,
one can prove a result for a suitably regularized version.} we define
the norm $\|u\|_{\BRCovariance{}^{-1}} = \sqrt{u^T
  (\BRCovariance{})^{-1} u}$.  With this notation, we have the
following bound.
\begin{lemma}[Concentrability with Bellman Rank]
\label{lem:ConcentrabilityBellmanRank}
For a mixture data-generation process and under the Bellman rank
condition~\eqref{eqn:BellmanRank}, we have the upper bound
\begin{align}
\ConcSimple & \leq \; \frac{1 + \nConstraints \TestFunctionReg}{1 +
  \TestFunctionReg} \;
\norm{\BRleft{\policy}}{\BRCovariance{}^{-1}}^2,
\end{align}

\end{lemma}
\begin{proof}
Our proof exploits the upper bound (ii) from the
claim~\eqref{EqnMixturePolicyUpper}
in~\cref{lem:MixturePolicyConcentrability}.  We first evaluate and
redefine the ratio in this upper bound.  Weak realizability coupled
with the Bellman rank condition~\eqref{eqn:BellmanRank} implies that
there exists some $\QpiWeak{\policy}$ such that
\begin{align*}
0 & = \innerprodweighted{\TestFunction{\iConstraint}}
{\BellmanError{\QpiWeak{\policy}}{\policy}{}}{\DatasetDistributionStateActions}
= \Frequencies{\iConstraint} \Expecti
{\BellmanError{\QpiWeak{\policy}}{\policy}{}}{\Dpi{\iConstraint}} =
\Frequencies{\iConstraint}
\inprod{\BRleft{\Dpi{\iConstraint}}}{\BRright{\QpiWeak{\policy}}},
\qquad \mbox{for all $\iConstraint = 1, \ldots, \nConstraints$, and}
\\
0 & = \innerprodweighted{\1}
{\BellmanError{\QpiWeak{\policy}}{\policy}{}}{\policy} = \Expecti
{\BellmanError{\QpiWeak{\policy}}{\policy}{}}{\policy} =
\inprod{\BRleft{\policy}}{\BRright{\QpiWeak{\policy}}}.
\end{align*}
Therefore, we have the equivalences
$\Expecti{\BellmanError{\Qfnc}{\policy}{}}{\Dpi{\iConstraint}} =
\inprod{\BRleft{\Dpi{\iConstraint}}}{ (\BRright{\Qfnc} -
  \BRright{\QpiWeak{\policy}})}$ for all $\iConstraint = 1, \ldots,
\nConstraints$, as well as
$\Expecti{\BellmanError{\Qfnc}{\policy}{}}{\policy} =
\inprod{\BRleft{\policy}}{(\BRright{\Qfnc} -
  \BRright{\QpiWeak{\policy}})}$.  Introducing the shorthand
$\BRDelta{\Qfnc} = \BRright{\Qfnc} - \BRright{\QpiWeak{\policy}}$, we
can bound the ratio as follows
\begin{align*}
\sup_{\Qfnc \in \PopulationFeasibleSet{\policy}} \Big \{ \frac{
  (\inprod{\BRleft{\policy}}{\BRDelta{\Qfnc}})^2 }{
  \SumOverConstraints \Frequencies{\iConstraint}^2
  (\inprod{\BRleft{\Dpi{\iConstraint}}}{ \BRDelta{\Qfnc}})^2 } \Big
\}& = \sup_{\Qfnc \in \PopulationFeasibleSet{\policy}} \Big \{ \frac{
  (\inprod{\BRleft{\policy}}{ \BRDelta{\Qfnc}})^2
}{\BRDelta{\Qfnc}^\top \Big( \SumOverConstraints
  \Frequencies{\iConstraint}^2 \BRleft{\Dpi{\iConstraint}}
  \BRleft{\Dpi{\iConstraint}}^\top \Big) \BRDelta{\Qfnc} } \Big \}
\\
  & = \sup_{\Qfnc \in \PopulationFeasibleSet{\policy}} \Big \{ \frac{
  ( \smlinprod{\BRleft{\policy}}{ \BRCovariance{}^{-\frac{1}{2}}
    \BRDeltaCoord{\Qfnc}})^2 }{ \|\BRDeltaCoord{\Qfnc}\|_2^2} \Big
\} \qquad \mbox{where $\BRDeltaCoord{\Qfnc} =
  \BRCovariance{}^{\frac{1}{2}}\BRDelta{\Qfnc}$}\\
  & \leq \norm{\BRleft{\policy}}{\BRCovariance{}^{-1}}^2,
\end{align*}
where the final step follows from the Cauchy--Schwarz inequality.
\end{proof}
Thus, when performing off-policy evaluation with a mixture
distribution under the Bellman rank condition, the coefficient
$\ConcSimple$ is bounded by the alignment between the target policy
$\policy$ and the data-generating distribution
$\DatasetDistributionStateActions$, as measured in the the embedded
space guaranteed by the Bellman rank condition.  The structure of this
upper bound is similar to a result that we derive in the sequel for
linear approximation under Bellman closure
(see~\cref{prop:LinearConcentrabilityBellmanClosure}).

\subsection{Further comments on the prediction error test space}
\label{sec:appBubnov}
A few comments on the bound in \cref{lem:PredictionError}: 
as in our previous results, the
pre-factor $\frac{\norm{\QfncErr}{\DatasetDistributionStateActions}^2
  + \TestFunctionReg } { \norm{\1}{\policy}^2 + \TestFunctionReg }$
serves as a normalization factor.  Disregarding this leading term, the
second ratio measures how the prediction error
$\QfncErr = \Qfnc -
\QpiWeak{\policy} $ along $\DatasetDistributionStateActions$ transfers
to $\policy$, as measured via the operator $\IdentityOperator -
\discount\TransitionOperator{\policy}$.  This interaction is complex,
since it includes the \emph{bootstrapping term}
$-\discount\TransitionOperator{\policy}$.  (Notably, such a term is
not present for standard prediction or bandit problems, in which case
$\discount = 0$.)  This term reflects the dynamics intrinsic to
reinforcement learning, and plays a key role in proving ``hard'' lower
bounds for offline RL (e.g., see the
work~\cite{zanette2020exponential}).

Observe that the bound in \cref{lem:PredictionError} requires only
weak realizability, and thus it always applies. 
This fact is significant in light of a recent
lower bound~\cite{foster2021offline}, showing that without Bellman
closure, off-policy learning is challenging even under strong
concentrability assumption (such as bounds on density ratios).
\cref{lem:PredictionError} gives a sufficient condition without
Bellman closure, but with a different measure that accounts for
bootstrapping. \\

\noindent If, in fact, (weak) Bellman closure holds,
then~\cref{lem:PredictionError} takes the following simplified form:
\begin{lemma}[\shortgoodname coefficient under Bellman closure]
\label{lem:PredictionErrorBellmanClosure}
If $\QclassErr{\policy} \subseteq \TestFunctionClass{\policy}$ and
weak Bellman closure holds, then
\begin{align*}
\ConcSimple \leq \max_{\QfncErrNC \in \QclassErr{\policy}} \Big \{
\frac{ \norm{\QfncErrNC}{\DatasetDistributionStateActions}^2 +
  \TestFunctionReg } { 1 + \TestFunctionReg } \, \cdot \, \frac{
  \innerprodweighted{\1} {\QfncErrNC} {\policy}^2 } {
  \innerprodweighted{\QfncErrNC} {\QfncErrNC}
                    {\DatasetDistributionStateActions}^2 } \Big \}
\leq \max_{\QfncErrNC \in \QclassErr{\policy}} \Big \{ \frac{ \norm
  {\QfncErrNC} {\policy}^2 } { \norm{\QfncErrNC}
  {\DatasetDistributionStateActions}^2 } \Big \}.
\end{align*}
\end{lemma}
\noindent See \cref{sec:PredictionErrorBellmanClosure} for the
proof. \\

In such case, the concentrability measures the increase in the
discrepancy $\Qfnc - \Qfnc'$ of the feasible predictors when moving
from the dataset distribution $\DatasetDistributionStateActions$ to
the distribution of the target policy $\policy$. 
In \cref{sec:DomainKnowledge}, we give another bound under weak
Bellman closure, and thereby recover a recent result due to Xie et
al.~\cite{xie2021bellman}.  Finally, in~\cref{sec:Linear}, we provide
some applications of this concentrability factor to the linear
setting.

\subsection{From Importance Sampling to Bellman Closure}
\label{sec:IS2BC}
Let us show an application of \cref{prop:MultipleRobustness} on an
example with just two test spaces.  Suppose that we suspect that
Bellman closure holds, but rather than committing to such assumption,
we wish to fall back to an importance sampling estimator if Bellman
closure does not hold.

In order to streamline the presentation of the idea, let us introduce
the following setup.  Let $\policyBeh{}$ be a behavioral policy that
generates the dataset, i.e., such that each state-action $\psa$ in the
dataset is sampled from its discounted state distribution
$\DistributionOfPolicy{\policyBeh{}}$.  Next, let the identifier
$\identifier$ contain the trajectory from $\startdistribution$ up to
the state-action pair $\psa$ recorded in the dataset.  That is, each
tuple $\sarsi{}$ in the dataset $\Dataset$ is such that $\psa \sim
\DistributionOfPolicy{\policyBeh{}}$ and $\identifier$ contains the
trajectory up to $\psa$.

We now define the test spaces.  The first one is denoted with
$\TestFunctionISClass{\policy}$ and leverages importance sampling.  It
contains a single test function defined as the importance sampling
estimator
\begin{align}
\label{eqn:ImportanceSampling}
\TestFunctionISClass{\policy} = \{ \TestFunction{\policy}\}, \qquad
\text{where} \; \TestFunction{\policy}\psai =
\frac{1}{\scaling{\policy}} \prod_{(\state_\hstep,\action_\hstep) \in
  \identifier} \frac{ \policy(\action_{\hstep} \mid \state_{\hstep})
}{ \policyBeh{}(\action_{\hstep} \mid \state_{\hstep}) }.
\end{align}
The above product is over the random trajectory contained in the
identifier $\identifier$.  The normalization factor $\scaling{\policy}
\in \R$ is connected to the maximum range of the importance sampling
estimator, and ensures that $\sup_{\psai} \TestFunction{\policy}\psai
\leq 1$.  The second test space is the prediction error test space
$\QclassErr{\policy}$ defined in \cref{sec:ErrorTestSpace}.

With this choice, let us define three concentrability coefficients.
$\ConcentrabilityGenericSub{\policy}{1}$ arises from importance
sampling, $\ConcentrabilityGenericSub{\policy}{2}$ from the prediction
error test space when Bellman closure holds and
$\ConcentrabilityGenericSub{\policy}{3}$ from the prediction error
test space when just weak realizability holds. They are defined as
\begin{align*}
\ConcentrabilityGenericSub{\policy}{1} \leq \sqrt{ \scaling{\policy}
  \frac{(1 + \TestFunctionReg\scaling{\policy})}{1+\TestFunctionReg} }
\qquad \ConcentrabilityGenericSub{\policy}{2} \leq \max_{\QfncErr \in
  \QclassErrCentered{\policy}} \frac{ \innerprodweighted{\1}
  {(\IdentityOperator - \gamma\TransitionOperator{\policy})\QfncErr}
  {\policy}^2 } { \innerprodweighted{\QfncErr} {(\IdentityOperator -
    \gamma\TransitionOperator{\policy})\QfncErr}
  {\DatasetDistributionStateActions}^2 } \times \frac{
  \norm{\QfncErr}{\DatasetDistributionStateActions}^2 +
  \TestFunctionReg } { \norm{\1}{\policy}^2 + \TestFunctionReg },
\qquad \ConcentrabilityGenericSub{\policy}{3} \leq \cOne \frac{
  \norm{\BellmanError{\Qfnc}{\policy}{}}{\policy}^2 }{
  \norm{\BellmanError{\Qfnc}{\policy}{}}{\DatasetDistributionStateActions}
  ^2}.
\end{align*}

\begin{lemma}[From Importance Sampling to Bellman Closure]
\label{lem:IS}
The choice $\TestFunctionClass{\policy} =
\TestFunctionISClass{\policy} \cup \QclassErr{\policy} \; \text{for
  all } \policy\in\PolicyClass$ ensures that with probability at least
$1-\FailureProbability$, the oracle inequality~\eqref{EqnOracle} holds
with $\ConcentrabilityGeneric{\policy} \leq \min\{
\ConcentrabilityGenericSub{\policy}{1},
\ConcentrabilityGenericSub{\policy}{2},
\ConcentrabilityGenericSub{\policy}{3} \} $ if weak Bellman closure
holds and $\ConcentrabilityGeneric{\policy} \leq \min\{
\ConcentrabilityGenericSub{\policy}{1},
\ConcentrabilityGenericSub{\policy}{2} \} $ otherwise.
\end{lemma}
\begin{proof}
Let us calculate the {\goodname} associated with
$\TestFunctionISClass{\policy}$.  The unbiasedness of the importance
sampling estimator gives us the following population constraint (here
$\DatasetDistributionStateActions =
\DistributionOfPolicy{\policyBeh{}}$)
\begin{align*}
\abs{ \innerprodweighted{\TestFunction{\policy}}
  {\BellmanError{\Qfnc}{\policy}{}} {\DatasetDistributionStateActions}
} = \abs{
  \Expecti{\TestFunction{\policy}\BellmanError{\Qfnc}{\policy}{}}
          {\DatasetDistributionStateActions} } =
\frac{1}{\scaling{\policy}} \abs{
  \Expecti{\BellmanError{\Qfnc}{\policy}{}} {\policy} } =
\frac{1}{\scaling{\policy}} \abs{
  \innerprodweighted{\1}{\BellmanError{\Qfnc}{\policy}{}} {\policy} }
\leq \frac{\ConfidenceInterval{\FailureProbability}}{\sqrt{\nSamples}}
\sqrt{\norm{\TestFunction{\policy}}{2}^2 + \TestFunctionReg}
\end{align*}
The norm of the test function reads (notice that
$\DatasetDistributionStateActions$ generates $\psai$ here)
\begin{align*}
\norm{\TestFunction{\policy}}{\DatasetDistributionStateActions}^2 =
\Expecti{\TestFunction{\policy}^2}{\DatasetDistributionStateActions} =
\frac{1}{\scaling{\policy}^2} \Expecti{ \Bigg[
    \prod_{(\state_\hstep,\action_\hstep) \in \identifier} \frac{
      \policy(\action_{\hstep} \mid \state_{\hstep}) }{
      \policyBeh{}(\action_{\hstep} \mid \state_{\hstep}) }
  }{\DatasetDistributionStateActions}\Bigg]^2 =
\frac{1}{\scaling{\policy}^2} \Expecti{ \Bigg[
    \prod_{(\state_\hstep,\action_\hstep) \in \identifier} \frac{
      \policy(\action_{\hstep} \mid \state_{\hstep}) }{
      \policyBeh{}(\action_{\hstep} \mid \state_{\hstep}) }
  }{\policy}\Bigg] \leq \frac{1}{\scaling{\policy}}.  \\
\end{align*}
Together with the prior display, we obtain
\begin{align*}
  \frac{\innerprodweighted{\1}{\BellmanError{\Qfnc}{\policy}{}}
    {\policy}^2}
       {\scaling{\policy}^2(\norm{\TestFunction{\policy}}{2}^2 +
         \TestFunctionReg)} \leq
       \frac{\Rad}{\nSamples}.
\end{align*}
The resulting concentrability coefficient is therefore
\begin{align*}
\ConcentrabilityGeneric{\policy} \leq \max_{\Qfnc\in
  \PopulationFeasibleSet{\policy}} \frac{
  \innerprodweighted{\1}{\BellmanError{\Qfnc}{\policy}{}}{\policy}^2}
                       {1 + \TestFunctionReg} \times
                       \frac{\nSamples}{\Rad} \leq \max_{\Qfnc\in
                         \PopulationFeasibleSet{\policy}} \frac{
                         \innerprodweighted{\1}{\BellmanError{\Qfnc}{\policy}{}}{\policy}^2}
                            {1 + \TestFunctionReg} \times
                            \frac{\scaling{\policy}^2(\norm{\TestFunction{\policy}}{2}^2
                              + \TestFunctionReg)}
                                 {\innerprodweighted{\1}{\BellmanError{\Qfnc}{\policy}{}}
                                   {\policy}^2} \leq \scaling{\policy}
                                 \frac{(1 +
                                   \TestFunctionReg\scaling{\policy})}{1+\TestFunctionReg}.
\end{align*}
Chaining the above result with
\cref{lem:BellmanTestFunctions,lem:PredictionError}, using
\cref{prop:MultipleRobustness} and plugging back into
\cref{thm:NewPolicyEvaluation} yields the thesis.
\end{proof}


\subsection{Implementation for Off-Policy Predictions}
\label{sec:LinearConfidenceIntervals}

In this section, we describe a computationally efficient way in which
to compute the upper/lower
estimates~\eqref{eqn:ConfidenceIntervalEmpirical}.  Given a finite set
of $\nTestFunctions$ test functions, it involves solving a quadratic
program with $2 \nTestFunctions + 1$ constraints.

Let us first work out a concise description of the constraints
defining membership in $\EmpiricalFeasibleSet{\policy}$.  Introduce
the shorthand $\nLinEffSq{\TestFunction{}} \defeq
\norm{\TestFunction{\iConstraint}}{\numobs}^2 + \TestFunctionReg$.  We
then define the empirical average feature vector
$\LinPhiEmpiricalExpecti{\TestFunction{}}$, the empirical average
reward $\LinEmpiricalReward{\TestFunction{}}$, and the average
next-state feature vector
$\LinPhiEmpiricalExpectiBootstrap{\TestFunction{}}{\policy}$ as
\begin{align*}
\quad \LinPhiEmpiricalExpecti{\TestFunction{}} =
\frac{1}{\nLinEff{\TestFunction{}}} \SumOverSars \TestFunction{}\psa
\LinPhi\psa, \qquad \LinEmpiricalReward{\TestFunction{}} =
\frac{1}{\nLinEff{\TestFunction{}}} \SumOverSars
\TestFunction{}\psa\reward, \\ \qquad
\LinPhiEmpiricalExpectiBootstrap{\TestFunction{}}{\policy} =
\frac{1}{\nLinEff{\TestFunction{}}} \SumOverSars \TestFunction{}\psa
\LinPhi(\successorstate,\policy).
\end{align*} 

In terms of this notation, each empirical constraint defining
$\EmpiricalFeasibleSet{\policy}$ can be written in the more compact
form
\begin{align*}
\frac{\abs{\innerprodweighted{\TestFunction{}}
    {\TDError{\Qfnc}{\policy}{}}{\numobs}} }
     {\nLinEff{\TestFunction{}}} & = \Big |
     \smlinprod{\LinPhiEmpiricalExpecti{\TestFunction{}} - \discount
       \LinPhiEmpiricalExpectiBootstrap{\TestFunction{}}{\policy}}{\CriticPar{}}
     - \LinEmpiricalReward{\TestFunction{}} \Big| \leq
     \sqrt{\frac{\Rad}{\numobs}}.
\end{align*}

Then the set of empirical constraints can be written as a set of
constraints linear in the critic parameter $\CriticPar{}$ coupled with
the assumed regularity bound on $\CriticPar{}$
\begin{align}
\label{eqn:LinearConstraints}
	\EmpiricalFeasibleSet{\policy} = \Big\{ \CriticPar{} \in \R^d
        \mid \norm{\CriticPar{}}{2} \leq 1, \quad \mbox{and} \quad -
        \sqrt{\frac{\Rad}{\numobs}} \leq
        \smlinprod{\LinPhiEmpiricalExpecti{\TestFunction{}} -
          \discount
          \LinPhiEmpiricalExpectiBootstrap{\TestFunction{}}{\policy}}{\CriticPar{}}
        - \LinEmpiricalReward{\TestFunction{}} \leq
        \sqrt{\frac{\Rad}{\numobs}} \quad \mbox{for all
          $\TestFunction{} \in \TestFunctionClass{\policy}$} \Big\}.
\end{align}
Thus, the estimates $\VminEmp{\pol}$ (respectively $\VmaxEmp{\pol}$)
acan be computed by minimizing (respectively maximizing) the linear
objective function $w \mapsto
\inprod{[\Expecti{\Expecti{\LinPhi(\state,\action)} {\action \sim
        \policy}}{\state \sim \startdistribution}]}{\CriticPar{}}$
subject to the $2 \nTestFunctions + 1$ constraints in
equation~\eqref{eqn:LinearConstraints}.  Therefore, the estimates can
be computed in polynomial time for any test function with a
cardinality that grows polynomially in the problem parameters.


\subsection{Discussion of Linear Approximate Optimization}
\label{sec:LinearDiscussion}

Here we discuss the presence of the supremum over policies in the
coefficient $\ConcentrabilityGenericSub{\comparator}{1}$ from
equation~\eqref{EqnNewConc}.  In particular, it arises because our
actor-critic method iteratively approximates the maximum in the
max-min estimate~\eqref{eqn:MaxMinEmpirical} using a gradient-based
scheme.  The ability of a gradient-based method to make progress is
related to the estimation accuracy of the gradient, which is the
$\Qfnc$ estimates of the actor's current policy $\ActorPolicy{\iter}$;
more specifically, the gradient is the $\Qfnc$ function parameter
$\CriticPar{\iter}$.  In the general case, the estimation error of the
gradient $\CriticPar{\iter}$ depends on the policy under consideration
through the matrix $\CovarianceWithBootstrapReg{\ActorPolicy{\iter}}$,
while it is independent in the special case of Bellman closure (as it
depends on just $\CovarianceStandard$).  As the actor's policies are
random, this yields the introduction of a
$\sup_{\policy\in\PolicyClass}$ in the general bound.  Notice the
method still competes with the best comparator $\comparator$ by
measuring the errors along the distribution of the comparator (through
the operator $\Expecti{}{\comparator}$).  To be clear,
$\sup_{\policy\in\PolicyClass}$ may not arise with approximate
solution methods that do not rely only on the gradient to make
progress (such as second-order methods); we leave this for future
research.  Reassuringly, when Bellman closure, the approximate
solution method recovers the standard guarantees established in the
paper~\cite{zanette2021provable}.

\newpage
\section{General Guarantees}
\label{sec:GeneralGuarantees}

\subsection{A deterministic guarantee}

We begin our analysis stating a deterministic set of sufficient
conditions for our estimators to satisfy the
guarantees~\eqref{EqnPolEval} and~\eqref{EqnOracle}.  This formulation
is useful, because it reveals the structural conditions that underlie
success of our estimators, and in particular the connection to weak
realizability.  In Section~\ref{SecHighProb}, we exploit this
deterministic result to show that, under a fairly general sampling
model, our estimators enjoy these guarantees with high probability.

In the previous section, we introduced the population level set
$\PopulationFeasibleSet{\policy}$ that arises in the statement of our
guarantees.  Also central in our analysis is the infinite data limit
of this set.  More specifically, for any fixed $(\Rad, \regpar)$, if
we take the limit $\nSamples \rightarrow \infty$, then
$\PopulationFeasibleSet{\policy}$ reduces to the set of all solutions
to the weak formulation~\eqref{eqn:WeakFormulation}---that is
\begin{align}
\PopulationFeasibleSetInfty{\policy}(\TestFunctionClass{\policy}) = \{
\Qfnc \in \Qclass{\policy} \mid
\innerprodweighted{\TestFunction{}}{\BellmanError{\Qfnc}{\pi}{}}
                  {\DatasetDistributionStateActions} = 0 \quad
                  \mbox{for all $\TestFunction{} \in
                    \TestFunctionClass{\policy}$} \}.
\end{align}
As before, we omit the dependence on the test function class
$\TestFunctionClass{\policy}$ when it is clear from context. By
construction, we have the inclusion
$\PopulationFeasibleSetInfty{\policy}(\TestFunctionClass{\policy})
\subseteq \SuperPopulationFeasible$ for any non-negative pair $(\Rad,
\regpar)$.

Our first set of guarantees hold when the random set
$\EmpiricalFeasibleSet{\policy}$ satisfies the \emph{sandwich
relation}
\begin{align}
\label{EqnSandwich}  
\PopulationFeasibleSetInfty{\policy}(\TestFunctionClass{\policy})
\subseteq \EmpiricalFeasibleSet{\policy}(\Rad, \regpar;
\TestFunctionClass{\policy}) \subseteq
\PopulationFeasibleSet{\policy}(4 \Rad, \regpar;
\TestFunctionClass{\policy})
\end{align}
To provide intuition as to why this sandwich condition is natural,
observe that it has two important implications:
\begin{enumerate}
\item[(a)] Recalling the definition of weak
  realizability~\eqref{EqnWeakRealizable}, the weak solution
  $\QpiWeak{\policy}$ belongs to the empirical constraint set
  $\EmpiricalFeasibleSet{\policy}$ for any choice of test function
  space.  This important property follows because $\QpiWeak{\policy}$
  must satisfy the constraints~\eqref{eqn:WeakFormulation}, and thus
  it belongs to $\PopulationFeasibleSetInfty{\policy} \subseteq
  \EmpiricalFeasibleSet{\policy}$.
\item[(b)] All solutions in $\EmpiricalFeasibleSet{\policy}$ also
  belong to $\PopulationFeasibleSet{\policy}$, which means they
  approximately satisfy the weak Bellman equations in a way quantified
  by $\PopulationFeasibleSet{\policy}$.
\end{enumerate}
By leveraging these facts in the appropriate way, we can establish
the following guarantee: \\

\begin{proposition}
\label{prop:Deterministic}
The following two statements hold. 
\begin{enumerate}
\item[(a)] \underline{Policy evaluation:} If the set
  $\EmpiricalFeasibleSet{\policy}$ satisfies the sandwich
  relation~\eqref{EqnSandwich}, then the estimates
  $(\VminEmp{\policy}, \VmaxEmp{\policy})$ satisfy the width
  bound~\eqref{EqnWidthBound}.  If, in addition, weak Bellman
  realizability for $\policy$ is assumed, then the
  coverage~\eqref{EqnCoverage} condition holds.
\item[(b)] \underline{Policy optimization:} If the sandwich
  relation~\eqref{EqnSandwich} and weak Bellman realizability hold for
  all $\policy \in \PolicyClass$, then any
  max-min~\eqref{eqn:MaxMinEmpirical} optimal policy $\piemp$
  satisfies the oracle inequality~\eqref{EqnOracle}.
\end{enumerate}
\end{proposition}
\noindent See Section~\ref{SecProofPropDeterministic} for the proof of
this claim. \\

In summary, \Cref{prop:Deterministic} ensures that when weak
realizability is in force, then the sandwich
relation~\eqref{EqnSandwich} is a sufficient condition for both the
policy evaluation~\eqref{EqnPolEval} and
optimization~\eqref{EqnOracle} guarantees to hold.  Accordingly, the
next phase of our analysis focuses on deriving sufficient conditions
for the sandwich relation to hold with high probabability.

\subsection{Some high-probability guarantees}
\label{SecHighProb}

As stated, \Cref{prop:Deterministic} is a ``meta-result'', in that it
applies to any choice of set $\EmpiricalFeasibleSet{\policy} \equiv
\SuperEmpiricalFeasible$ for which the sandwich
relation~\eqref{EqnSandwich} holds.  In order to obtain a more
concrete guarantee, we need to impose assumptions on the way in which
the dataset was generated, and concrete choices of $(\Rad, \regpar)$
that suffice to ensure that the associated sandwich
relation~\eqref{EqnSandwich} holds with high probability.  These tasks
are the focus of this section.

\subsubsection{A model for data generation}
\label{SecDataGen}
Let us begin by describing a fairly general model for data-generation.
Any sample takes the form $\sarsizNp{} \defeq \sarsi{}$, where the
five components are defined as follows:

\bcar
\item the pair $(\state, \action)$ index the current state and action.
\item the random variable $\reward$ is a noisy observation of the mean
  reward.
\item the random state $\successorstate$ is the next-state sample,
  drawn according to the transition $\TransitionLaw\psa$.
\item the variable $\identifier$ is an optional identifier.
  \ecar

\noindent As one example of the use of an identifier variable, if
samples might be generated by one of two possible policies---say
$\policy_1$ and $\policy_2$---the identifier can take values in the
set $\{1, 2 \}$ to indicate which policy was used for a particular
sample. \\

Overall, we observe a dataset $\Dataset = \{\sarsizNp{i}
\}_{i=1}^\numobs$ of $\numobs$ such quintuples.  In the simplest of
possible settings, each triple $\psai$ is drawn i.i.d.  from some
fixed distribution $\DatasetDistributionStateActions$, and the noisy
reward $\reward_i$ is an unbiased estimate of the mean reward function
$\Reward(\state_i, \action_i)$.  In this case, our dataset consists of
$\numobs$ i.i.d.  quintuples.  More generally, we would like to
accommodate richer sampling models in which the sample $z_i =
(\state_i, \action_i, \identifier_i, \reward_i, \successorstate_i)$ at
a given time $i$ is allowed to depend on past samples.  In order to
specify such dependence in a precise way, define the nested sequence
of sigma-fields
\begin{align}
\Filtration_1 = \emptyset, \quad \mbox{and} \quad \Filtration_i \defeq
\sigma \Big(\{\sarsizNp{j}\}_{j=1}^{i-1} \Big) \qquad \mbox{for $i =
  2, \ldots, \numobs$.}
\end{align}
In terms of this filtration, we make the following definition:

\begin{assumption}[Adapted dataset]
  \label{asm:Dataset}
  An adapted dataset is a collection $\Dataset = \{
  \sarsizNp{\iSample} \}_{i=1}^\numobs$ such that for each $i = 1,
  \ldots, \numobs$:
\bcar
\item There is a conditional distribution
  $\DatasetDistributionStateActions_i$ such that $(\state_i,
  \action_i, \identifier_i) \sim \DatasetDistributionStateActions_i
  (\cdot \mid \Filtration_i)$.
\item Conditioned on
  $(\state_\iSample,\action_\iSample,\identifier_\iSample )$, we
  observe a noisy reward $\reward_\iSample =
  \reward(\state_\iSample,\action_\iSample) + \eta_i$ with $\E[\eta_i
    \mid \Filtration_{\iSample} ] = 0$, and $|\reward_i| \leq 1$.
\item Conditioned on
  $(\state_\iSample,\action_\iSample,\identifier_\iSample )$, 
  the next state $\successorstate_\iSample$ is generated
  according to
  $\TransitionLaw(\state_\iSample,\action_\iSample)$.
\ecar
\end{assumption}

Under this assumption, we can define the (possibly) random reference
measure
\begin{align}
\DatasetDistributionStateActions(\state, \action, \identifier) &
\defeq \frac{1}{\numobs} \sum_{i=1}^\numobs
\DatasetDistributionStateActions_i \big(\state, \action, \identifier
\mid \Filtration_i \big).
\end{align}
In words, it corresponds to the distribution induced by first drawing
a time index $i \in \{1, \ldots, \numobs \}$ uniformly at random, and
then sampling a triple $\psai$ from the conditional distribution
$\DatasetDistributionStateActions_i \big(\cdot \mid \Filtration_i
\big)$.

\subsubsection{A general guarantee}
\newcommand{\Mfun}{\ensuremath{\phi}}
\label{SecGeneral}

Recall that there are three function classes that underlie our method:
the test function class $\TestFunctionClass{}$, the policy class
$\PolicyClass$, and the $Q$-function class $\Qclass{}$.  In this
section, we state a general guarantee (\cref{thm:NewPolicyEvaluation})
that involves the metric entropies of these sets.  In
Section~\ref{SecCorollaries}, we provide corollaries of this guarantee
for specific function classes.

In more detail, we equip the test function class and the $Q$-function
class with the usual sup-norm
\begin{align*}
\|f - \ftil\|_\infty \defeq \sup_{\psai} |f\psai - \ftil \psai|, \quad
\mbox{and} \quad \|Q - \Qtil\|_\infty \defeq \sup_{\psa} |Q \psa -
\Qtil \psa|,
\end{align*}
and the policy class with the sup-TV norm
\begin{align*}
  \|\pi - \pitil\|_{\infty, 1} & \defeq \sup_{\state} \|\pi(\cdot \mid
  \state) - \pitil(\cdot \mid \state) \|_1 = \sup_{\state}
  \sum_{\action} |\pi(\action \mid \state) - \pitil(\action \mid
  \state)|.
\end{align*}
For a given $\precision > 0$, we let
$\CoveringNumber{\TestFunctionClass{}}{\precision}$,
$\CoveringNumber{\Qclass{}}{\precision}$, and
$\CoveringNumber{\PolicyClass}{\precision}$ denote the
$\precision$-covering numbers of each of these function classes in the
given norms.  Given these covering numbers, a tolerance parameter
$\delta \in (0,1)$ and the shorthand $\Mfun(t) = \max \{t, \sqrt{t} \}$, define the radius function
\begin{subequations}
\label{EqnTheoremChoice}
\begin{align}
\label{EqnDefnR}
\Rad(\epsilon, \delta) & \defeq \nSamples \Big\{
\int_{\epsilon^2}^\epsilon \Mfun \big( \frac{\log
  N_u(\TestFunctionClass{})}{\numobs} \big) du + \frac{\log
  N_\epsilon(\Qclass{})}{\numobs} + \frac{ \log
  N_\epsilon(\Pi)}{\numobs} + \frac{\log(\numobs/\delta)}{\numobs}
\Big\}.
\end{align}
In our theorem, we implement the estimator using a radius $\Rad =
\Rad(\epsilon, \delta)$, where $\epsilon > 0$ is any parameter that
satisfies the bound
\begin{align}
\label{EqnRadChoice}
\epsilon^2 & \stackrel{(i)}{\leq} \cspec \: \frac{\Rad(\epsilon,
  \delta)}{\numobs}, \quad \mbox{and} \quad \regpar \stackrel{(i)}{=}
4 \frac{\Rad(\epsilon, \delta)}{\numobs}.
\end{align}
\end{subequations}
Here $\cspec > 0$ is a suitably chosen but universal constant (whose
value is determined in the proof), and we adopt the shorthand $\Rad =
\Rad(\epsilon, \delta)$ in our statement below.

\renewcommand{\descriptionlabel}[1]{%
  \hspace{\labelsep}\normalfont\underline{#1}%
}

\begin{theorem}[High-probability guarantees]
\label{thm:NewPolicyEvaluation}
Consider the estimates implemented using triple $\InputFunctionClass$
that is weakly Bellman realizable (\cref{asm:WeakRealizability}); an
adapted dataset (\Cref{asm:Dataset}); and with the
choices~\eqref{EqnTheoremChoice} for $(\epsilon, \Rad, \regpar)$.
Then with probability at least $1 - \delta$:
\begin{description}
\item[Policy evaluation:] For any $\pi \in \PolicyClass$, the
  estimates $(\VminEmp{\policy}, \VmaxEmp{\policy})$ specify a
  confidence interval satisfying the coverage~\eqref{EqnCoverage} and
  width bounds~\eqref{EqnWidthBound}.
\item[Policy optimization:] Any max-min
  policy~\eqref{eqn:MaxMinEmpirical} $\piemp$ satisfies the oracle
  inequality~\eqref{EqnOracle}.
\end{description}
\end{theorem}
\noindent See \cref{SecProofNewPolicyEvaluation} for the proof of the
claim. \\

\paragraph{Choices of $(\Rad, \epsilon, \regpar)$:}  Let us provide
a few comments about the choices of $(\Rad, \epsilon, \regpar)$ from
equations~\eqref{EqnDefnR} and~\eqref{EqnRadChoice}.  The quality of
our bounds depends on the size of the constraint set
$\PopulationFeasibleSet{\policy}$, which is controlled by the
constraint level $\sqrt{\frac{\Rad}{\numobs}}$.  Consequently, our
results are tightest when $\Rad = \Rad(\epsilon, \delta)$ is as small
as possible.  Note that $\Rad$ is an decreasing function of
$\epsilon$, so that in order to minimize it, we would like to choose
$\epsilon$ as large as possible subject to the
constraint~\eqref{EqnRadChoice}(i).  Ignoring the entropy integral
term in equation~\eqref{EqnRadChoice} for the moment---see below for
some comments on it---these considerations lead to
\begin{align}
\numobs \epsilon^2 \asymp \log N_\epsilon(\TestFunctionClass{}) + \log
N_\epsilon(\Qclass{}) + \log N_\epsilon(\Pi).
\end{align}
This type of relation for the choice of $\epsilon$ in non-parametric
statistics is well-known (e.g., see Chapters 13--15 in the
book~\cite{wainwright2019high} and references therein).  Moreover,
setting $\regpar \asymp \epsilon^2$ as in
equation~\eqref{EqnRadChoice}(ii) is often the correct scale of
regularization.

\paragraph{Key technical steps in proof:}   It is worthwhile
making a few comments about the structure of the proof so as to
clarify the connections to \Cref{prop:Deterministic} along with the
weak formulation that underlies our methods.  Recall that
\Cref{prop:Deterministic} requires the empirical
$\EmpiricalFeasibleSet{\policy}$ and population sets
$\PopulationFeasibleSet{\policy}$ to satisfy the sandwich
relation~\eqref{EqnSandwich}.  In order to prove that this condition
holds with high probability, we need to establish uniform control over
the family of random variables
\begin{align}
\label{EqnKeyFamily}  
  \frac{ \big| \inprod{f}{\delta^\policy(\Qfnc)}_\numobs -
    \inprod{f}{\BellError^\policy(\Qfnc)}_\mudist
    \big|}{\TestNormaRegularizerEmp{}}, \qquad \mbox{as indexed by the
    triple $(f, \Qfnc, \policy)$.}
\end{align}
Note that the differences in the numerator of these variables
correspond to moving from the empirical constraints on $Q$-functions
that are enforced using the TD errors, to the population constraints
that involve the Bellman error function.

Uniform control of the family~\eqref{EqnKeyFamily}, along with the
differences $\|f\|_\numobs - \|f\|_\mudist$ uniformly over $f$, allows
us to relate the empirical and population sets, since the associated
constraints are obtained by shifting between the empirical inner
products $\inprod{\cdot}{\cdot}_\numobs$ to the reference inner
products $\inprod{\cdot}{\cdot}_\mudist$.  A simple discretization
argument allows us to control the differences uniformly in $(\Qfnc,
\policy)$, as reflected by the metric entropies appearing in our
definition~\eqref{EqnTheoremChoice}.  Deriving uniform bounds over
test functions $f$---due to the self-normalizing nature of the
constraints---requires a more delicate argument.  More precisely, in
order to obtain optimal results for non-parametric problems (see
Corollary~\ref{cor:alpha} to follow), we need to localize the
empirical process at a scale $\epsilon$, and derive bounds on the
localized increments.  This portion of the argument leads to the
entropy integral---which is localized to the interval $[\epsilon^2,
  \epsilon]$---in our definition~\eqref{EqnDefnR} of the radius
function.

\paragraph{Intuition from the on-policy setting:}

In order to gain intuition for the statistical meaning of the
guarantees in~\cref{thm:NewPolicyEvaluation}, it is worthwhile
understanding the implications in a rather special case---namely, the
simpler on-policy setting, where the discounted occupation measure
induced by the target policy $\policy$ coincides with the dataset
distribution $\DatasetDistributionStateActions$.  Let us consider the
case in which the identity function $\1$ belongs to the test class
$\TestFunctionClass{\policy}$.  Under these conditions, for any $\Qfnc
\in \PopulationFeasibleSet{\policy}$, we can write
\begin{align*}
\max_{\Qfnc \in \PopulationFeasibleSet{\policy}} |\PolComplex| &
\stackrel{(i)}{=} \max_{\Qfnc \in \PopulationFeasibleSet{\policy}}|
\Expecti{\BellmanError{\Qfnc}{\policy}{}}{\mudist}|
\; \stackrel{(ii)}{\leq} \sqrt{1 + \regpar} \;
\sqrt{\frac{\Rad}{\numobs}},
\end{align*}
where equality (i) follows from the on-policy assumption, and step
(ii) follows from the definition of the set
$\PopulationFeasibleSet{\policy}$, along with the condition that $\1
\in \TestFunctionClass{\policy}$.  Consequently, in the on-policy
setting, the width bound~\eqref{EqnWidthBound} ensures that
\begin{align}
\label{EqnOnPolicy}
\abs{\VminEmp{\policy} - \VmaxEmp{\policy}} \leq 2 \frac{\sqrt{1 +
    \regpar}}{1-\discount} \sqrt{\frac{ \Rad}{\numobs}}.
\end{align}
In this simple case, we see that the confidence interval scales as
$\sqrt{\Rad/\numobs}$, where the quantity $\Rad$ is related to the
metric entropy via equation~\eqref{EqnRadChoice}.  In the more general
off-policy setting, the bound involves this term, along with
additional terms that reflect the cost of off-policy data.  We discuss
these issues in more detail in \cref{sec:Applications}.  Before doing
so, however, it is useful derive some specific corollaries that show
the form of $\Rad$ under particular assumptions on the underlying
function classes, which we now do.

\subsubsection{Some corollaries}
\label{SecCorollaries}

Theorem~\ref{thm:NewPolicyEvaluation} applies generally to triples of
function classes $\InputFunctionalSpace$, and the statistical error
$\sqrt{\frac{\Rad(\epsilon, \delta)}{\numobs}}$ depends on the metric
entropies of these function classes via the
definition~\eqref{EqnDefnR} of $\Rad(\epsilon, \delta)$, and the
choices~\eqref{EqnRadChoice}.  As shown in this section, if we make
particular assumptions about the metric entropies, then we can derive
more concrete guarantees.

\paragraph{Parametric and finite VC classes:}  One form of metric
entropy, typical for a relatively simple function class $\Gclass$
(such as those with finite VC dimension) scales as
\begin{align}
  \label{EqnPolyMetric}
  \log N_\epsilon(\Gclass) & \asymp d \; \log \big(\frac{1}{\epsilon}
  \big),
\end{align}
for some dimensionality parameter $d$.  For instance, bounds of this
type hold for linear function classes with $d$ parameters, and for
finite VC classes (with $d$ proportional to the VC dimension); see
Chapter 5 of the book~\cite{wainwright2019high} for more details.

\begin{corollary}
\label{cor:poly}  
Suppose each class of the triple $\InputFunctionClass$ has metric
entropy that is at most polynomial~\eqref{EqnPolyMetric} of order $d$.
Then for a sample size $\numobs \geq 2 d$, the claims of
Theorem~\ref{thm:NewPolicyEvaluation} hold with $\epsilon^2 =
d/\numobs$ and
\begin{align}
    \label{EqnPolyRchoice}
    \tilde{\Rad}\big( \sqrt{\frac{d}{\numobs}}, \delta \big) & \defeq
    c \; \Big \{ d \; \log \big( \frac{\numobs}{d} \big) + \log
    \big(\frac{\nSamples}{\delta} \big) \Big \},
  \end{align}
where $c$ is a universal constant.  
\end{corollary}
\begin{proof}
Our strategy is to upper bound the radius $\Rad$ from
equation~\eqref{EqnDefnR}, and then show that this upper bound
$\tilde{\Rad}$ satisfies the conditions~\eqref{EqnRadChoice} for the
specified choice of $\epsilon^2$.  We first control the term $\log
N_\epsilon(\TestFunctionClass{})$.  We have
\begin{align*}
\frac{1}{\sqrt{\numobs}} \int_{\epsilon^2}^\epsilon \sqrt{\log
  N_u(\TestFunctionClass{})} du & \leq \sqrt{\frac{d}{\numobs}} \;
\int_{0}^{\epsilon} \sqrt{\log(1/u)} du \; = \; \epsilon
\sqrt{\frac{d}{\numobs}} \; \int_0^1 \sqrt{\log(1/(\epsilon t))} dt \;
= \; c \epsilon \log(1/\epsilon) \sqrt{\frac{d}{\numobs}}.
\end{align*}
Similarly, we have
\begin{align*}
\frac{1}{\numobs} \int_{\epsilon^2}^\epsilon \log
N_u(\TestFunctionClass{}) du & \leq \epsilon \frac{d}{\numobs} \Big \{
\int_{\epsilon}^{1} \log(1/t) dt + \log(1/\epsilon) \Big \} \; \leq \;
c \, \epsilon \log(1/\epsilon) \frac{d}{\numobs}.
\end{align*}
Finally, for terms not involving entropy integrals, we have
\begin{align*}
  \max \Big \{ \frac{\log N_\epsilon(\Qclass{})}{\numobs}, \frac{\log
    N_\epsilon(\Pi)}{\numobs} \Big\} & \leq c \frac{d}{\numobs}
  \log(1/\epsilon).
\end{align*}
Setting $\epsilon^2 = d/\numobs$, we see that the required
conditions~\eqref{EqnRadChoice} hold with the specified
choice~\eqref{EqnPolyRchoice} of $\tilde{\Rad}$.
\end{proof}

\paragraph{Richer function classes:}
In the previous section, the metric entropy scaled logarithmically in
the inverse precision $1/\epsilon$.  For other (richer) function
classes, the metric entropy exhibits a polynomial scaling in the
inverse precision, with an exponent $\alpha > 0$ that controls the
complexity.  More precisely, we consider classes of the form
\begin{align}
  \label{EqnAlphaMetric}
  \log N_\epsilon(\Gclass) & \asymp \Big(\frac{1}{\epsilon}
  \Big)^\alpha.
\end{align}
For example, the class of Lipschitz functions in dimension $d$ has
this type of metric entropy with $\alpha = d$.  More generally, for
Sobolev spaces of functions that have $s$ derivatives (and the
$s^{th}$-derivative is Lipschitz), we encounter metric entropies of
this type with $\alpha = d/s$.  See Chapter 5 of the
book~\cite{wainwright2019high} for further background.

\begin{corollary}
\label{cor:alpha}  
Suppose that each function class $\InputFunctionClass$ has metric
entropy with at most \mbox{$\alpha$-scaling~\eqref{EqnAlphaMetric}}
for some $\alpha \in (0,2)$.  Then the claims of
Theorem~\ref{thm:NewPolicyEvaluation} hold with $\epsilon^2 =
(1/\numobs)^{\frac{2}{2 + \alpha}}$, and
  \begin{align}
 \label{EqnAlphaRchoice}
    \tilde{\Rad}\big((1/\numobs)^{\frac{1}{2 + \alpha}}, \delta \big)
    & = c \; \Big \{ \numobs^{\frac{\alpha}{2 + \alpha}} +
    \log(\nSamples/\delta) \Big \}.
  \end{align}
where $c$ is a universal constant.  
\end{corollary}
\noindent 
We note that for standard regression problems over classes with
$\alpha$-metric entropy, the rate $(1/\numobs)^{\frac{2}{2 + \alpha}}$
is well-known to be minimax optimal (e.g., see Chapter 15 in the
book~\cite{wainwright2019high}, as well as references therein).

\begin{proof}
  We start by controlling the terms involving entropy integrals.
In particular, we have
  \begin{align*}
\frac{1}{\sqrt{\numobs}} \int_{\epsilon^2}^\epsilon \sqrt{\log
  N_u(\TestFunctionClass{})} du & \leq \frac{c}{\sqrt{\numobs}} u^{1 -
  \frac{\alpha}{2}} \Big |_{0}^\epsilon \; = \;
\frac{c}{\sqrt{\numobs}} \epsilon^{1 - \frac{\alpha}{2}}.
  \end{align*}
Requiring that this term is of order $\epsilon^2$ amounts to enforcing
that $\epsilon^{1 + \frac{\alpha}{2}} \asymp (1/\sqrt{\numobs})$, or
equivalently that $\epsilon^2 \asymp (1/\numobs)^{\frac{2}{2 +
    \alpha}}$.

If $\alpha \in (0,1]$, then the second entropy integral converges and
is of lower order. Otherwise, if $\alpha \in (1,2)$, then we have
\begin{align*}    
\frac{1}{\numobs} \int_{\epsilon^2}^\epsilon \log
N_u(\TestFunctionClass{}) du & \leq \frac{c}{\numobs}
\int_{\epsilon^2}^\epsilon (1/u)^{\alpha} du \leq \frac{c}{\numobs}
(\epsilon^2)^{1 - \alpha}.
\end{align*}
Hence the requirement that this term is bounded by $\epsilon^2$ is
equivalent to $\epsilon^{2 \alpha} \succsim (1/\numobs)$, or
$\epsilon^2 \succsim (1/\numobs)^{1/\alpha}$.  When $\alpha \in
(1,2)$, we have $\frac{1}{\alpha} > \frac{2}{2 + \alpha}$, so that
this condition is milder than our first condition.

Finally, we have $ \max \big \{ \frac{\log
  N_\epsilon(\Qclass{})}{\numobs}, \frac{\log
  N_\epsilon(\Pi)}{\numobs} \big\} \leq \frac{c}{\numobs}
\big(1/\epsilon)^{\alpha}$, and requiring that this term scales as
$\epsilon^2$ amounts to requiring that $\epsilon^{2 +\alpha} \asymp
(1/\numobs)$, or equivalently $\epsilon^2 \asymp
(1/\numobs)^{\frac{2}{2 + \alpha}}$, as before.
\end{proof}


\newpage
\section{Main Proofs}
\label{sec:AnalysisProofs}

This section is devoted to the proofs of our guarantees for general
function classes---namely, \cref{prop:Deterministic} that holds in a
deterministic manner, and \cref{thm:NewPolicyEvaluation} that gives
high probability bounds under a particular sampling model.


\subsection{Proof of \cref{prop:Deterministic}}
\label{SecProofPropDeterministic}

Our proof makes use of an elementary simulation lemma, which we state
here:
\begin{lemma}[Simulation lemma]
  \label{lem:Simulation}
For any policy $\policy$ and function $\Qfnc$, we have
\begin{align}
\Estart{\Qfnc-\Qpi{\policy}}{,\policy} & =
\frac{\PolComplex}{1-\discount}
\end{align}
\end{lemma}
\noindent See \cref{SecProofLemSimulation} for the proof of this
claim.

\subsubsection{Proof of policy evaluation claims}

First of all, we have the elementary bounds
\begin{align*}
\abs{\VminEmp{\policy} - \Vpi{\policy}} & = \abs{ \min_{\Qfnc \in
    \EmpiricalFeasibleSet{\policy}} \Expecti{\Qfnc(\MyState,
    \policy)}{\MyState \sim \startdistribution} - \Vpi{\policy} } \leq
\max_{\Qfnc \in \EmpiricalFeasibleSet{\policy}} \abs{
  \Expecti{\Qfnc(\MyState ,\policy)}{\MyState \sim \startdistribution}
  - \Vpi{\policy} }, \quad \mbox{and} \\
\abs{\VmaxEmp{\policy} - \Vpi{\policy}} & = \abs{ \max_{\Qfnc \in
    \EmpiricalFeasibleSet{\policy}} \Expecti{\Qfnc(\MyState
    ,\policy)}{\MyState \sim \startdistribution} - \Vpi{\policy} }
\leq \max_{\Qfnc \in \EmpiricalFeasibleSet{\policy}} \abs{
  \Expecti{\Qfnc(\MyState, \policy)}{\MyState \sim \startdistribution}
  - \Vpi{\policy} }.
\end{align*}
Consequently, in order to prove the bound~\eqref{EqnWidthBound} it
suffices to upper bound the right-hand side common in the two above
displays.  Since $\EmpiricalFeasibleSet{\policy} \subseteq
\PopulationFeasibleSet{\policy}$, we have the upper bound
\begin{align*}
\max_{\Qfnc \in \EmpiricalFeasibleSet{\policy}} \abs{
  \Expecti{\Qfnc(\MyState ,\policy)}{\MyState \sim \startdistribution}
  - \Vpi{\policy} } & \leq
\max_{\Qfnc\in\PopulationFeasibleSet{\policy}} \abs{
  \Expecti{\Qfnc(\MyState ,\policy)}{\MyState \sim \startdistribution} -
  \Vpi{\policy} } \\
& = \max_{\Qfnc\in\PopulationFeasibleSet{\policy}}
\abs{\Expecti{[\Qfnc(\MyState ,\policy) - \Qpi{\policy}(\MyState
      ,\policy)]} {\MyState \sim \startdistribution}} \\ &
\overset{\text{(i)}}{=} \horizon
\max_{\Qfnc\in\PopulationFeasibleSet{\policy}}
\frac{\PolComplex}{1-\discount}
\end{align*}
where step (i) follows from \cref{lem:Simulation}.  Combined with the
earlier displays, this completes the proof of the
bound~\eqref{EqnWidthBound}.

We now show the inclusion $[ \VminEmp{\policy}, \VmaxEmp{\policy}] \ni
\Vpi{\policy}$ when weak realizability holds.  By definition of weak
realizability, there exists some $\QpiWeak{\policy} \in
\PopulationFeasibleSetInfty{\policy}$.  In conjunction with  our
sandwich assumption, we are guaranteed that $\QpiWeak{\policy} \in
\PopulationFeasibleSetInfty{\policy} \subseteq
\EmpiricalFeasibleSet{\policy}$, and consequently
\begin{align*}
  \VminEmp{\policy} & = \min_{\Qfnc\in\EmpiricalFeasibleSet{\policy}}
  \Expecti{\Qfnc(\MyState ,\policy)}{\MyState \sim \startdistribution}
  \leq \min_{\Qfnc\in\PopulationFeasibleSetInfty{\policy}}
  \Expecti{\Qfnc(\MyState ,\policy)}{\MyState \sim \startdistribution}
  \leq \Expecti{\QpiWeak{\policy}(\MyState ,\policy)}{\MyState \sim
    \startdistribution} = \Vpi{\policy}, \quad \mbox{and} \\
\VmaxEmp{\policy} & = \max_{\Qfnc\in\EmpiricalFeasibleSet{\policy}}
\Expecti{\Qfnc(\MyState ,\policy)}{\MyState \sim \startdistribution}
\geq \max_{\Qfnc\in\PopulationFeasibleSetInfty{\policy}}
\Expecti{\Qfnc(\MyState ,\policy)}{\MyState \sim \startdistribution}
\geq \Expecti{\QpiWeak{\policy}(\MyState ,\policy)}{\MyState \sim
  \startdistribution} = \Vpi{\policy}.
\end{align*}

\subsubsection{Proof of policy optimization claims}

We now prove the oracle inequality~\eqref{EqnOracle} on the value
$\Vpi{\piemp}$ of a policy $\piemp$ that optimizes the max-min
criterion.  Fix an arbitrary comparator policy $\polcomp$.  Starting
with the inclusion $[\VminEmp{\piemp}, \VmaxEmp{\piemp}] \ni
\Vpi{\piemp}$, we have
\begin{align*}
\Vpi{\piemp} \stackrel{(i)}{\geq} \VminEmp{\piemp}
\stackrel{(ii)}{\geq} \VminEmp{\polcomp} \; = \; \Vpi{\polcomp} -
\Big(\Vpi{\polcomp} - \VminEmp{\polcomp} \Big)
\stackrel{(iii)}{\geq} \Vpi{\polcomp} - \horizon \max_{\Qfnc \in
  \PopulationFeasibleSet{\polcomp}}
\frac{|\PolComplexGen{\polcomp}|}{1-\discount},
\end{align*}
where step (i) follows from the stated inclusion at the start of the
argument; step (ii) follows since $\piemp$ solves the max-min program;
and step (iii) follows from the bound $|\Vpi{\polcomp} -
\VminEmp{\polcomp}| \leq \horizon
\max_{\Qfnc\in\PopulationFeasibleSet{\polcomp}}
\frac{\PolComplex}{1-\discount}$, as proved in the preceding section.
This lower bound holds uniformly for all comparators $\polcomp$,
from which the stated claim follows.


\subsection{Proof of \cref{lem:Simulation}}
\label{SecProofLemSimulation}

For each $\timestep = 1, 2, \ldots$, let $\Expecti{}{\timestep}$ be
the expectation over the state-action pair at timestep $\timestep$
upon starting from $\startdistribution$, so that we have
$\Estart{\Qfnc - \Qpi{\policy} }{,\policy} = 
\Expecti{[\Qfnc -
    \Qpi{\policy}]}{0}$ by definition.  We claim that
\begin{align}
\label{EqnInduction}  
\Expecti{[\Qfnc - \Qpi{\policy}]}{0} & = \sum_{\tau=1}^{\timestep}
\discount^{\tau-1} \Expecti{\BellmanError{\Qfnc}{\policy}{}}{\tau-1} +
\discount^{\timestep}\Expecti{[\Qfnc - \Qpi{\policy} ]}{\timestep}
\qquad \mbox{for all $\timestep = 1, 2, \ldots$.}
\end{align}
For the base case $\timestep = 1$, we have
\begin{align}
\label{EqnBase}  
  \Expecti{[ \Qfnc - \Qpi{\policy}]}{0} = \Expecti{[\Qfnc -
      \BellmanEvaluation{\policy}\Qfnc]}{0} + \Expecti{
    [\BellmanEvaluation{\policy}\Qfnc - \BellmanEvaluation{\policy}
      \Qpi{\policy} ]}{0} & = \Expecti{[ \Qfnc -
      \BellmanEvaluation{\policy}\Qfnc ]}{0} + \discount \Expecti{[
      \Qfnc - \Qpi{\policy} ]}{1},
\end{align}
where we have used the definition of the Bellman evaluation operator
to assert that $\Expecti{ [\BellmanEvaluation{\policy}\Qfnc -
    \BellmanEvaluation{\policy} \Qpi{\policy} ]}{0} = \discount
\Expecti{[ \Qfnc - \Qpi{\policy} ]}{1}$.  Since $\Qfnc -
\BellmanEvaluation{\policy}\Qfnc = \BellmanError{\Qfnc}{\policy}{}$,
the equality~\eqref{EqnBase} is equivalent to the
claim~\eqref{EqnInduction} with $\timestep = 1$.

Turning to the induction step, we now assume that the
claim~\eqref{EqnInduction} holds for some $\timestep \geq 1$, and show
that it holds at step $\timestep + 1$.  By a similar argument, we can
write
\begin{align*}
\discount^{\timestep}\Expecti{[ \Qfnc - \Qpi{\policy}]}{\timestep} =
\discount^{\timestep}\Expecti{[\Qfnc -
    \BellmanEvaluation{\policy}\Qfnc +
    \BellmanEvaluation{\policy}\Qfnc - \BellmanEvaluation{\policy}
    \Qpi{\policy} ]}{\timestep} & = \discount^{\timestep}\Expecti{[
    \Qfnc - \BellmanEvaluation{\policy}\Qfnc ]}{\timestep} +
\discount^{\timestep+1}\Expecti{[ \Qfnc - \Qpi{\policy}
]}{\timestep+1} \\ & =
\discount^{\timestep}\Expecti{\BellmanError{\Qfnc}{\policy}{}}{\timestep}
+ \discount^{\timestep+1}\Expecti{[ \Qfnc - \Qpi{\policy}
]}{\timestep+1}.
\end{align*}
By the induction hypothesis, equality~\eqref{EqnInduction}
holds for $\timestep$, and substituting the above equality
shows that it also holds at time $\timestep + 1$.

Since the equivalence~\eqref{EqnInduction} holds for all $\timestep$,
we can take the limit as $\timestep \rightarrow \infty$, and doing so
yields the claim.


\subsection{Proof of \cref{thm:NewPolicyEvaluation}}
\label{SecProofNewPolicyEvaluation}

The proof relies on proving a high probability bound to \cref{EqnSandwich}
and then invoking \cref{prop:Deterministic} to conclude.

In the statement of the theorem, we require choosing $\epsilon > 0$ to
satisfy the upper bound $\epsilon^2 \precsim \frac{\Rad(\epsilon,
  \delta)}{\numobs}$, and then provide an upper bound in terms of
$\sqrt{\Rad(\epsilon,\delta)/\numobs}$.  It is equivalent to instead
choose $\epsilon$ to satisfy the lower bound $\epsilon^2 \succsim
\frac{\Rad(\epsilon,\delta)}{\numobs}$, and then provide upper bounds
proportional to $\epsilon$.  For the purposes of the proof, the latter
formulation turns out to be more convenient and we pursue it here. \\

To streamline notation, let us introduce the shorthand
$\inprod{f}{\Diff{\policy}(\Qfnc)} \defeq
\inprod{f}{\delta^\policy(\Qfnc)}_\numobs -
\inprod{f}{\BellError^\policy(\Qfnc)}_\mudist$.  For each pair
$(\Qfnc, \policy)$, we then define the random variable
\begin{align*}
  \ZvarShort \defeq \sup_{\TestFunction{} \in
    \TestFunctionClass{\policy} } \frac{ \big|
    \innerprodweighted{\TestFunction{}}{\Diff{\policy}(\Qfnc)}{}\big|}{\TestNormaRegularizerEmp{}}.
\end{align*}
Central to our proof of the theorem is a uniform bound on this random
variable, one that holds for all pairs $(\Qfnc, \policy)$.  In
particular, our strategy is to exhibit some $\epsilon > 0$ for which,
upon setting $\regpar = 4 \epsilon^2$, we have the guarantees
\begin{subequations}
  \begin{align}
\label{EqnSandwichZvar}
\frac{1}{4} \leq \frac{\sqrt{\|f\|_\numobs^2 + \regpar}}{
  \sqrt{\|f\|_\mudist^2 + \regpar}} \leq 2 \qquad & \mbox{uniformly
  for all $f \in \TestFunctionClass{}$, and} \\
\label{EqnUniformZvar}    
\ZvarShort \leq \epsilon \quad & \mbox{uniformly for all $(\Qfnc,
  \policy)$,}
  \end{align}
\end{subequations}
both with probability at least $1 - \delta$.  In particular,
consistent with the theorem statement, we show that this claim holds
if we choose $\epsilon > 0$ to satisfy the inequality
\begin{align}
\label{EqnOriginalChoice}  
\epsilon^2 & \geq \cspec \frac{\Rad(\epsilon, \delta)}{\numobs} \;
\end{align}
where $\cspec > 0$ is a sufficiently large (but universal) constant.

Supposing that the bounds~\eqref{EqnSandwichZvar}
and~\eqref{EqnUniformZvar} hold, let us now establish the set
inclusions claimed in the theorem.

\paragraph{Inclusion  $\PopulationFeasibleSetInfty{\policy} \subseteq
  \EmpiricalFeasibleSet{\policy}(\epsilon)$:}

Define the random variable $\Mplain_\numobs(\Qfnc, \policy) \defeq
\sup \limits_{\TestFunction{} \in \TestFunctionClass{\policy} } \frac{
  |\innerprodweighted{\TestFunction{}}{\BellError^\policy(\Qfnc)}{\mudist}|}{\TestNormaRegularizerEmp{}}$,
and observe that $\Qfnc \in \PopulationFeasibleSetInfty{\policy}$
implies that $\Mplain_\numobs(\Qfnc, \policy) = 0$.  With this
definition, we have
\begin{align*}
\sup_{\TestFunction{} \in \TestFunctionClass{\policy} } \frac{ \big|
  \innerprodweighted{\TestFunction{}}{\delta^{\policy}(\Qfnc)}{\numobs}\big|}{\TestNormaRegularizerEmp{}}
& \stackrel{(i)}{\leq} \Mplain_\numobs(\Qfnc, \policy) + \ZvarShort \;
\stackrel{(ii)}{\leq} \epsilon
\end{align*}
where step (i) follows from the triangle inequality; and step (ii)
follows since $\Mplain_\numobs(\Qfnc, \policy) = 0$, and $\ZvarShort
\leq \epsilon$ from the bound~\eqref{EqnUniformZvar}.

\paragraph{Inclusion  $\EmpiricalFeasibleSet{\policy}(\epsilon) \subseteq
  \PopulationFeasibleSet{\policy}(4 \epsilon)$}

By the definition of $\PopulationFeasibleSet{\policy}(4 \epsilon)$, we
need to show that 
\begin{align*}
    \Mbar(\Qfnc, \policy) \defeq \sup \limits_{\TestFunction{} \in
      \TestFunctionClass{\policy} } \frac{ \big|
      \innerprodweighted{\TestFunction{}}{\BellError^\policy(\Qfnc)}{\mudist}\big|}{\TestNormaRegularizerPop{}}
    \leq 4 \epsilon \qquad \mbox{for any $\Qfnc \in
      \EmpiricalFeasibleSet{\policy}(\epsilon)$.}
\end{align*}
Now we have 
\begin{align*}
\Mbar(\Qfnc, \policy) \stackrel{(i)}{\leq} 2 \Mplain_\numobs(\Qfnc,
\policy) \stackrel{(ii)}{\leq} 2 \left \{
\sup_{\TestFunction{} \in
  \TestFunctionClass{\policy} } \frac{ \big|
  \innerprodweighted{\TestFunction{}}{\delta^{\policy}(\Qfnc)}{\numobs}\big|}{\TestNormaRegularizerEmp{}}
+ \ZvarShort \right \} \; \stackrel{(iii)}{\leq} 2 \big \{ \epsilon +
\epsilon \} \; = \; 4 \epsilon,
\end{align*}
where step (i) follows from the sandwich
relation~\eqref{EqnSandwichZvar}; step (ii) follows from the triangle
inequality and the definition of $\ZvarShort$; and step (iii) follows
since $\ZvarShort \leq \epsilon$ from the
bound~\eqref{EqnUniformZvar}, and
\begin{align*}
\sup_{\TestFunction{} \in \TestFunctionClass{\policy} } \frac{ \big|
  \innerprodweighted{\TestFunction{}}{\delta^{\policy}(\Qfnc)}{\numobs}\big|}{\TestNormaRegularizerEmp{}}
& \leq \epsilon, \qquad \mbox{using the inclusion $\Qfnc \in
  \EmpiricalFeasibleSet{\policy}(\epsilon)$.}
\end{align*}

Consequently, the remainder of our proof is devoted to establishing
the claims~\eqref{EqnSandwichZvar} and~\eqref{EqnUniformZvar}.  In
doing so, we make repeated use of some Bernstein bounds, stated in
terms of the shorthand $\SpecFun(\delta) =
\frac{\log(\nSamples/\delta)}{\numobs}$.
\begin{lemma}
\label{LemBernsteinBound}  
There is a universal constant $c$ such each the following statements
holds with probability at least $1 - \delta$.  For any $f$, we have
\begin{subequations}
  \begin{align}
\label{EqnBernsteinFsquare}
    \Big| \|f\|_\numobs^2 - \|f\|_\mudist^2 \Big| & \leq c \; \Big \{
    \|f\|_\mudist \sqrt{\SpecFun(\delta)} + \SpecFun(\delta) \Big \},
\end{align}
and for any $(\Qfnc, \policy)$ and any function $f$, we have
\begin{align}
 \label{EqnBernsteinBound}
    \big|\inprod{f}{\delta^\policy(\Qfnc)}_\numobs -
    \inprod{f}{\BellError^\policy(\Qfnc)}_\mudist \big| & \leq c \;
    \Big \{ \|f\|_\mudist \sqrt{\SpecFun(\delta)} + \|f\|_\infty
    \SpecFun(\delta) \Big \}.
\end{align}
\end{subequations}
\end{lemma}
\noindent These bounds follow by identifying a martingale difference
sequence, and applying a form of Bernstein's inequality tailored to
the martingale setting.  See Section~\ref{SecProofLemBernsteinBound}
for the details.

\subsection{Proof of the sandwich relation~\eqref{EqnSandwichZvar}}

We claim that (modulo the choice of constants) it suffices to show
that
\begin{align}
\label{EqnCleanSandwich}
\Big| \|f\|_\numobs - \|f\|_\mudist \Big| & \leq \epsilon \qquad
\mbox{uniformly for all $f \in \TestFunctionClass{}$}
\end{align}
for some universal constant $c'$.  Indeed, when this bound holds, we
have
\begin{align*}
  \|f\|_\numobs + 2 \epsilon \leq \|f\|_\mudist + 3 \epsilon \leq
  \frac{3}{2} \{ \|f\|_\mudist + 2\epsilon \}, \quad \mbox{and} \quad
  \|f\|_\numobs + 2 \epsilon  \geq \|f\|_\mudist + \epsilon \geq
  \frac{1}{2} \big \{ \|f\|_\mudist + 2 \epsilon \},
\end{align*}
so that $\frac{\|f\|_\mudist + 2 \epsilon}{\|f \|_\numobs + 2
  \epsilon} \in \big[ \frac{1}{2}, \frac{3}{2} \big]$.  To relate
this statement to the claimed sandwich, observe the inclusion
$\frac{\|f\| + \sqrt{2 \epsilon}}{\sqrt{\|f\|^2 + 4 \epsilon^2}} \in
[1, \sqrt{2}]$, where $\|f\|$ can be either $\|f\|_\numobs$ or
$\|f\|_\mudist$.  Combining this fact with our previous bound, we see
that $\frac{\sqrt{\|f\|_\numobs^2 + 4
    \epsilon^2}}{\sqrt{\|f\|_\mudist^2 + 4 \epsilon^2}} \in \Big[
  \frac{1}{\sqrt{2}} \frac{1}{2}, \frac{3 \sqrt{2}}{2} \Big] \subset
\big[\frac{1}{4}, 3 \big]$, as claimed. \\

  The remainder of our analysis is focused on proving the
  bound~\eqref{EqnCleanSandwich}.  Defining the random variable
  $Y_\numobs(\testfun) = \big| \|f\|_\numobs - \|f\|_\mudist \big|$,
  we need to establish a high probability bound on $\sup_{f \in
    \TestFunctionClass{}} Y_\numobs(\testfun)$.  Let $\{f^1, \ldots,
  f^N \}$ be an $\epsilon$-cover of $\TestFunctionClass{}$ in the
  sup-norm.  For any $f \in \TestFunctionClass{}$, we can find some
  $f^j$ such that $\|f - f^j\|_\infty \leq \epsilon$, whence
\begin{align*}
  Y_\numobs(f) \leq Y_\numobs(f^j) + \big | Y_\numobs(f^j) -
  Y_\numobs(f) \big| & \stackrel{(i)}{\leq} Y_\numobs(f^j) + \big|
  \|f^j\|_\numobs - \|f\|_\numobs \big| + \big| \|f^j\|_\mudist -
  \|f\|_\mudist \big| \\
& \stackrel{(ii)}{\leq} Y_\numobs(f^j) + \|f^j - f\|_\numobs + \|f^j -
  f\|_\mudist \\
& \stackrel{(iii)}{\leq} Y_\numobs(f^j) + 2 \epsilon,
\end{align*}
where steps (i) and (ii) follow from the triangle inequality; and step
(iii) follows from the inequality $\max \{ \|f^j - f\|_\numobs, \|f^j
- f\|_\mudist \} \leq \|f^j - f\|_\infty \leq \epsilon$.  Thus, we
have reduced the problem to bounding a finite maximum.

Note that if $\max \{ \|f^j\|_\numobs, \|f^j\|_\mudist \} \leq
\epsilon$, then we have $Y_\numobs(f^j) \leq 2 \epsilon$ by the
triangle inequality.  Otherwise, we may assume that $\|f^j\|_\numobs +
\|f^j\|_\numobs \geq \epsilon$.  With probability at least $1 -
\delta$, we have
\begin{align*}
\Big| \|f^j\|_\numobs - \|f\|_\mudist \Big| = \frac{ \Big|
  \|f^j\|_\numobs^2 - \|f\|_\mudist^2 \Big|}{\|f^j\|_\numobs +
  \|f^j\|_\mudist} & \stackrel{(i)}{\leq} \frac{ c \big \{
  \|f^j\|_\mudist \sqrt{\SpecFun(\delta)} + \SpecFun(\delta) \big
  \}}{\|f^j\|_\mudist + \|f^j\|_\numobs} \\
& \stackrel{(ii)}{\leq} c \Big \{ \sqrt{\SpecFun(\delta)} +
\frac{\SpecFun(\delta)}{\epsilon} \Big \},
\end{align*}
where step (i) follows from the Bernstein
bound~\eqref{EqnBernsteinFsquare} from Lemma~\ref{LemBernsteinBound},
and step (ii) uses the fact that $\|f^j\|_\numobs + \|f^j\|_\numobs
\geq \epsilon$.

Taking union bound over all $N$ elements in the cover and replacing
$\delta$ with $\delta/N$, we have
\begin{align*}
\max_{j \in [N]} Y_\numobs(f^j) & \leq  c \Big \{
\sqrt{\SpecFun(\delta/N)} + \frac{\SpecFun(\delta/N)}{\epsilon} \Big
\}
\end{align*}
with probability at least $1 - \delta$.  Recalling that $N =
N_\epsilon(\Fclass)$, our choice~\eqref{EqnOriginalChoice} of
$\epsilon$ ensures that $\sqrt{\SpecFun(\delta/N)} \leq c \; \epsilon$
for some universal constant $c$.  Putting together the pieces (and
increasing the constant $\cspec$ in the
choice~\eqref{EqnOriginalChoice} of $\epsilon$ as needed) yields the
claim.


\subsection{Proof of the uniform upper bound~\eqref{EqnUniformZvar}}
\label{SecUniProof}

We need to establish an upper bound on $\ZvarShort$ that that holds
uniformly for all $(\Qfnc, \policy)$.  Our first step is to prove a
high probability bound for a fixed pair.  We then apply a standard
discretization argument to make it uniform in the pair.

Note that we can write $\ZvarShort = \sup_{f \in \TestFunctionClass{}}
\frac{V_\numobs(f)}{\sqrt{\|f\|_\numobs^2 + \regpar}}$, where we have
defined $V_\numobs(f) \defeq |\inprod{f}{\Diff{\policy}(Q)}|$.  Our
first lemma provides a uniform bound on the latter random variables:
\begin{lemma}
\label{LemVbound}  
Suppose that $\epsilon^2 \geq \SpecFun \big(\delta/N_\epsilon(\Fclass)
\big)$.  Then we have
  \begin{align}
\label{EqnVbound}    
V_\numobs(f) & \leq c \big \{ \|f\|_\mudist \epsilon + \epsilon^2 \big
\} \qquad \mbox{for all $f \in \TestFunctionClass{}$}
  \end{align}
  with probability at least $1 - \delta$.
\end{lemma}
\noindent See \cref{SecProofLemVbound} for the proof of this claim. \\

We claim that the bound~\eqref{EqnVbound} implies that, for any fixed
pair $(\Qfnc, \policy)$, we have
\begin{align*}
  Z_\numobs(\Qfnc, \policy) \leq
  c' \epsilon \qquad \mbox{with probability at least $1 - \delta$.}
\end{align*}
Indeed, when Lemma~\ref{LemVbound} holds, for any $f \in
\TestFunctionClass{}$, we can write
\begin{align*}
  \frac{V_\numobs(f)}{\sqrt{\|f\|_\numobs^2 + \regpar}} =
  \frac{\sqrt{\|f\|_\mudist^2 + \regpar}}{\sqrt{\|f\|_\numobs^2 +
      \regpar}} \; \frac{V_\numobs(f)}{\sqrt{\|f\|_\mudist^2 +
      \regpar}} \stackrel{(i)}{\leq} \; 3 \; \frac{c \big \{
    \|f\|_\mudist \epsilon + \epsilon^2 \big \}}{\sqrt{\|f\|_\mudist^2
      + \regpar}} \; \stackrel{(ii)}{\leq} \; c' \epsilon,
\end{align*}
where step (i) uses the sandwich relation~\eqref{EqnSandwichZvar},
along with the bound~\eqref{EqnVbound}; and step (ii) follows given
the choice $\regpar = 4 \epsilon^2$.  We have thus proved that for any
fixed $(\Qfnc, \policy)$ and $\epsilon \geq \SpecFun
\big(\delta/N_\epsilon(\Fclass) \big)$, we have
\begin{align}
\label{EqnFixZvarBound}  
\ZvarShort & \leq c' \epsilon \qquad \mbox{with probability at least
  $1 - \delta$.}
\end{align}

Our next step is to upgrade this bound to one that is uniform over all
pairs $(\Qfnc, \policy)$.  We do so via a discretization argument: let
$\{\Qfnc^j\}_{j=1}^J$ and $\{\pi^k\}_{k=1}^K$ be $\epsilon$-coverings
of $\Qclass{}$ and $\PolicyClass$, respectively.
\begin{lemma}
\label{LemDiscretization}
We have the upper bound
\begin{align}
  \sup_{\Qfnc, \policy} \ZvarShort & \leq \max_{(j,k) \in [J] \times
    [K]} \Zvar{\Qfnc^j}{\policy^k} + 4 \epsilon.
\end{align}
\end{lemma}
\noindent See Section~\ref{SecProofLemDiscretization} for the proof
of this claim.

If we replace $\delta$ with $\delta/(J K)$, then we are guaranteed
that the bound~\eqref{EqnFixZvarBound} holds uniformly over the family
$\{ \Qfnc^j \}_{j=1}^J \times \{ \policy^k \}_{k=1}^K$.  Recalling
that $J = N_\epsilon(\Qclass{})$ and $K = N_\epsilon(\Pi)$, we
conclude that for any $\epsilon$ satisfying the
inequality~\eqref{EqnOriginalChoice}, we have $\sup_{\Qfnc, \policy}
\ZvarShort \leq \ctil \epsilon$ with probability at least $1 -
\delta$.  (Note that by suitably scaling up $\epsilon$ via the choice
of constant $\cspec$ in the bound~\eqref{EqnOriginalChoice}, we can
arrange for $\ctil = 1$, as in the stated claim.)


\subsection{Proofs of supporting lemmas}

In this section, we collect together the proofs
of~\cref{LemVbound,LemDiscretization}, which were stated and used
in~\Cref{SecUniProof}.


\subsubsection{Proof of ~\cref{LemVbound}}
\label{SecProofLemVbound} 
We first localize the problem to the class $\DF(\epsilon) = \{f \in
\TestFunctionClass{} \mid \|f\|_\mu \leq \epsilon \}$.  In particular,
if there exists some $\ftil \in \TestFunctionClass{}$ that
violates~\eqref{EqnVbound}, then the rescaled function $f = \epsilon
\ftil/\|\ftil\|_\mudist$ belongs to $\DF(\epsilon)$, and satisfies
$V_\numobs(f) \geq c \epsilon^2$.  Consequently, it suffices to show
that $V_\numobs(f) \leq c \epsilon^2$ for all $f \in \DF(\epsilon)$.

Choose an $\epsilon$-cover of $\Fclass$ in the sup-norm with $N =
N_\epsilon(\TestFunctionClass{})$ elements.  Using this cover, for any
$f \in \DF(\epsilon)$, we can find some $f^j$ such that $\|f -
f^j\|_\infty \leq \epsilon$.  Thus, for any $f \in \DF(\epsilon)$, we
can write
\begin{align}
\label{EqnOriginalInequality}  
V_\numobs(f) \leq V_\numobs(f^j) + V_\numobs(f - f^j) \; \; \leq
\underbrace{V_\numobs(f^j)}_{\Term_1} + \underbrace{\sup_{g \in
    \Gclass(\epsilon)} V_\numobs(g)}_{\Term_2},
\end{align}
where $\Gclass(\epsilon) \defeq \{ f_1 - f_2 \mid f_1, f_2 \in
\TestFunctionClass{}, \|f_1 - f_2 \|_\infty \leq \epsilon \}$.  We
bound each of these two terms in turn.  In particular, we show that
each of $\Term_1$ and $\Term_2$ are upper bounded by $c \epsilon^2$
with high probability.

\paragraph{Bounding $\Term_1$:}
From the Bernstein bound~\eqref{EqnBernsteinBound}, we have
\begin{align*}
V_\numobs(f^k) & \leq c \big \{ \|f^k\|_\mudist
\sqrt{\SpecFun(\delta/N)} + \|f^k\|_\infty \SpecFun(\delta/N) \big \}
\qquad \mbox{for all $k \in [N]$}
\end{align*}
with probability at least $1 - \delta$.  Now for the particular $f^j$
chosen to approximate $f \in \DF(\epsilon)$, we have
\begin{align*}
\|f^j\|_\mudist & \leq \|f^j - f\|_\mudist + \|f\|_\mudist \leq 2
\epsilon,
\end{align*}
where the inequality follows since $\|f^j - f\|_\mudist \leq \|f^j -
f\|_\infty \leq \epsilon$, and $\|f\|_\mudist \leq \epsilon$.
Consequently, we conclude that
\begin{align*}
\Term_1 & \leq c \Big \{ 2 \epsilon \sqrt{\SpecFun(\delta/N)} +
\SpecFun(\delta/N) \Big \} \; \leq \; c' \epsilon^2 \qquad \mbox{with
  probability at least $1 - \delta$.}
\end{align*}
where the final inequality follows from our choice of $\epsilon$.

\paragraph{Bounding $\Term_2$:}

Define $\Gclass \defeq \{f_1 - f_2 \mid f_1, f_2 \in
\TestFunctionClass{} \}$.  We need to bound a supremum of the process
$\{ V_\numobs(g), g \in \Gclass \}$ over the subset
$\Gclass(\epsilon)$.  From the Bernstein
bound~\eqref{EqnBernsteinBound}, the increments $V_\numobs(g_1) -
V_\numobs(g_2)$ of this process are sub-Gaussian with parameter $\|g_1
- g_2\|_\mudist \leq \|g_1 - g_2\|_\infty$, and sub-exponential with
parameter $\|g_1 - g_2\|_\infty$.  Therefore, we can apply a chaining
argument that uses the metric entropy $\log N_t(\Gclass)$ in the
supremum norm.  Moreover, we can terminate the chaining at $2
\epsilon$, because we are taking the supremum over the subset
$\Gclass(\epsilon)$, and it has sup-norm diameter at most $2
\epsilon$.  Moreover, the lower interval of the chain can terminate at
$2 \epsilon^2$, since our goal is to prove an upper bound of this
order.  Then, by using high probability bounds for the suprema of
empirical processes (e.g., Theorem 5.36 in the
book~\cite{wainwright2019high}), we have
\begin{align*}
  \Term_2 \leq c_1 \; \int_{2 \epsilon^2}^{2 \epsilon} \Mfun \big(
  \frac{\log N_t(\Gclass)}{\numobs} \big) dt + c_2 \big \{ \epsilon
  \sqrt{\SpecFun(\delta)} + \epsilon \SpecFun(\delta) \big\} + 2
  \epsilon^2
\end{align*}
with probability at least $1 - \delta$.  (Here the reader should
recall our shorthand $\Mfun(s) = \max \{s, \sqrt{s} \}$.)

Since $\Gclass$ consists of differences from $\TestFunctionClass{}$,
we have the upper bound $\log N_t(\Gclass) \leq 2 \log
N_{t/2}(\TestFunctionClass{})$, and hence (after making the change of
variable $u = t/2$ in the integrals)
\begin{align*}
  \Term_2 \leq c'_1 \int_{\epsilon^2}^{ \epsilon} \Mfun \big(
  \frac{\log N_u(\TestFunctionClass{})}{\numobs} \big) du + c_2 \big
  \{ \epsilon \sqrt{\SpecFun(\delta)} + \epsilon \SpecFun(\delta)
  \big\} \; \leq \; \ctil \epsilon^2,
\end{align*}
where the last inequality follows from our choice of $\epsilon$.


\subsubsection{Proof of ~\cref{LemDiscretization}}
\label{SecProofLemDiscretization} 

By our choice of the $\epsilon$-covers, for any $(\Qfnc, \policy)$,
there is a pair $(\Qfnc^j, \policy^k)$ such that
\begin{align*}
  \|\Qfnc^j - Q\|_\infty \leq \epsilon, \quad \mbox{and} \quad
  \|\policy^k - \policy\|_{\infty,1} = \sup_{\state} \|\policy^k(\cdot
  \mid \state) - \policy(\cdot \mid \state) \|_1 \leq \epsilon.
\end{align*}
Using this pair, an application of the triangle inequality yields
\begin{align*}
  \big| \Zvar{\Qfnc}{\policy} - \Zvar{\Qfnc^j}{\policy^k} \big| & \leq
  \underbrace{\big| \Zvar{\Qfnc}{\policy} - \Zvar{\Qfnc}{\policy^k}
    \big|}_{\Term_1} + \underbrace{\big| \Zvar{\Qfnc}{\policy^k} -
    \Zvar{\Qfnc^j}{\policy^k} \big|}_{\Term_2}
\end{align*}
We bound each of these terms in turn, in particular proving
that $\Term_1 + \Term_2 \leq 24 \epsilon$.  Putting together the
pieces yields the bound stated in the lemma.

\paragraph{Bounding $\Term_2$:}
From the definition of $\ZvarPlain$, we have
\begin{align*}
\Term_2 = \big| \Zvar{\Qfnc}{\policy^k} - \Zvar{\Qfnc^j}{\policy^k}
\big| & \leq \sup_{f \in \TestFunctionClass{}} \frac{\big|
  \smlinprod{f}{\Diff{\policy^k}(\Qfnc -
    \Qfnc^j)}|}{\sqrt{\|f\|_\numobs^2 + \regpar}}.
\end{align*}
Now another application of the triangle inequality yields
\begin{align*}
  |\smlinprod{f}{\Diff{\policy^k}(\Qfnc - \Qfnc^j)}| & \leq
  |\smlinprod{f}{\TDError{(\Qfnc- \Qfnc^j)}{\policy^k}{}}_\numobs| +
  ||\smlinprod{f}{\BellError^{\policy^k}(\Qfnc- \Qfnc^j)}|_\mudist \\
  & \leq \|f\|_\numobs \|\TDError{(\Qfnc-
    \Qfnc^j)}{\policy^k}{}\|_\numobs + \|f\|_\mudist
  \|\BellError^{\policy^k}(\Qfnc- \Qfnc^j)\|_\mudist \\
  & \leq \max \{ \|f\|_\numobs, \|f\|_\mudist \} \; \Big \{
  \|\TDError{(\Qfnc- \Qfnc^j)}{\policy^k}{}\|_\infty +
  \|\BellError^{\policy^k}(\Qfnc- \Qfnc^j)\|_\infty \Big \}
\end{align*}
where step (i) follows from the Cauchy--Schwarz inequality.  Now in
terms of the shorthand \mbox{$\Delta \defeq \Qfnc - \Qfnc^j$,} we have
\begin{subequations}
  \begin{align}
\label{EqnCoffeeBell}    
\|\BellError^{\policy^k}(\Qfnc- \Qfnc^j)\|_\infty & = \sup_{\psa}
\Big| \Delta \psa - \discount
\E_{\successorstate\sim\TransitionLaw\psa} \big[
  \Delta(\successorstate,\policy) \big] \Big| \leq 2 \|\Delta\|_\infty
\leq 2 \epsilon.
\end{align}
An entirely analogous argument yields
\begin{align}
\label{EqnCoffeeTD}  
  \|\TDError{(\Qfnc- \Qfnc^j)}{\policy^k}{}\|_\infty  \leq 2 \epsilon
\end{align}
\end{subequations}
Conditioned on the sandwich relation~\eqref{EqnSandwichZvar}, we have
$\sup_{f \in \TestFunctionClass{}} \frac{\max \{ \|f\|_\numobs,
  \|f\|_\mudist \}}{\sqrt{\|f\|_\numobs^2 + \regpar}} \leq 4$.
Combining this bound with inequalities~\eqref{EqnCoffeeBell}
and~\eqref{EqnCoffeeTD}, we have shown that $\Term_2 \leq 4 \big \{2
\epsilon + 2 \epsilon \} = 16 \epsilon$.

\paragraph{Bounding $\Term_1$:}
In this case, a similar argument yields
\begin{align*}
  |\smlinprod{f}{(\Diff{\policy} - \Diff{\policy^k})(\Qfnc)}| & \leq
  \max \{ \|f\|_\numobs, \|f\|_\mudist \} \; \Big \{ \|(\delta^\policy
  - \delta^{\policy^k})(\Qfnc)\|_\numobs + \|(\BellmanErr^\policy -
  \BellmanErr^{\policy^k})(\Qfnc)\|_\mudist \}.
\end{align*}
Now we have
\begin{align*}
\|(\delta^\policy - \delta^{\policy^k})(\Qfnc)\|_\numobs & \leq
\max_{i = 1, \ldots, \numobs} \Big| \sum_{\action'}
\big(\policy(\action' \mid \state_i) - \policy^k(\action' \mid
\state_i) \big) \Qfnc(\state^+_i, \action') \Big| \\
& \leq \max_{\state} \sum_{\action'} |\policy(\action' \mid \state) -
\policy^k(\action \mid \state)| \; \|\Qfnc\|_\infty \\
& \leq \epsilon.
\end{align*}
A similar argument yields that $\|(\BellmanErr^\policy -
\BellmanErr^{\policy^k})(\Qfnc)\|_\mudist | \leq \epsilon$, and
arguing as before, we conclude that $\Term_1 \leq 4 \{\epsilon +
\epsilon \} = 8 \epsilon$.


\subsubsection{Proof of Lemma~\ref{LemBernsteinBound}}
\label{SecProofLemBernsteinBound}

Our proof of this claim makes use of the following known Bernstein
bound for martingale differences (cf. Theorem 1 in the
paper~\cite{beygelzimer2011contextual}). Recall the shorthand notation
$\SpecFun(\delta) = \frac{\log(\numobs/\delta)}{\numobs}$.
\begin{lemma}[Bernstein's Inequality for Martingales]
\label{lem:Bernstein}
Let $\{X_t\}_{t \geq 1}$ be a martingale difference sequence with
respect to the filtration $\{ \Filtration_t \}_{t \geq 1}$.  Suppose
that $|X_t| \leq 1$ almost surely, and let $\E_t$ denote expectation
conditional on $\Filtration_t$.  Then for all $\delta \in (0,1)$, we
have
\begin{align}
\Big| \frac{1}{\numobs} \sum_{t=1}^\numobs X_t \Big| & \leq 2 \Big[
  \Big(\frac{1}{\numobs} \sum_{t=1}^\numobs \E_t X^2_t \Big)
  \SpecFun(2 \delta) \Big]^{1/2} + 2 \SpecFun(2 \delta)
\end{align}
with probability at least $1 - \delta$.  
\end{lemma}

\noindent With this result in place, we divide our proof into two
parts, corresponding to the two claims~\eqref{EqnBernsteinBound}
and~\eqref{EqnBernsteinFsquare} stated in~\cref{LemBernsteinBound}.


\paragraph{Proof of the bound~\eqref{EqnBernsteinBound}:}

Recall that at step $i$, the triple $\psai$ is drawn according to a
conditional distribution
$\DatasetDistributionStateActions_\iSample(\cdot \mid \Filtration_i)$.
Similarly, we let $\DatasetDistribution_\iSample$ denote the
distribution of $\sarsi{}$ conditioned on the filtration
$\Filtration_\iSample$.  Note that
$\DatasetDistributionStateActions_\iSample$ is obtained from
$\DatasetDistribution_\iSample$ by marginalizing out the pair
$(\reward, \successorstate)$.  Moreover, by the tower property of
expectation, the Bellman error is equivalent to the average TD error.

Using these facts, we have the equivalence
\begin{align*}
\innerprodweighted{\TestFunction{}} {\TDError{\Qfnc}{\policy}{}}
                  {\DatasetDistribution_{\iSample}} & =
                  \Expecti{\big\{\TestFunctionDefCompact{}{}[\TDErrorDefCompact{\Qfnc}{\policy}{}]\big\}}
                          {\DatasetDistribution_{\iSample}} \\ & =
                          \Expecti{\big\{\TestFunction{}\psai
                            \Expecti{[\TDErrorDefCompact{\Qfnc}{\policy}{}]\big\}}
                                    {\substack{\reward \sim
                                        \RewardLaw\psa,
                                        \successorstate\sim\TransitionLaw\psa}}}
                                  {\psai \sim
                                    \DatasetDistributionStateActions_{\iSample}}
                                  \\ & =
                                  \Expecti{\big\{\TestFunction{}\psai
                                    [\BellmanErrorDefCompact{\Qfnc}{\policy}{}]\big\}}
                                          {\psai \sim
                                            \DatasetDistributionStateActions_{\iSample}}
                                          \\ & =
                                          \innerprodweighted{\TestFunction{}}
                                                            {\BellmanError{\Qfnc}{\policy}{}}
                                                            {\DatasetDistributionStateActions_{\iSample}}.
\end{align*}

As a consequence, we can write
$\inprod{f}{\delta^\policy(\Qfnc)}_\numobs -
\inprod{f}{\BellError^\policy(\Qfnc)}_\mudist = \frac{1}{\numobs}
\sum_{\iSample=1}^\numobs \martone_\iSample$ where
\begin{align*}
\martone_i & \defeq \TestFunctionDefCompact{}{\iSample}
        [\TDErrorDefCompact{\Qfnc}{\policy}{\iSample}] -
        \Expecti{\big\{\TestFunctionDefCompact{}{}{[\TDErrorDefCompact{\Qfnc}{\policy}{}}]\big\}}{\DatasetDistribution_{\iSample}}
\end{align*}
defines a martingale difference sequence (MDS).  Thus, we can prove
the claim by applying a Bernstein martingale inequality.

Since $\|r\|_\infty \leq 1$ and $\|\Qfnc\|_\infty \leq 1$ by
assumption, we have $\|\martone_i\|_\infty \leq 3 \|f\|_\infty$, and
\begin{align*}
\frac{1}{\numobs} \sum_{i=1}^\numobs
\Exp_{\DatasetDistribution_\iSample} [\martone_i^2] & \leq 9 \;
\frac{1}{\numobs} \sum_{i=1}^\numobs
\Exp_{\DatasetDistributionStateActions_{\iSample}}[f^2(\state_{i},\action_{i},
  \identifier_{i})] \; = \; 9 \|f\|_\mudist^2.
\end{align*}
Consequently, the claimed bound~\eqref{EqnBernsteinBound} follows by
applying the Bernstein bound stated in \cref{lem:Bernstein}.

\paragraph{Proof of the bound~\eqref{EqnBernsteinFsquare}:}

In this case, we have the additive decomposition
\begin{align*}
\|\testfunc\|_\numobs^2 - \|\testfun\|_\mudist^2 & = \frac{1}{\numobs}
\sum_{i=1}^\numobs \Big \{
\underbrace{\testfunc^2(\state_{\iSample},\action_{\iSample},\identifier_{\iSample})
  - \Exp_{\DatasetDistributionStateActions_{\iSample}}[f^2(\state,
    \action, \identifier)]}_{\marttwo_i} \Big\},
\end{align*}
where $\{\marttwo_i\}_{i=1}^\numobs$ again defines a martingale
difference sequence.  Note that $\|\marttwo_i\|_\infty \leq 2
\|\testfunc\|_\infty^2 \leq 2$, and
\begin{align*}
\frac{1}{\numobs} \sum_{i=1}^\numobs
\Exp_{\DatasetDistributionStateActions_\iSample} [(\marttwo_i)^2] &
\stackrel{(i)}{\leq} \frac{1}{\numobs} \sum_{i=1}^\numobs
\Exp_{\DatasetDistributionStateActions_\iSample} \big[
  \testfunc^4(\MyState, \Action, \Identifier) \big] \; \leq \;
\|\testfunc\|_\infty^2 \frac{1}{\numobs} \sum_{i=1}^\numobs
\Exp_{\DatasetDistributionStateActions_\iSample}
\big[\testfunc^2(\MyState, \Action, \Identifier) \big] \;
\stackrel{(ii)}{\leq} \; \|\testfunc\|_\mudist^2,
\end{align*}
where step (i) uses the fact that the variance of $\testfunc^2$ is at
most the fourth moment, and step (ii) uses the bound
$\|\testfunc\|_\infty \leq 1$.  Consequently, the claimed
bound~\eqref{EqnBernsteinFsquare} follows by applying the Bernstein
bound stated in \cref{lem:Bernstein}.

\newpage
\section{Proofs for \texorpdfstring{\cref{sec:Applications}}{} and \cref{sec:appConc}}
In this section, we collect together the proofs of results stated
without proof in Section~\ref{sec:Applications} and \cref{sec:appConc}.


\subsection{Proof of \cref{prop:LikeRatio}}
\label{sec:LikeRatio}
\begin{proof}

  Since $\TestFunction{}^* \in \TestFunctionClass{\policy}$, we are
  guaranteed that the corresponding constraint must hold.  It reads as
\begin{align*}
|\Expecti{\frac{1}{\scaling{\policy}}
  \frac{\DistributionOfPolicy{\policy}}{\DatasetDistributionStateActions}
  \BellmanError{\Qfnc}{\policy}{}}{\DatasetDistributionStateActions}
|^2 = \frac{1}{\scaling{\policy}^2} |
\Expecti{\BellmanError{\Qfnc}{\policy}{}}{\policy} |^2 &
\overset{(iii)}{\leq} \big( \frac{1}{\scaling{\policy}^2}
\munorm{\frac{\DistributionOfPolicy{\policy}}{\DatasetDistributionStateActions}
}^2 + \regpar \big) \frac{\Rad}{\numobs}.
\end{align*}
where step (iii) follows from the definition of population constraint.
Re-arranging yields the upper bound
\begin{align*}
\frac{|\Expecti{\frac{\DistributionOfPolicy{\policy}}{\DatasetDistributionStateActions}
    \BellmanError{\Qfnc}{\policy}{}}{\DatasetDistributionStateActions}
  |^2}{(1 + \regpar) \frac{\Rad}{\numobs}} & \leq
\frac{\big(\munorm{\frac{\DistributionOfPolicy{\policy}}{\DatasetDistributionStateActions}
  }^2 + \scalingsq{\policy}\regpar \big) \frac{\Rad}{\numobs}}{(1 +
  \regpar) \frac{\Rad}{\numobs}} \; = \; \frac{
  \Expecti{\Big[\frac{\DistributionOfPolicy{\policy}\PSA}{\DatasetDistributionStateActions\PSA}\Big]
  }{\policy} + \scalingsq{\policy} \regpar} {1 + \regpar},
\end{align*}
where the final step uses the fact that
\begin{align*}
\norm{\frac{\DistributionOfPolicy{\policy}}{\DatasetDistributionStateActions}}{\DatasetDistributionStateActions}^2
=
\Expecti{\frac{\DistributionOfPolicy{\policy}^2\PSA}{\DatasetDistributionStateActions^2\PSA}
}{\DatasetDistributionStateActions} =
\Expecti{\frac{\DistributionOfPolicy{\policy}\PSA}{\DatasetDistributionStateActions\PSA}
}{\policy}
\end{align*}
Thus, we have established the bound (i) in our
claim~\eqref{EqnLikeRatioBound}.

The upper bound (ii) follows immediately since
$\Expecti{\frac{\DistributionOfPolicy{\policy}\psa}{\DatasetDistributionStateActions\psa}
}{\policy} \leq \sup_{\psa}
\frac{\DistributionOfPolicy{\policy}\psa}{\DatasetDistributionStateActions\psa}
\leq \scaling{\policy}$.

\end{proof}

\subsection{Proof of \texorpdfstring{\cref{lem:PredictionError}}{}}
\label{sec:PredictionError}

Some simple algebra yields
\begin{align*}
\BellmanError{\Qfnc}{\policy}{} -
\BellmanError{\QpiWeak{\policy}}{\policy}{} = [\Qfnc -
  \BellmanEvaluation{\policy}\Qfnc] - [\QpiWeak{\policy} -
  \BellmanEvaluation{\policy}\QpiWeak{\policy}] = (\IdentityOperator -
\gamma\TransitionOperator{\policy}) (\Qfnc - \QpiWeak{\policy}) =
(\IdentityOperator - \gamma\TransitionOperator{\policy}) \QfncErr.
\end{align*}
Taking expectations under $\policy$ and recalling that
$\innerprodweighted{\TestFunction{}}{\BellmanError{\QpiWeak{\policy}}{\policy}{}}{\policy}
= 0$ for all $\TestFunction{} \in \TestFunctionClass{\policy}$ yields
\begin{align*}
\innerprodweighted{\TestFunction{}}
{\BellmanError{\Qfnc}{\policy}{}}
{\policy}
=
\innerprodweighted{\TestFunction{}}
{(\IdentityOperator -  \gamma\TransitionOperator{\policy})\QfncErr}
{\policy}.
\end{align*}
Notice that for any $\Qfnc \in \Qclass{\policy}$ there exists a test
function $\QfncErr = \Qfnc - \QpiWeak{\policy} \in
\QclassErr{\policy}$, and the associated population constraint reads
\begin{align*}
 \frac{ \big| \innerprodweighted{\QfncErr} {(\IdentityOperator -
     \gamma\TransitionOperator{\policy})\QfncErr}
   {\DatasetDistributionStateActions} \big| } { \sqrt{
     \norm{\QfncErr}{\DatasetDistributionStateActions}^2 +
     \TestFunctionReg } } & \leq \sqrt{\frac{\Rad}{\numobs}}.
\end{align*}
Consequently, the {\goodname} can be upper bounded as
\begin{align*}
\ConcSimple & \leq \max_{\QfncErr \in \QclassErrCentered{\policy}}
\Big \{ \frac{\Rad}{\numobs} \; \frac{ \innerprodweighted{\1}
  {(\IdentityOperator - \gamma\TransitionOperator{\policy})\QfncErr}
  {\policy}^2 } {1+ \TestFunctionReg} \Big \} \; \leq \;
\max_{\QfncErr \in \QclassErrCentered{\policy}} \Big \{ \frac{
  \norm{\QfncErr}{\DatasetDistributionStateActions}^2 +
  \TestFunctionReg } { \norm{\1}{\policy}^2 + \TestFunctionReg } \:
\frac{ \innerprodweighted{\1} {(\IdentityOperator -
    \gamma\TransitionOperator{\policy})\QfncErr} {\policy}^2 } {
  \innerprodweighted{\QfncErr} {(\IdentityOperator -
    \gamma\TransitionOperator{\policy})\QfncErr}
                    {\DatasetDistributionStateActions}^2 } \Big \},
\end{align*}
as claimed in the bound~\eqref{EqnPredErrorBound}.


\subsection{Proof of \texorpdfstring{\cref{lem:PredictionErrorBellmanClosure}}{}}

\label{sec:PredictionErrorBellmanClosure}

If weak Bellman closure holds, then we can write
\begin{align*}
\BellmanError{\Qfnc}{\policy}{} = \Qfnc -
\BellmanEvaluation{\policy}\Qfnc = \Qfnc - \QpiProj{\policy}{\Qfnc}
\in \QclassErr{\policy}.
\end{align*}
For any $\Qfnc \in \Qclass{\policy}$, the function $\QfncErrNC = \Qfnc
- \QpiProj{\policy}{\Qfnc}$ belongs to $\QclassErr{\policy}$, and the
associated population constraint reads $\frac{
  |\innerprodweighted{\QfncErrNC} {\QfncErrNC}
  {\DatasetDistributionStateActions} |} { \sqrt{
    \norm{\QfncErrNC}{\DatasetDistributionStateActions}^2 +
    \TestFunctionReg }} \leq \sqrt{\frac{\Rad}{\numobs}}$.
Consequently, the {\goodname} is upper bounded as
\begin{align*}
  \ConcSimple & \leq \max_{\QfncErrNC \in \QclassErr{\policy}} \Big \{
  \frac{\numobs}{\Rad} \; \frac{ v\innerprodweighted{\1} {\QfncErrNC}
    {\policy}^2 } {1+ \TestFunctionReg} \Big \} \leq
  \max_{\QfncErrNC \in \QclassErr{\policy}} \Big \{ \frac{
    \norm{\QfncErrNC}{\DatasetDistributionStateActions}^2 +
    \TestFunctionReg } { 1 + \TestFunctionReg } \; \frac{
    \innerprodweighted{\1} {\QfncErrNC} {\policy}^2 } {
    \innerprodweighted{\QfncErrNC} {\QfncErrNC}
                      {\DatasetDistributionStateActions}^2 } \Big \}
  \; \leq \; \max_{\QfncErrNC \in \QclassErr{\policy}} \Big \{ \frac{
    \innerprodweighted{\1} {\QfncErrNC} {\policy}^2 } {
    \innerprodweighted{\QfncErrNC} {\QfncErrNC}
                      {\DatasetDistributionStateActions}^2 } \Big \},
\end{align*}
where the final inequality follows from the fact that
$\norm{\QfncErrNC}{\DatasetDistributionStateActions} \leq 1$.


\subsection{Proof of \texorpdfstring{\cref{lem:BellmanTestFunctions}}{}}
\label{sec:BellmanTestFunctions}

We split our proof into the two separate claims.

\paragraph{Proof of the bound~\eqref{EqnBellBound}:}     
When the test function class includes
$\TestFunctionClassBubnov{\policy}$, then any $\Qfnc$ feasible must
satisfy the population constraints
\begin{align*}
  \frac{\innerprodweighted{ \BellmanError{\QfncTest}{\policy}{}}
    {\BellmanError{\Qfnc}{\policy}{}}{\DatasetDistributionStateActions}}
       {\sqrt{\norm{\BellmanError{\QfncTest}{\policy}{}}{\DatasetDistributionStateActions}^2
           + \TestFunctionReg}} \leq \sqrt{\frac{\Rad}{\numobs}},
       \qquad \mbox{for all $\QfncTest \in \Qpi{\policy}$.}
\end{align*}
Setting $\QfncTest = \Qfnc$ yields $\frac{
  \norm{\BellmanError{\Qfnc}{\policy}{}}{\DatasetDistributionStateActions}^2}
{ \sqrt{
    \norm{\BellmanError{\Qfnc}{\policy}{}}{\DatasetDistributionStateActions}^2
    + \TestFunctionReg } } \leq \sqrt{\frac{\Rad}{\numobs}}$.  If
$\norm{\BellmanError{\Qfnc}{\policy}{}}{\DatasetDistributionStateActions}^2
\geq \TestFunctionReg$, then the claim holds, given our choice
$\TestFunctionReg = c \frac{\Rad}{\numobs}$ for some constant
$c$. Otherwise, the constraint can be weakened to $\frac{
  \norm{\BellmanError{\Qfnc}{\policy}{}}{\DatasetDistributionStateActions}^2}
     { \sqrt{
         2\norm{\BellmanError{\Qfnc}{\policy}{}}{\DatasetDistributionStateActions}^2
     } } \leq \sqrt{\frac{\Rad}{\numobs}}$, which yields the
     bound~\eqref{EqnBellBound}.

\paragraph{Proof of the bound~\eqref{EqnBellBoundConc}:}     
We now prove the sequence of inequalities stated in
equation~\eqref{EqnBellBoundConc}.  Inequality (i) follows directly
from the definition of $\ConcSimple$ and
\cref{lem:BellmanTestFunctions}.  Turning to inequality (ii),
an application of
Jensen's inequality yields
\begin{align*}
  \innerprodweighted{\1}{\BellmanError{\Qfnc}{\policy}{}}{\policy}^2 =
                    [\Expecti{\BellmanError{\Qfnc}{\policy}{}}{\policy}]^2
                    \leq
                    \Expecti{[\BellmanError{\Qfnc}{\policy}{}]^2}{\policy}
                    =
                    \norm{\BellmanError{\Qfnc}{\policy}{}}{\policy}^2.
\end{align*}
Finally, inequality (iii) follows by observing that
\begin{align*}
	\sup_{\Qfnc \in \Qclass{\policy}}
        \frac{\norm{\BellmanError{\Qfnc}{\policy}{}}{\policy}^2}
             {\norm{\BellmanError{\Qfnc}{\policy}{}}{\DatasetDistributionStateActions}^2}
             = \sup_{\Qfnc \in \Qclass{\policy}}
             \frac{\Expecti{[(\BellmanError{\Qfnc}{\policy}{})\psa]^2}{\policy}}{\Expecti{[(\BellmanError{\Qfnc}{\policy}{})\psa]^2}{\DatasetDistributionStateActions}}
             = \sup_{\Qfnc \in \Qclass{\policy}} \frac{\Expecti{ \Big[
                   \frac{\DistributionOfPolicy{\policy}\psa
                   }{\DatasetDistributionStateActions \psa}
                 }{\DatasetDistributionStateActions}
                 \Big][(\BellmanError{\Qfnc}{\policy}{})\psa]^2
             }{\Expecti{[(\BellmanError{\Qfnc}{\policy}{})\psa]^2}{\DatasetDistributionStateActions}}
             \leq \sup_{\psa} \frac{\DistributionOfPolicy{\policy}\psa
             }{\DatasetDistributionStateActions \psa}.
\end{align*}


\newpage
\section{Proofs for the Linear Setting}

We now prove the results stated in~\cref{sec:Linear}.  Throughout this
section, the reader should recall that $\Qfnc$ takes the linear
function $\Qfnc \psa = \inprod{\CriticPar{}}{\LinPhi \psa}$, so that
the bulk of our arguments operate directly on the weight vector
$\CriticPar{} \in \R^\dim$.

%
%
%
%
%

%
%
 
 Given the linear structure, the population and empirical covariance
matrices of the feature vectors play a central role.  We make use of
the following known result (cf.  Lemma 1 in the
paper~\cite{zhang2021optimal}) that relates these objects:
\begin{lemma}[Covariance Concentration]
\label{lem:CovarianceConcentration}
There are universal constants $(c_1, c_2, c_3)$ such that for any
$\delta \in (0, 1)$, we have
\begin{align}
\cOne \CovarianceStandardExplicit \preceq
\frac{1}{\nSamples}\CovarianceEmpiricalStandardExplicit +
\frac{\cTwo}{\nSamples} \log \frac{\nSamples
  \dim}{\FailureProbability}\Identity \preceq \cThree
\CovarianceStandardExplicit + \frac{\cFour}{\nSamples} \log
\frac{\nSamples \dim}{\FailureProbability}\Identity.
\end{align}
with probability at least $1 - \FailureProbability$.
\end{lemma}


\subsection{Proof of \texorpdfstring{\cref{prop:LinearConcentrability}}{}}
\label{sec:LinearConcentrability}

Under weak realizability,
we have
\begin{align}
  \innerprodweighted {\TestFunction{\iConstraint}}
                     {\BellmanError{\QpiWeak{\policy}}{\policy}{}}
                     {\DatasetDistributionStateActions} = 0 \qquad
                     \mbox{for all $j = 1, \ldots, \dim$.}
\end{align}
Thus, at $\psa$ the Bellman error difference reads
\begin{align*}
\BellmanError{\Qfnc}{\policy}{}\psa -
\BellmanError{\QpiWeak{\policy}}{\policy}{}\psa & = [\Qfnc -
  \BellmanEvaluation{\policy}\Qfnc]\psa - [\QpiWeak{\policy} -
  \BellmanEvaluation{\policy} \QpiWeak{\policy} ]\psa \\
& = [\Qfnc - \QpiWeak{\policy} ]\psa - \discount \Expecti{[\Qfnc -
    \QpiWeak{\policy}](\successorstate,\policy)}{\successorstate \sim
  \TransitionLaw\psa } \\
\numberthis{\label{eqn:LinearBellmanError}} & = \inprod{\CriticPar{} -
  \CriticParBest{\policy}}{ \LinPhi\psa - \discount
  \LinPhiBootstrap{\policy}\psa}
\end{align*}
To proceed we need the following auxiliary result:
\begin{lemma}[Linear Parameter Constraints]
\label{lem:RelaxedLinearConstraints}
With probability at least $1-\FailureProbability$, there exists a
universal constant $\cOne > 0$ such that if $\Qfnc \in
\PopulationFeasibleSet{\policy}$ then $ \norm{\CriticPar{} -
  \CriticParBest{\policy}}{\CovarianceWithBootstrapReg{\policy}}^2
\leq \cOne \frac{\dim \Rad}{\nSamples}$.
\end{lemma}
\noindent See \cref{sec:RelaxedLinearConstraints} for the proof.

\newcommand{\IntermediateSmall}[4]{Using this lemma, we can bound the
  \shortgoodname coefficient as follows
\begin{align*}
  \ConcSimple \overset{(i)}{\leq}
  \frac{\nSamples}{\Rad} \; \max_{\Qfnc \in
    \PopulationFeasibleSet{\policy}} \innerprodweighted{\1} {
    \BellmanError{\Qfnc}{\policy}{} #4} {\policy}^2 
  & \overset{(ii)}{\leq} \frac{\nSamples}{\Rad} \; [\Expecti{(#1)^\top}{\policy} (#2)]^2
  \\
  & \overset{(iii)}{\leq} \frac{\nSamples}{\Rad} \: \norm{\Expecti{ #1
    }{\policy}}{(#3)^{-1}}^2 \norm{#2}{#3}^2 \\
& \leq \cOne \dim \norm{\Expecti{ #1}{\policy}}{(#3)^{-1}}^2.
\end{align*}
Here step $(i)$ follows from the definition of \goodname, 
$(ii)$ leverages the linear structure 
and $(iii)$ is Cauchy-Schwartz.
}


\IntermediateSmall
{\LinPhi - \discount \LinPhiBootstrap{\policy}}
{\CriticPar{} - \CriticParBest{\policy}}
{\CovarianceWithBootstrapReg{\policy}}
{-\BellmanError{\QpiWeak{\policy}}{\policy}{}}

\subsection{Proof of \texorpdfstring{\cref{lem:RelaxedLinearConstraints}}{}}
\label{sec:RelaxedLinearConstraints}

\intermediate{(\LinPhi - \discount\LinPhiBootstrap{\policy})^\top}
             {(\CovarianceStandardReg - \discount
               \CovarianceBootstrap{\policy})} {(\CriticPar{} -
               \CriticParBest{\policy})}
             {\CovarianceWithBootstrapReg{\policy}}
             {(\CovarianceStandard - \discount
               \CovarianceBootstrap{\policy})}
             {\cref{eqn:LinearBellmanError}}

             
\subsection{Proof of \cref{prop:LinearConcentrabilityBellmanClosure}}
\label{sec:LinearConcentrabilityBellmanClosure}

Under weak Bellman closure, we have
\begin{align}
\numberthis{\label{eqn:LinearBellmanErrorWithClosure}}  
\BellmanError{\Qfnc}{\policy}{} = \Qfnc -
\BellmanEvaluation{\policy}\Qfnc = \LinPhi^\top(\CriticPar{} -
\CriticParProjection{\policy}).
\end{align}
With a slight abuse of notation, let $\QpiProj{\policy}{\CriticPar{}}$
denote the weight vector that defines the action-value function
$\QpiProj{\policy}{\Qfnc}$.  We introduce the following auxiliary
lemma:

\begin{lemma}[Linear Parameter Constraints with Bellman Closure]
\label{lem:RelaxedLinearConstraintsBellmanClosure}
With probability at least $1-\FailureProbability$, if $\Qfnc \in
\PopulationFeasibleSet{\policy}$ then $ \norm{\CriticPar{} -
  \CriticParProjection{\policy}}{\CovarianceStandardReg }^2 \leq
\cOne\frac{\dim \Rad}{\nSamples}$.
\end{lemma}
See~\cref{sec:RelaxedLinearConstraintsBellmanClosure} for the proof.
%
\IntermediateSmall
{\LinPhi}
{\CriticPar{} - \CriticParProjection{\policy}}
{\CovarianceStandardReg}
{}
%


\subsection{Proof of
\texorpdfstring{\cref{sec:RelaxedLinearConstraintsBellmanClosure}}{}}
\label{sec:RelaxedLinearConstraintsBellmanClosure}

\intermediate{\LinPhi^\top} {\CovarianceStandardReg} {(\CriticPar{} -
  \CriticParProjection{\policy})} {\CovarianceStandardReg}
             {\CovarianceStandard}
             {\cref{eqn:LinearBellmanErrorWithClosure}}

             

\newpage
\section{Proof of \texorpdfstring{\cref{thm:LinearApproximation}}{}}
\label{sec:LinearApproximation}

In this section, we prove the guarantee on our actor-critic procedure
stated in \cref{thm:LinearApproximation}.

\hidecom{ The proof consists of four main steps.  First, we show that
  the critic linear $\QminEmp{\policy}$ function can be interpreted as
  the \emph{exact} value function on an MDP with a perturbed reward
  function.  Second, we show that the global update rule in
  \cref{eqn:LinearActorUpdate} is equivalent to an instantiation of
  Mirror descent where the gradient is the $\Qfnc$ of the adversarial
  MDP.  Third, we analyze the progress of Mirror descent in finding a
  good solution on the sequence of adversarial MDP identified by the
  critic.  Finally we put everything together to derive a performance
  bound.  }

\subsection{Adversarial MDPs}
\label{sec:AdversarialMDP}

We now introduce sequence of adversarial MDPs
$\{\MDPadv{\iter}\}_{\iter=1}^\nIter$ used in the analysis.  Each MDP
$\MDPadv{\iter}$ is defined by the same state-action space and
transition law as the original MDP $\MDP$, but with the reward
functions $\Reward$ perturbed by $\RewardLawAdv{\iter}$---that
is
\begin{align}
\label{eqn:AdversarialMDP}
\MDPadv{\iter} \defeq \langle \StateSpace,\ActionSpace, \RewardLaw +
\RewardLawAdv{\iter}, \TransitionLaw, \discount \rangle.
\end{align}
For an arbitrary policy $\policy$, we denote with
$\QpiAdv{\iter}{\policy}$ and with $\ApiAdv{\iter}{\policy}$ the
action value function and the advantage function on $\MDPadv{\iter}$;
the value of $\policy$ from the starting distribution
$\startdistribution$ is denoted by $\VpiAdv{\iter}{\policy}$.  We
immediately have the following expression for the value function,
which follows because the dynamics of $\MDPadv{\iter}$ and $\MDP$ are
identical and the reward function of $\MDPadv{\iter}$ equals that of
$\MDP$ plus $\RewardLawAdv{\iter}$
	\begin{align}
	\label{eqn:VonAdv}
		\VpiAdv{\iter}{\policy} 
		\defeq \horizon \Expecti{
		\Big[ 
		\RewardLaw 
		+ \RewardLawAdv{\iter}
		\Big]}{\policy}.
	\end{align}

Consider the action value function
 $\QminEmp{\ActorPolicy{\iter}}$
returned by the critic,
and let the reward perturbation 
$\RewardLawAdv{\iter} 
= 
\BellmanError{\QminEmp{\ActorPolicy{\iter}}}{\ActorPolicy{\iter}}{}$ 
be the Bellman error of the critic value function 
$\QminEmp{\ActorPolicy{\iter}}$.
The special property of $\MDPadv{\iter}$ is that 
the action value function of $\ActorPolicy{\iter}$ on $\MDPadv{\iter}$ 
equals the critic lower estimate $\QminEmp{\ActorPolicy{\iter}}$.
\begin{lemma}[Adversarial MDP Equivalence]
  \label{lem:QfncOnAdversarialMDP}
Given the perturbed MDP $\MDPadv{\iter}$ from
equation~\eqref{eqn:AdversarialMDP} with $\RewardLawAdv{\iter} \defeq
\BellmanError{\QminEmp{\ActorPolicy{\iter}}}{\ActorPolicy{\iter}}{}$,
we have the equivalence
\begin{align*}
  \QpiAdv{\iter}{\ActorPolicy{\iter}} = \QminEmp{\ActorPolicy{\iter}}.
\end{align*}
\end{lemma}
\begin{proof}
We need to check that $\QminEmp{\ActorPolicy{\iter}}$ solves the
Bellman evaluation equations for the adversarial MDP, ensuring that
$\QminEmp{\ActorPolicy{\iter}}$ is the action-value function of
$\ActorPolicy{\iter}$ on $\MDPadv{\iter}$.  Let
$\BellmanEvaluation{\ActorPolicy{\iter}}_\iter$ be the Bellman
evaluation operator on $\MDPadv{\iter}$ for policy
$\ActorPolicy{\iter}$.  We have
\begin{align*}
\QminEmp{\ActorPolicy{\iter}} -
\BellmanEvaluation{\ActorPolicy{\iter}}_\iter(\QminEmp{\ActorPolicy{\iter}})
= \QminEmp{\ActorPolicy{\iter}} -
\BellmanEvaluation{\ActorPolicy{\iter}}(\QminEmp{\ActorPolicy{\iter}})
- \RewardLawAdv{\iter} & =
\BellmanError{\QminEmp{\ActorPolicy{\iter}}}{\ActorPolicy{\iter}}{} -
\BellmanError{\QminEmp{\ActorPolicy{\iter}}}{\ActorPolicy{\iter}}{} =
0.
\end{align*}
Thus, the function $\QminEmp{\ActorPolicy{\iter}}$ is the action value
function of $\ActorPolicy{\iter}$ on $\MDPadv{\iter}$, and it is by
definition denoted by $\QpiAdv{\iter}{\ActorPolicy{\iter}}$.
\end{proof}

This lemma shows that the action-value function
$\QminEmp{\ActorPolicy{\iter}}$ computed by the critic is equivalent
to the action-value function of $\ActorPolicy{\iter}$ on
$\MDPadv{\iter}$.  Thus, we can interpret the critic as performing a
model-based pessimistic estimate of $\ActorPolicy{\iter}$; this view
is useful in the rest of the analysis.

\subsection{Equivalence of Updates}
The second step is to establish the equivalence between the update
rule~\eqref{eqn:LinearActorUpdate}, or equivalently as the
update~\eqref{eqn:GlobalRule}, to the exponentiated gradient update
rule~\eqref{eqn:LocalRule}.
\begin{lemma}[Equivalence of Updates]
\label{lem:UpdateEquivalence}
For linear $Q$-functions of the form $\QpiAdv{\iter}{} \psa =
\inprod{\CriticPar{\iter}}{\LinPhi \psa}$, the parameter update
\begin{subequations} 
\begin{align}
\label{eqn:GlobalRule}	
\ActorPolicy{\iter+1} \pas & \propto \exp(\LinPhi\psa^\top
(\ActorPar{\iter} + \LearningRate \CriticPar{\iter})), \qquad
\intertext{is equivalent to the policy update}
\label{eqn:LocalRule}
\ActorPolicy{\iter+1}\pas & \propto \ActorPolicy{\iter}\pas
\exp(\LearningRate\QpiAdv{\iter}{} \psa), \qquad \ActorPolicy{1}\pas =
\frac{1}{\card{\ActionSpace_{\state}}}.
\end{align}
\end{subequations}
\end{lemma}
\begin{proof}
We prove this claim via induction on $\iter$.  The base case ($\iter =
1$) holds by a direct calculation.  Now let us show that the two
update rules update $\ActorPolicy{\iter}$ in the same way.  As an
inductive step, assume that both rules maintain the same policy
$\ActorPolicy{\iter} \propto \exp(\LinPhi\psa^\top\ActorPar{\iter})$
at iteration $\iter$; we will show the policies are still the same at
iteration $\iter+1$.  At any $\psa$, we have
\begin{align*}
\ActorPolicy{\iter+1}(\action \mid \state) \propto
\exp(\LinPhi\psa^\top (\ActorPar{\iter} +
\LearningRate\CriticPar{\iter})) & \propto \exp(\LinPhi\psa^\top
\ActorPar{\iter}) \exp(\LearningRate\LinPhi\psa^\top
\CriticPar{\iter}) \\ & \propto \ActorPolicy{\iter}(\action \mid
\state) \exp(\LearningRate\QpiAdv{\iter}{} \psa).
\end{align*}
\end{proof}

Recall that $\ActorPar{\iter}$ is the parameter associated to
$\ActorPolicy{\iter}$ and that $\CriticPar{\iter}$ is the parameter
associated to $\QminEmp{\ActorPolicy{\iter}}$.  Using
\cref{lem:UpdateEquivalence} together with
\cref{lem:QfncOnAdversarialMDP} we obtain that the actor policy
$\ActorPolicy{\iter}$ satisfies through its parameter
$\ActorPar{\iter}$ the mirror descent update
rule~\eqref{eqn:LocalRule} with $\QpiAdv{\iter}{} =
\QminEmp{\ActorPolicy{\iter}} = \QpiAdv{\iter}{\ActorPolicy{\iter}}$
and $\ActorPolicy{1}(\action \mid \state) =
1/\abs{\ActionSpace_\state}, \; \forall \psa$.  In words, the actor is
using Mirror descent to find the best policy on the sequence of
adversarial MDPs $\{\MDPadv{\iter} \}$ implicitly identified by the
critic.

\subsection{Mirror Descent on Adversarial MDPs}

Our third step is to analyze the behavior of mirror descent on the MDP
sequence $\{\MDPadv{\iter}\}_{\iter=1}^\nIter$, and then translate
such guarantees back to the original MDP $\MDP$.  The following result
provides a bound on the average of the value functions
$\{\Vpi{\ActorPolicy{\iter}} \}_{\iter=1}^\nIter$ induced by the
actor's policy sequence.  This bound involves a form of optimization
error\footnote{Technically, this error should depend on
$\card{\ActionSpace_\state}$, if we were to allow the action spaces to
have varyign cardinality, but we elide this distinction here.}  given
by
\begin{align*}
\MirrorRegret{\nIter} & = 2 \, \sqrt{ \frac{2 \log
    |\ActionSpace|}{\nIter}},
\end{align*}
as is standard in mirror descent schemes.  It also involves the
\emph{perturbed rewards} given by $\RewardLawAdv{\iter} \defeq
\BellmanError{\QpiAdv{\iter}{\ActorPolicy{\iter}}}
             {\ActorPolicy{\iter}}{}$.
\begin{lemma}[Mirror Descent on Adversarial MDPs]
\label{prop:MirrorDescentAdversarialRewards}
For any positive integer $\nIter$, applying the update
rule~\eqref{eqn:LocalRule} with $\QpiAdv{\iter}{} =
\QpiAdv{\iter}{\ActorPolicy{\iter}}$ for $\nIter$ rounds yields a
sequence such that
\begin{align}
  \label{eqn:MirrorDescentAdversarialRewards}
  \frac{1}{\nIter} \sumiter \Big[ \Vpi{\comparator} -
    \Vpi{\ActorPolicy{\iter}} \Big] \leq \horizon \left \{
  \MirrorRegret{\nIter} + \frac{1}{\nIter} \sumiter \Big[ -
    \Expecti{\RewardLawAdv{\iter}}{\comparator} +
    \Expecti{\RewardLawAdv{\iter}}{\ActorPolicy{\iter}} \Big] \right
  \},
\end{align}
valid for any comparator policy $\comparator$.
\end{lemma}
\noindent See \cref{sec:MirrorDescentAdversarialRewards} for the
proof. \\

To be clear, the comparator policy $\comparator$ need belong to the
soft-max policy class.  Apart from the optimization error term, our
bound~\eqref{eqn:MirrorDescentAdversarialRewards} involves the
behavior of the perturbed rewards $\RewardLawAdv{\iter}$ along the
comparator $\comparator$ and $\ActorPolicy{\iter}$, respectively.
These correction terms arise because the actor performs the policy
update using the action-value function
$\QpiAdv{\iter}{\ActorPolicy{\iter}}$ on the perturbed MDPs instead of
the real underlying MDP.

\subsection{Pessimism: Bound on 
\texorpdfstring{$\Expecti{\RewardLawAdv{\iter}}{\ActorPolicy{\iter}}$}{}}

The fourth step of the proof is to leverage the pessimistic estimates
returned by critic to simplify
equation~\eqref{eqn:MirrorDescentAdversarialRewards}.  Using
\cref{lem:Simulation} and the definition of adversarial reward
$\RewardLawAdv{\iter}$ we can write
\begin{align*}
	\VminEmp{\ActorPolicy{}} - \Vpi{\ActorPolicy{\iter}} =
        \horizon
        \innerprodweighted{\1}{\BellmanError{\QminEmp{\ActorPolicy{\iter}}}{\ActorPolicy{\iter}}{}}{\ActorPolicy{\iter}}
        = \horizon
        \Expecti{\BellmanError{\QminEmp{\ActorPolicy{\iter}}}{\ActorPolicy{\iter}}{}}{\ActorPolicy{\iter}}
        & = \horizon
        \Expecti{\RewardLawAdv{\iter}}{\ActorPolicy{\iter}}.
\end{align*}
Since weak realizability holds, \cref{thm:NewPolicyEvaluation}
guarantees that $\VminEmp{\ActorPolicy{}} \leq \Vpi{\policy}$
uniformly for all $\policy \in \PolicyClass$ with probability at least
$1-\FailureProbability$.  Coupled with the prior display, we find that
\begin{align}
\label{eqn:PessimisticAdversarialReward}
	\Expecti{\RewardLawAdv{\iter}}{\ActorPolicy{\iter}}  \leq 0.
\end{align}
Using the above display, the result in
\cref{eqn:MirrorDescentAdversarialRewards} can be further upper
bounded and simplified.

\subsection{Concentrability: Bound on 
\texorpdfstring{$\Expecti{\RewardLawAdv{\iter}}{\comparator}$}{}} The
term $\Expecti{\RewardLawAdv{\iter}}{\comparator}$ can be interpreted
as an approximate concentrability factor for the approximate algorithm
that we are investigating.

\paragraph{Bound under only weak realizability:}
\cref{lem:RelaxedLinearConstraints} gives with probability at least
$1-\FailureProbability$ that any surviving $\Qfnc$ in
$\PopulationFeasibleSet{\ActorPolicy{\iter}}$ must satisfy: $ \norm{
  \CriticPar{} -
  \CriticParBest{\ActorPolicy{\iter}}}{\CovarianceWithBootstrapReg{\ActorPolicy{\iter}}}^2
\lesssim \frac{\dim \Rad}{\nSamples} $ where
$\CriticParBest{\ActorPolicy{\iter}}$ is the parameter associated to
the weak solution $\QpiWeak{\ActorPolicy{\iter}}$.  Such bound must
apply to the parameter $ \CriticPar{\iter} \in
\EmpiricalFeasibleSet{\ActorPolicy{\iter}}$ identified by the
critic.\footnote{We abuse the notation and write $\CriticPar{} \in
\EmpiricalFeasibleSet{\policy}$ in place of $\Qfnc \in
\EmpiricalFeasibleSet{\policy}$}.

We are now ready to bound the remaining adversarial reward along the
distribution of the comparator $\comparator$.
\begin{align*}
  \abs{\Expecti{\RewardLawAdv{\iter}}{\comparator}} & =
  \abs{\Expecti{\BellmanError{\QminEmp{\ActorPolicy{\iter}}}{\ActorPolicy{\iter}}{}}{\comparator}}
  \\ & \overset{\text(i)}{=} \abs{\Expecti{ (\LinPhi - \discount
      \LinPhiBootstrap{\ActorPolicy{\iter}})^\top (\CriticPar{\iter} -
      \CriticParBest{\ActorPolicy{\iter}}) }{\comparator}} \\ & \leq
  \norm{\Expecti{[\LinPhi - \discount
        \LinPhiBootstrap{\ActorPolicy{\iter}}]}{\comparator}}{(\CovarianceWithBootstrapReg{\ActorPolicy{\iter}})^{-1}}
  \norm{\CriticPar{\iter} -
    \CriticParBest{\ActorPolicy{\iter}}}{\CovarianceWithBootstrapReg{\ActorPolicy{\iter}}}
  \\ & \leq c \; \sqrt{\frac{\dim \Rad}{\nSamples}} \; \sup_{\policy
    \in \PolicyClass} \left \{ \norm{\Expecti{[\LinPhi - \discount
        \LinPhiBootstrap{\policy}]}{\comparator}}{(\CovarianceWithBootstrapReg{\policy})^{-1}}
  \right \}.
  \numberthis{\label{eqn:ApproximateLinearConcentrability}}
	\end{align*}
Step (i) follows from the expression~\eqref{eqn:LinearBellmanError}
for the weak Bellman error, along with the definition of the weak
solution $\QpiWeak{\ActorPolicy{\iter}}$.

\paragraph{Bound under weak Bellman closure:}
When Bellman closure holds we proceed analogously.  The bound in
\cref{lem:RelaxedLinearConstraintsBellmanClosure} ensures with
probability at least $1-\FailureProbability$ that $ \norm{\CriticPar{}
  -
  \CriticParProjection{\ActorPolicy{\iter}}}{\CovarianceStandardReg}^2
\leq c \; \frac{\dim \Rad}{\nSamples} $ for all $\CriticPar{} \in
\PopulationFeasibleSet{\ActorPolicy{\iter}}$; as before, this relation
must apply to the parameter chosen by the critic $\CriticPar{\iter}
\in \EmpiricalFeasibleSet{\ActorPolicy{\iter}}$.  The bound on the
adversarial reward along the distribution of the comparator
$\comparator$ now reads
\begin{align*}
\abs{\Expecti{\RewardLawAdv{\iter}}{\comparator}} \: = \:
\abs{\Expecti{\BellmanError{\QminEmp{\ActorPolicy{\iter}}}{\ActorPolicy{\iter}}{}}{\comparator}}
& \overset{\text{(i)}}{=} \abs{\Expecti{ \LinPhi^\top
    (\CriticPar{\iter} -
    \CriticParProjectionFull{\ActorPolicy{\iter}}{\CriticPar{\iter}})
  }{\comparator}} \\ & \leq
\norm{\Expecti{\LinPhi}{\comparator}}{\CovarianceStandardReg^{-1}}
\norm{\CriticPar{\iter} -
  \CriticParProjectionFull{\ActorPolicy{\iter}}{\CriticPar{\iter}}}{\CovarianceStandardReg}
\\
& \leq c \;
\norm{\Expecti{\LinPhi}{\comparator}}{\CovarianceStandardReg^{-1}}
\sqrt{\frac{\dim \Rad}{\nSamples}}.
\numberthis{\label{eqn:ApproximateLinearConcentrabilityBellmanClosure}}
\end{align*}
Here step (i) follows from the
expression~\eqref{eqn:LinearBellmanErrorWithClosure} for the Bellman
error under weak closure.


\subsection{Proof of \cref{prop:MirrorDescentAdversarialRewards}}
\label{sec:MirrorDescentAdversarialRewards}

We now prove our guarantee for a mirror descent procedure on the
sequence of adversarial MDPs.  Our analysis makes use of a standard
result on online mirror descent for linear functions (e.g., see
Section 5.4.2 of Hazan~\cite{hazan2021introduction}), which we state
here for reference.  Given a finite cardinality set $\xSpace$, a
function $f: \xSpace \rightarrow \R$, and a distribution $\mirrdist$
over $\xSpace$, we define $f(\mirrdist) \defeq \sum_{x \in \xSpace}
\mirrdist(x) f(x)$.  The following result gives a guarantee that holds
uniformly for any sequence of functions
$\{f_\iter\}_{\iter=1}^\nIter$, thereby allowing for the possibility
of adversarial behavior.
\begin{proposition}[Adversarial Guarantees for Mirror Descent]
\label{prop:MirrorDescent}
Suppose that we initialize with the uniform distribution
$\mirrdist_{1}(\xvar) = \frac{1}{\abs{\xSpace}}$ for all $\xvar
\in \xSpace$, and then perform $\nIter$ rounds of the update
\begin{align}
  \label{eqn:ExponentiatedGradient}
  \mirrdist_{\iter+1}(\xvar ) \propto \mirrdist_{\iter}(\xvar)
  \exp(\LearningRate \fnc_{\iter}(\xvar)), \quad \mbox{for all $\xvar
    \in \xSpace$,}
\end{align}
using $\LearningRate  = \sqrt{\frac{\log \abs{\xSpace}}{2\nIter}}$.  
If $\norm{\fnc_{\iter}}{\infty} \leq 1$ for all $\iter \in [\nIter]$ then we have the
bound
\begin{align}
\label{EqnMirrorBoundStatewise}  
\frac{1}{\nIter} \sum_{\iter=1}^\nIter \Big[
  \fnc_{\iter}(\widetilde{\mirrdist}) -
  \fnc_{\iter}(\mirrdist_{\iter})\Big] \leq \MirrorRegret{\nIter}
\defeq 2\sqrt{\frac{2\log \abs{\xSpace}}{\nIter}}.
\end{align}
where $\widetilde{\mirrdist}$ is any comparator distribution over
$\xSpace$.
\end{proposition}

We now use this result to prove our claim.  So as to streamline the
presentation, it is convenient to introduce the advantage function
corresponding to $\ActorPolicy{\iter}$.  It is a function of the
state-action pair $\psa$ given by
\begin{align*}
\ApiAdv{\iter}{\ActorPolicy{\iter}}\psa \defeq
\QpiAdv{\iter}{\ActorPolicy{\iter}}\psa -
\Expecti{\QpiAdv{\iter}{\ActorPolicy{\iter}}(\state,\successoraction)}{\successoraction
  \sim \ActorPolicy{\iter}(\cdot \mid \state)}.
\end{align*}
In the sequel, we omit dependence on $\psa$ when referring to this
function, consistent with the rest of the paper.

From our earlier observation~\eqref{eqn:VonAdv}, recall that the
reward function of the perturbed MDP $\MDPadv{\iter}$ corresponds to
that of $\MDP$ plus the perturbation $\RewardLawAdv{\iter}$.
Combining this fact with a standard simulation lemma (e.g.,
\cite{kakade2003sample}) applied to $\MDPadv{\iter}$, we find that
\begin{subequations}
\begin{align}
\label{EqnInitialBound}  
\Vpi{\comparator} - \Vpi{\ActorPolicy{\iter}} & =
\VpiAdv{\iter}{\comparator} - \VpiAdv{\iter}{\ActorPolicy{\iter}} +
\horizon \Big[ - \Expecti{\RewardLawAdv{\iter}}{\comparator} +
  \Expecti{\RewardLawAdv{\iter}}{\ActorPolicy{\iter}} \Big] \; = \;
\horizon \Big[
  \Expecti{\ApiAdv{\iter}{\ActorPolicy{\iter}}}{\comparator} -
  \Expecti{\RewardLawAdv{\iter}}{\comparator} +
  \Expecti{\RewardLawAdv{\iter}}{\ActorPolicy{\iter}} \Big].
\end{align}
Now for any given state $\state$, we introduce the linear objective
function
\begin{align*}
  \fnc_\iter(\mirrdist) & \defeq \Exp_{\action \sim \mirrdist}
  \QpiAdv{\iter}{\ActorPolicy{\iter}}(\state, \action) \; = \;
  \sum_{\action \in \ActionSpace} \mirrdist(\action)
  \QpiAdv{\iter}{\ActorPolicy{\iter}}(\state, \action),
\end{align*}
where $\mirrdist$ is a distribution over the action space.  With this
choice, we have the equivalence
\begin{align*}
  \Expecti{\ApiAdv{\iter}{\ActorPolicy{\iter}}}{\action \sim
    \comparator}\psa & = \fnc_\iter(\comparator(\cdot \mid \state)) -
  \fnc_\iter\big(\ActorPolicy{\iter}(\cdot \mid \state) \big),
\end{align*}
where the reader should recall that we have fixed an arbitrary state
$\state$.  Consequently, applying the bound~\eqref{EqnMirrorBoundStatewise}
with $\xSpace = \ActionSpace$ and these choices of linear functions,
we conclude that
\begin{align}
\label{EqnMyMirror}  
\frac{1}{\nIter} \sumiter
\Expecti{\ApiAdv{\iter}{\ActorPolicy{\iter}}}{\action \sim
  \comparator}\psa \leq \MirrorRegret{\nIter}.
\end{align}
\end{subequations}
This bound holds for any state, and also for any average over the
states.

We now combine the pieces to conclude.  By computing the average of
the bound~\eqref{EqnInitialBound} over all $\nIter$ iterations, we
find that
\begin{align*}
  \frac{1}{\nIter} \sumiter \Big[ \Vpi{\comparator} -
    \Vpi{\ActorPolicy{\iter}} \Big] & \leq \horizon \left \{
  \frac{1}{\nIter} \sumiter
  \Expecti{\ApiAdv{\iter}{\ActorPolicy{\iter}}}{\comparator} +
  \frac{1}{\nIter} \sumiter \Big[ -
    \Expecti{\RewardLawAdv{\iter}}{\comparator} +
    \Expecti{\RewardLawAdv{\iter}}{\ActorPolicy{\iter}} \Big] \right
  \} \\
& \leq \horizon \left \{ \MirrorRegret{\nIter} + \frac{1}{\nIter}
  \sumiter \Big[ - \Expecti{\RewardLawAdv{\iter}}{\comparator} +
    \Expecti{\RewardLawAdv{\iter}}{\ActorPolicy{\iter}} \Big] \right
  \},
\end{align*}
where the final inequality follow from the
bound~\eqref{EqnMirrorBoundStatewise}, applied for each $\state$.  We have thus
established the claim.






%





\end{document}